%% file: main.tex
\documentclass{article}
\usepackage{arxiv}

\usepackage[american]{babel}

\usepackage[round,semicolon]{natbib} 
    \renewcommand{\bibsection}{\subsubsection*{References}}
\usepackage{mathtools} 
\usepackage{siunitx} 
\usepackage{tikz} 

\newcommand\iprod[2]{\langle #1, #2 \rangle}

\usepackage{anyfontsize}

\usepackage[capitalize,noabbrev]{cleveref}

\theoremstyle{plain}
\newtheorem{theorem}{Theorem}[section]

\newtheorem{lemma}[theorem]{Lemma}
\newtheorem{corollary}[theorem]{Corollary}
\theoremstyle{definition}
\newtheorem{definition}[theorem]{Definition}

\theoremstyle{remark}

\usepackage{subcaption}     
\usepackage{multirow}
\usepackage{bbm}
\usepackage{pifont}
\usepackage{hhline}
\usepackage{framed}
\usepackage{enumitem}

\Crefname{section}{Section}{Sections}
\Crefname{table}{Table}{Tables}
\Crefname{figure}{Figure}{Figures}

\makeatletter
\DeclareRobustCommand\onedot{\futurelet\@let@token\@onedot}
\def\@onedot{\ifx\@let@token.\else.\null\fi\xspace}
\def\eg{\emph{e.g}\onedot} 
\def\ie{\emph{i.e}\onedot}

\makeatother

\usepackage{soul}
\soulregister\cite7
\soulregister\citet7
\soulregister\citep7
\soulregister\ref7
\soulregister\cref7
\soulregister\pageref7
\soulregister\footnote7

\soulregister\wip7
\soulregister\needcite7
\soulregister\changelater7
\soulregister\checkchanges7

\newcommand\suggestchange[2]{\ifthenelse{{\equal{#1}{}}}{}{\sethlcolor{red!20}\hl{\textbf{{[old]}} #1}} \ifthenelse{{\equal{#2}{}}}{\sethlcolor{green!20}\hl{\textbf{[new]} (TBD)}}{\sethlcolor{green!20}\hl{\textbf{[new]} #2}}}

\input{math_commands.tex}

\title{Critical Influence of Overparameterization on \\Sharpness-aware Minimization}

\renewcommand\footnotemark{}
\author{Sungbin Shin$^1$*\thanks{*\,Equal contribution.} $\ \ \ \ $ Dongyeop Lee$^1$* $\ \ \ \ $ Maksym Andriushchenko$^2$  $\ \ \ \ $ Namhoon Lee$^1$\\
POSTECH$^1$, EPFL$^2$\\
\texttt{\{ssbin4, dylee23, namhoonlee\}@postech.ac.kr}\\
\texttt{maksym.andriushchenko@epfl.ch}
}

\begin{document}

\maketitle

\input{text/0_abstract}

\input{text/1_intro}

\input{text/2_background}

\input{text/3_experiments}

\input{text/4_understanding}

\input{text/5_practical}

\input{text/6_theory}

\input{text/7_discussion}

\input{text/acknowledgement}

\bibliography{main}

\newpage

\appendix

\input{text/appendix}

\end{document}

%% file: math_commands.tex
\def\eqref#1{equation~\ref{#1}}

\def\1{\bm{1}}

\DeclareMathAlphabet{\mathsfit}{\encodingdefault}{\sfdefault}{m}{sl}
\SetMathAlphabet{\mathsfit}{bold}{\encodingdefault}{\sfdefault}{bx}{n}

\newcommand{\E}{\mathbb{E}}

\newcommand{\R}{\mathbb{R}}

\newcommand\expec[2]{\mathbb{E}_{#1}{\left\lbrack#2\right\rbrack}}

\newcommand{\ip}[2]{\langle #1,#2\rangle}

\newcommand{\ex}[2]{\underset{#1}{\mathbb{E}}\left[ #2 \right]}

\newcommand{\Id}{\mathbbm{1}}

\def\R{\mathbb{R}}

\newcommand{\norm}[1]{\left\|#1\right\|}

\newcommand*{\defeq}{\stackrel{\text{def}}{=}}

\newcommand\eqdef{\ensuremath{\stackrel{\rm def}{=}}} %

\newcommand{\compsam}{\mathfrak{C}_\mathfrak{s}}
\newcommand{\compnet}{\mathfrak{C}_\epsilon}
\newcommand{\Z}{\mathcal{Z}}
\newcommand{\Y}{\mathcal{Y}}
\newcommand{\C}{\mathcal{C}}
\newcommand{\D}{\mathcal{D}}
\newcommand{\calH}{\mathcal{H}}
\newcommand{\initw}{W^{(0)}}
\newcommand{\init}[1]{#1^{(0)}}
\newcommand{\wt}[1]{\widetilde{#1}}
\newcommand{\wttheta}{\widetilde{\Theta}}

\newcommand{\iterthalf}[1]{#1^{(t + 1/2)}}
\newcommand{\itert}[1]{#1^{(t)}}
\newcommand{\poly}{\textsf{\textup{poly}}}
\newcommand{\opt}{\textsf{\textup{OPT}}}
\newcommand{\refw}{W^{\text{\textreferencemark}}}

%% file: text/0_abstract.tex
\begin{abstract}
Sharpness-Aware Minimization (SAM) has attracted considerable attention for its effectiveness in improving generalization in deep neural network training by explicitly minimizing sharpness in the loss landscape.
Its success, however, relies on the assumption that there exists sufficient variability of flatness in the solution space—a condition commonly facilitated by overparameterization.
Yet, the interaction between SAM and overparameterization has not been thoroughly investigated, leaving a gap in understanding precisely how overparameterization affects SAM.
Thus, in this work, we analyze SAM under varying degrees of overparameterization, presenting both empirical and theoretical findings that reveal its critical influence on SAM's effectiveness.
First, we conduct extensive numerical experiments across diverse domains, demonstrating that SAM consistently benefits from overparameterization.
Next, we attribute this phenomenon to the interplay between the enlarged solution space and increased implicit bias resulting from overparameterization.
Furthermore, we show that this effect is particularly pronounced in practical settings involving label noise and sparsity, and yet, sufficient regularization is necessary.
Last but not least, we provide other theoretical insights into how overparameterization helps SAM achieve minima with more uniform Hessian moments compared to SGD, and much faster convergence at a linear rate.
\end{abstract}

%% file: text/1_intro.tex
\section{Introduction} \label{sec:intro}

Optimization algorithms, though primarily designed to minimize training loss, have increasingly been recognized for their role in implicitly regularizing machine learning models, with some optimizers leading to stronger generalization than others \citep{keskar2017on, wilson2017marginal, ji2020gradient, andriushchenko2023sgd}.
This has motivated extensive efforts to uncover the underlying mechanisms and incorporate these insights into the design of more effective optimizers \citep{izmailov2018averaging,foret2020sharpness,orvieto2022anticorrelated, zhao2022penalizing}.

One prominent line of research examines the relationship between the sharpness of the loss landscape and generalization error, with flatter minima generally associated with improved generalization performance \citep{hochreiter1997flat, keskar2017on, jiang2019fantastic}.
This observation has motivated the development of optimization techniques aimed at encouraging convergence to such flat regions \citep{izmailov2018averaging, chaudhari2017entropysgd, foret2020sharpness, orvieto2022anticorrelated, zhao2022penalizing}.
Notably, SAM \citep{foret2020sharpness} has drawn significant interest for its ability to promote flatter minima and enhance generalization beyond what is typically achieved with standard optimizers \citep{bahri2022sharpness, chenvision, qu2022generalized}.

However, it relies on a seemingly implicit assumption: that the loss landscape provides sufficient variability in flatness for SAM to exploit.
Recent perspectives suggest that overparameterization may be precisely what gives rise to such conditions, as it enlarges the solution space and potentially enables solutions with greater variation in local geometry, such as sharpness \citep{ma2023recent}.
If so, overparameterization might not be merely optional but essential: without it, SAM might fail to produce similar benefits.

This line of reasoning motivates us to conduct a closer examination into the effects of overparameterization on sharpness-aware minimization (SAM) \citep{foret2020sharpness}, with an eye toward understanding not just whether SAM is effective but also under what conditions and why.
Specifically, we conduct extensive experiments to precisely measure the impact of overparameterization across a diverse set of tasks, ranging from standard tasks in computer vision and natural language processing, to molecular property prediction, and further, to video games in reinforcement learning.
To gain further insight into the results, we perform detailed investigations into the interactions between overparameterization and SAM through visual inspection of the solution space on a simple regression setting as well as analyzing the influence of overparameterization on the implicit bias of SAM.
Furthermore, we study how overparameterization influences SAM under various conditions, including label noise, sparsity, and regularization.
Last but not least, we explore other implications of overparameterization on SAM through theoretical analyses, including the characteristics of the attained minima and the convergence rate.

Our key contributions and findings are summarized as follows.

\vspace{-0.2em}
\begin{description}[font={\tiny$\bullet$}~~\normalfont, leftmargin=5.9em, labelsep=1em]
    \item [\cref{sec:experiments-main}.]
    We perform extensive experiments across eight workloads of datasets and models at varying scales, spanning synthetic, vision, language, chemistry, and game domains.
    We observe that overparameterization consistently improves the generalization benefit of SAM\footnote{By ``generalization benefit'', we mean the improvement made by SAM over SGD in validation accuracy.}.
    This phenomenon is general and previously unknown\footnote{While evidence of the similar observation can be found in the literature \citep{chenvision}, no prior work has conducted experiments or confirmed this phenomenon at any scale comparable to ours.}.
    \item [\cref{sec:understanding}.]
    We propose hypotheses to understand this general phenomenon, positing that two factors may be at play:
    (i) overparameterization first increases the number of simpler and flatter solution candidates, and 
    (ii) it also increases the implicit bias of SAM.
    These are verified with standard experiments in both synthetic and realistic settings.
    \item [\cref{sec:practical}.]
    We present the merits and caveats of overparameterization in employing SAM in practice: (i) the benefit of overparameterization for SAM is more pronounced under label noise and sparsity, while (ii) sufficient regularization is needed.
    This can serve as a useful guidance for practitioners.
    \item [\cref{sec:theory}.]
    We develop theoretical analyses\footnote{We note that these are not intended to directly support \cref{sec:experiments-main} and \ref{sec:understanding}, which we discuss in \cref{sec:discussion}.} on linear stability and convergence:
    under overparameterization, (i) linearly stable minima for SAM are flatter and have more uniformly distributed Hessian moments compared to SGD, and (ii) a stochastic SAM can converge at a linear rate.
    These are also numerically verified.
    \item [Overall.\hspace{.7em}]
    We discover that overparameterization has \emph{critical} influences on SAM.
    Both empirical performance and theoretical aspects of SAM all improve with overparameterization.
    In other words, SAM may not take its advantage over SGD without overparameterization.
\end{description}  

\newpage

%% file: text/2_background.tex
\section{Background} 
\label{sect:2_Background}

Let us consider the general unconstrained optimization problem:
\vspace{0.5em}
\begin{equation} \label{eq:optimization}
    \min_x f(x)
\end{equation}
\vspace{0.5em}
where $f: \R^d \rightarrow \R$ is the objective function to minimize, and $x \in \R^d$ is the optimization variable.
Based on recent studies that indicate a strong correlation between the sharpness of $f$ at a minimum and its generalization error \citep{keskar2017on,dziugaite2017computing,jiang2019fantastic}, \citet{foret2020sharpness} suggest to turn (\ref{eq:optimization}) into a min-max problem of the following form
\begin{equation} \label{eq:sam-objective}
    \min_x \max_{\lVert\epsilon\rVert_2\leq\rho} f(x + \epsilon)
\end{equation}
where $\epsilon$ and $\rho$ denote some perturbation added to $x$ and its norm bound, respectively.
Thus, the goal is now to seek $x$ that minimizes $f$ in its $\epsilon$-neighborhood, such that the objective landscape becomes locally flat.
Taking the first-order Talyor approximation of $f$ at $x$ and solving for optimal $\epsilon^\star$ gives the following update rule for SAM:

\begin{equation} \label{eq:sam-step}
    x_{t+1} = x_{t} - \eta \nabla f \left(x_t+\rho\frac{\nabla f(x_t)}{\lVert\nabla f(x_t)\rVert_2}\right).
\end{equation}

SAM has been shown to be effective for improving generalization performance compared against SGD \citep{chenvision, kaddour2022flat, bahri2022sharpness}, and subsequent works have analyzed various aspects of SAM under different settings including its convergence rates \citep{andriushchenko2022towards, mi2022make, si2023practical} and implicit bias \citep{compagnoni2023sde, wensharpness,andriushchenko2024sharpness}.

Meanwhile, a considerable amount of evidence has indicated the benefit of overparameterization for training neural networks.
Besides the empirical success witnessed across different domains \citep{kaplan2020scaling,radford2021learning,dehghani2023scaling}, overparameterization turns all local minima into global ones in theory enabling local methods to succeed under non-convex settings \citep{kawaguchi2016deep,du2019gradient}.
Researchers have also proved the power of overparameterization to enable much faster convergence \citep{ma2018power,vaswani2019fast,meng2020fast} and better generalization \citep{allen2019learning,brutzkus2019larger}.
To our knowledge, however, previous work has mostly focused on non-sharpness-aware optimizers, and the effects of overparameterization on SAM has been left rather unattended despite its contemporary significance to large-scale training trends and widespread usage in practice.

%% file: text/3_experiments.tex
\section{Key observation: SAM improves with overparameterization}

\label{sec:experiments-main}

\begin{table*}[!th]
    \centering
    \resizebox{0.9\linewidth}{!}{
    \begin{tabular}{c c c c c c c}
      \toprule
      Workload \# & Domain & Task & Dataset & Architecture & Model \\
      \midrule
      $1$ &  Synthetic         & Regression   & Synthetic & MLP & Two-layer MLP \\ 
      $2$ & Vision         & Image classification   & MNIST & MLP & LeNet-300-100 \\    
      $3$ & Vision  & Image classification & CIFAR-10 & CNN & ResNet-18\\ 
      $4$ & Vision  & Image classification & ImageNet & CNN & ResNet-50\\ 
      $5$ & Language  & PoS tagging & Universal Dependencies & Transformer & Encoder-only Transformer\\ 
      $6$ & Language  & Sentiment classification & SST-2 & RNN & LSTM\\ 
      $7$ & Chemistry  & Graph property prediction & ogbg-molpcba & GNN & GCN\\ 
      $8$ & Game  & Proximal policy optimization & Atari Breakout & CNN & Five-layer CNN\\ 
      \bottomrule
  \end{tabular}
  }
  \caption{
    Summary of evaluation workloads.
    They cover eight different datasets spanning five domains and six tasks at varying scale, and include eight neural network models of five different architecture types.
    For each workload, we test up to ten different models of varying degrees of parameterization.
  }
  \label{tab:workloads}
\end{table*}

\begin{figure*}[!th]
\vspace{-0.4em}
    \centering
  \begin{subfigure}{0.9\linewidth}
      \centering
      \includegraphics[width=0.24\linewidth]{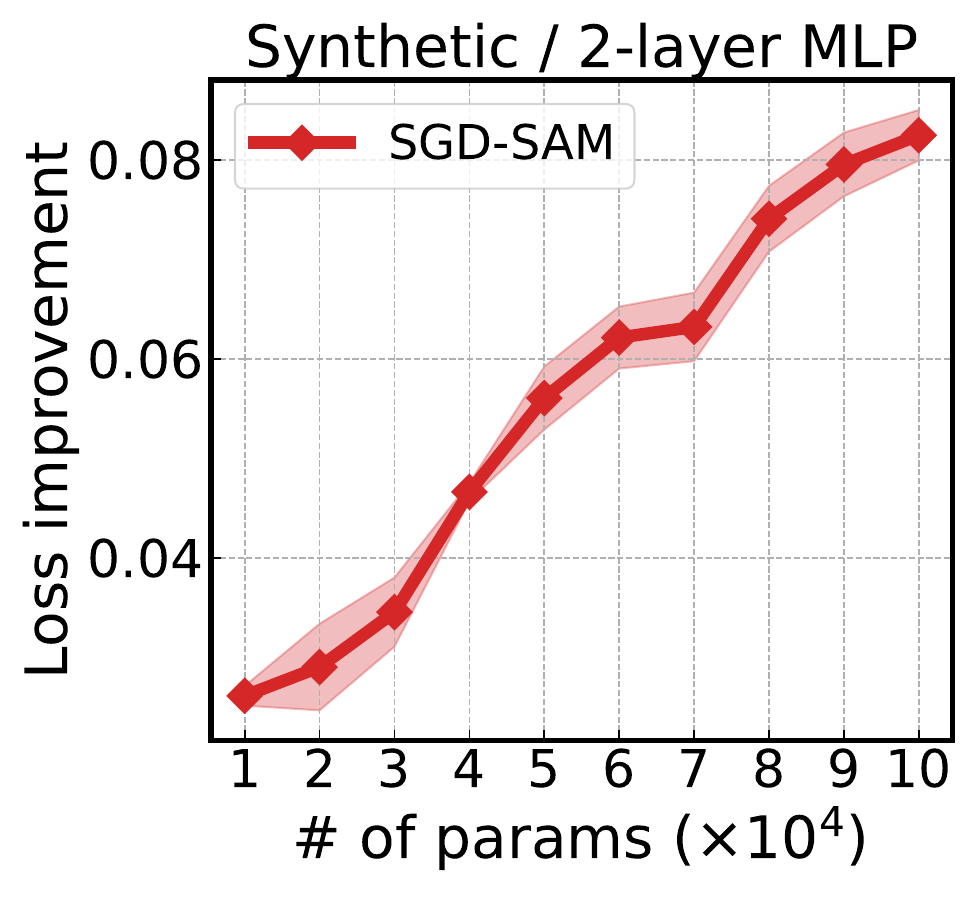}
      \includegraphics[width=0.24\linewidth]{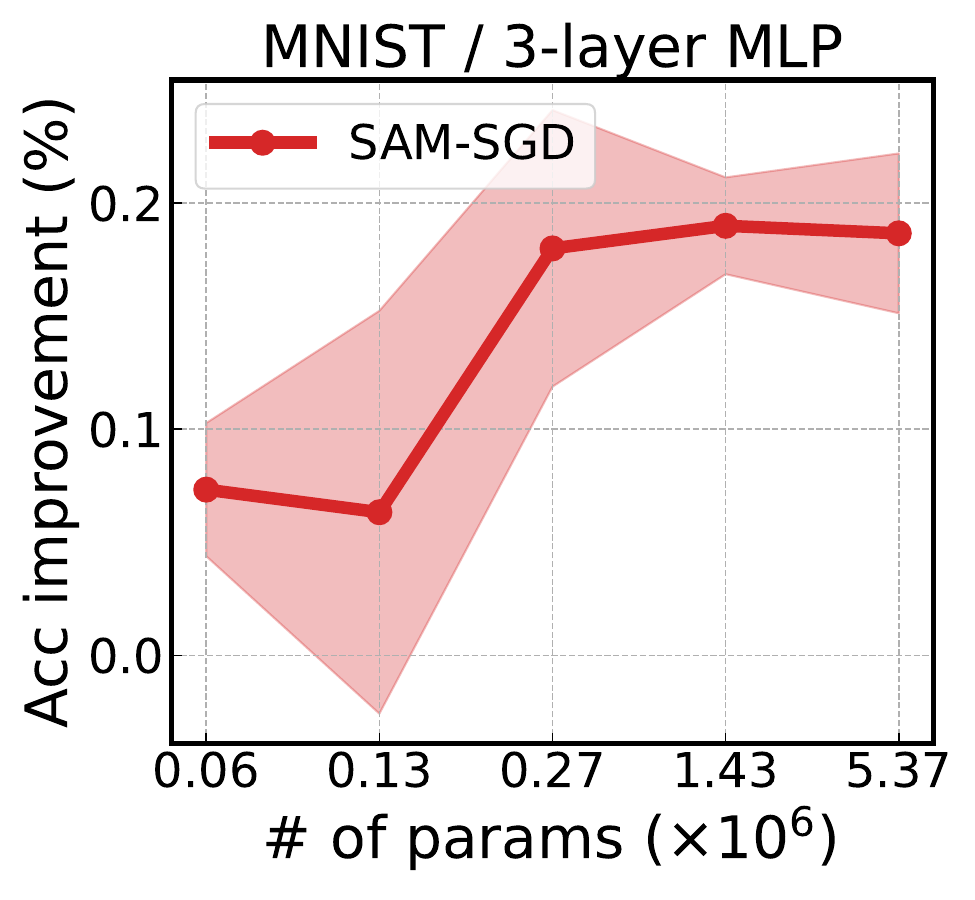}
      \includegraphics[width=0.24\linewidth]{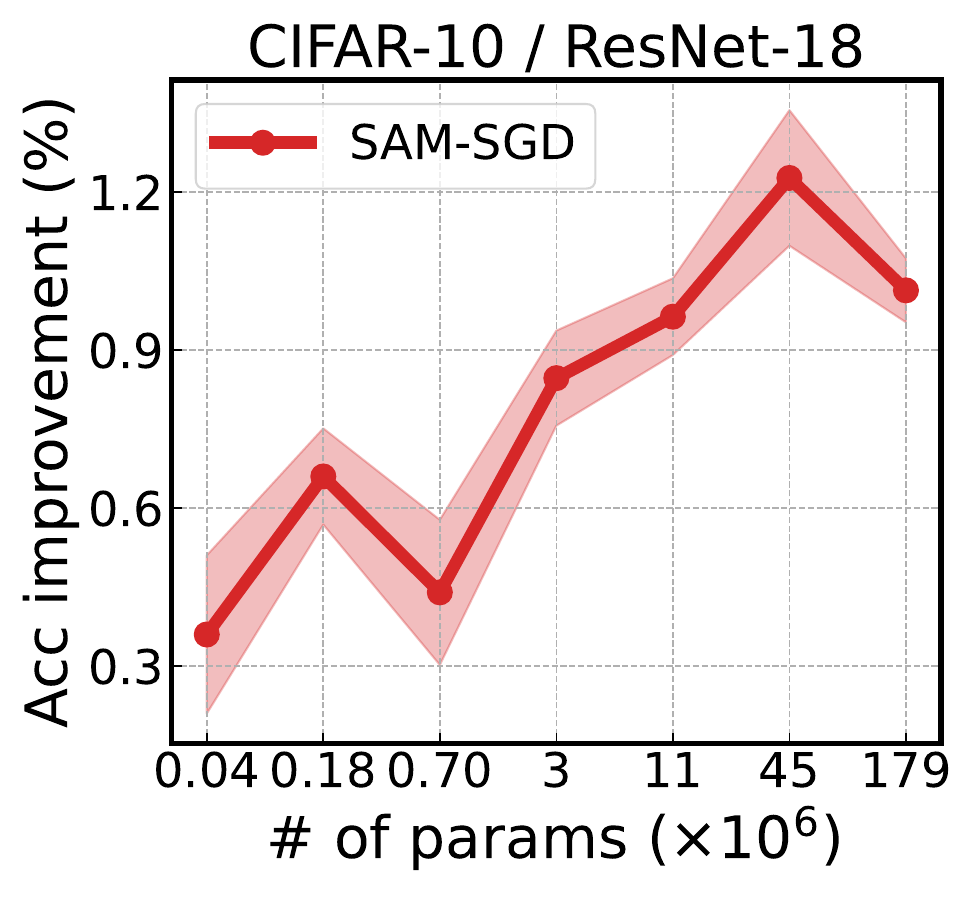}
      \includegraphics[width=0.24\linewidth]{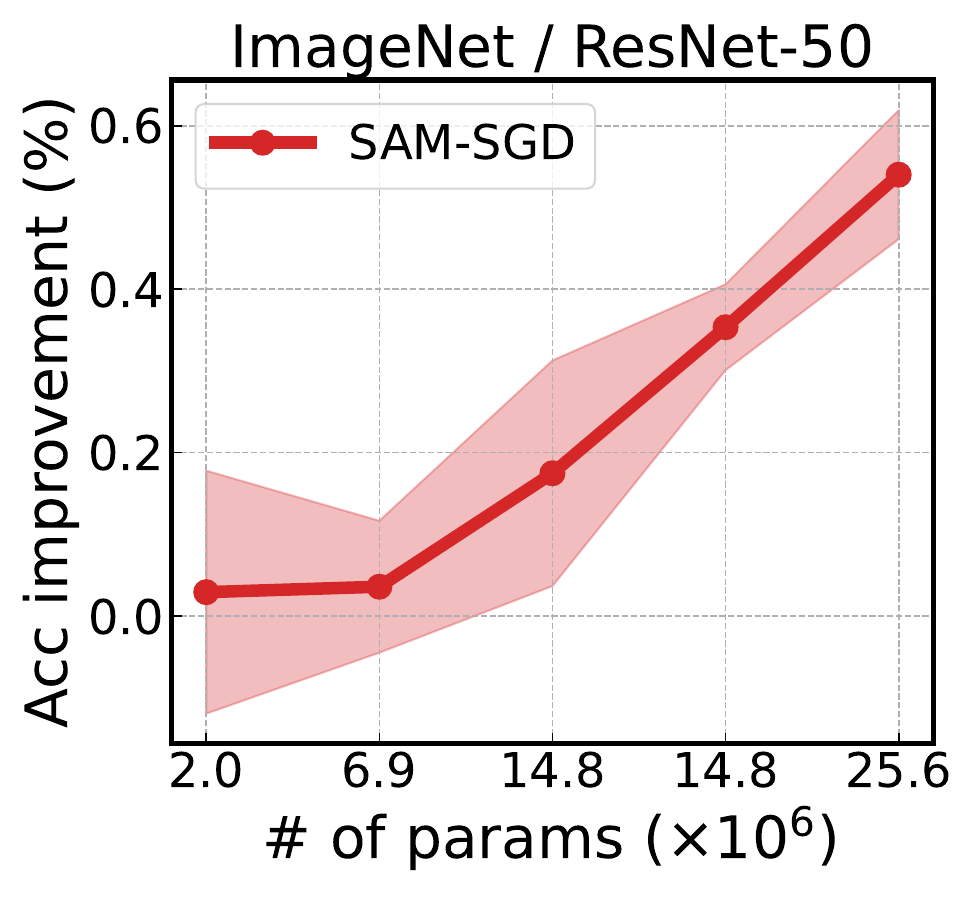}
      \includegraphics[width=0.24\linewidth]{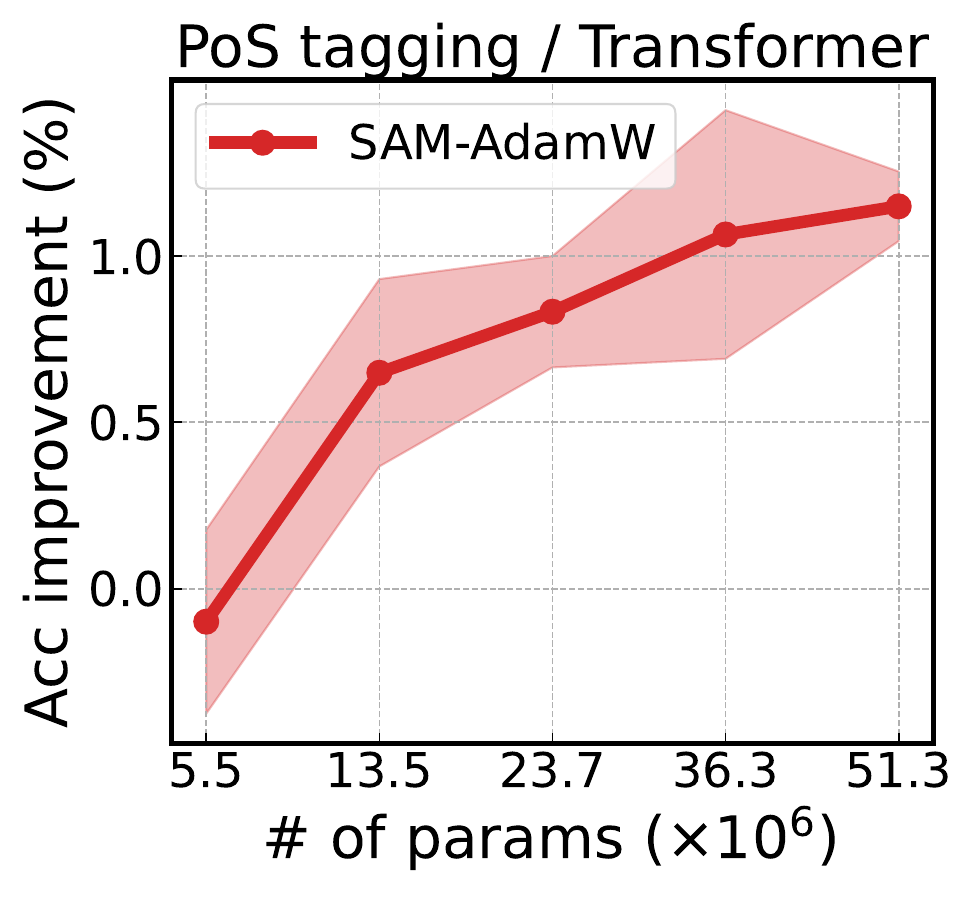}
      \includegraphics[width=0.24\linewidth]{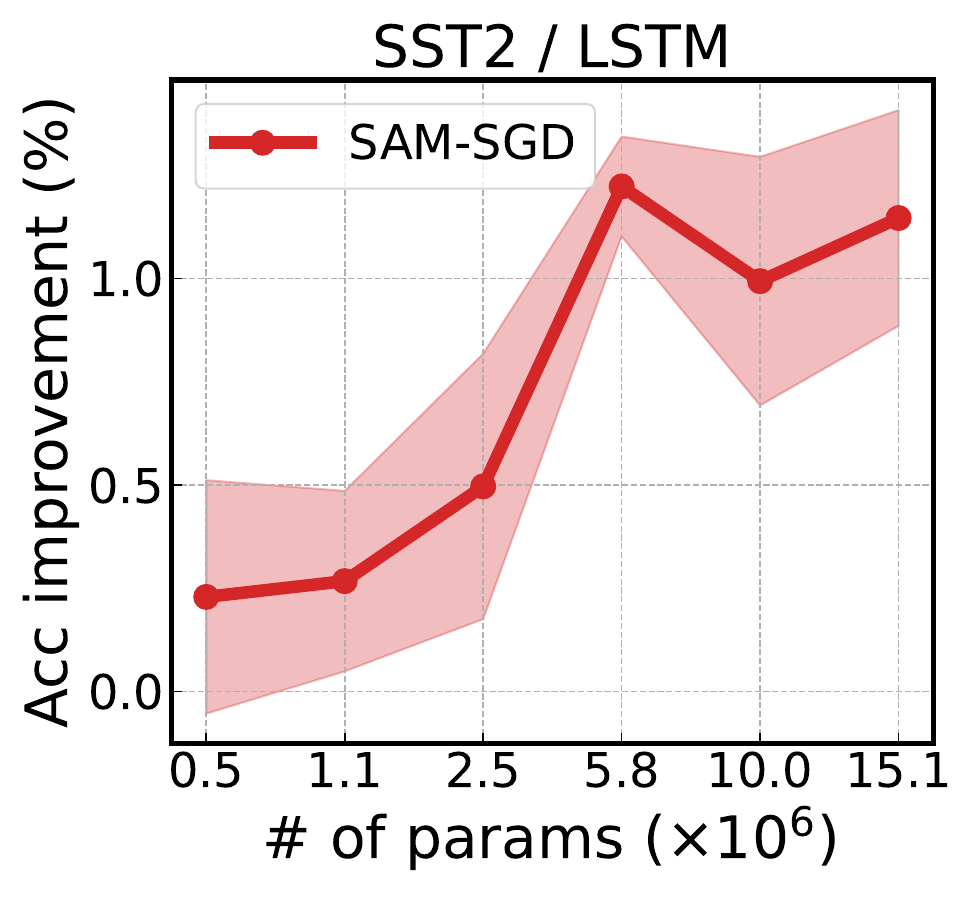}
      \includegraphics[width=0.24\linewidth]{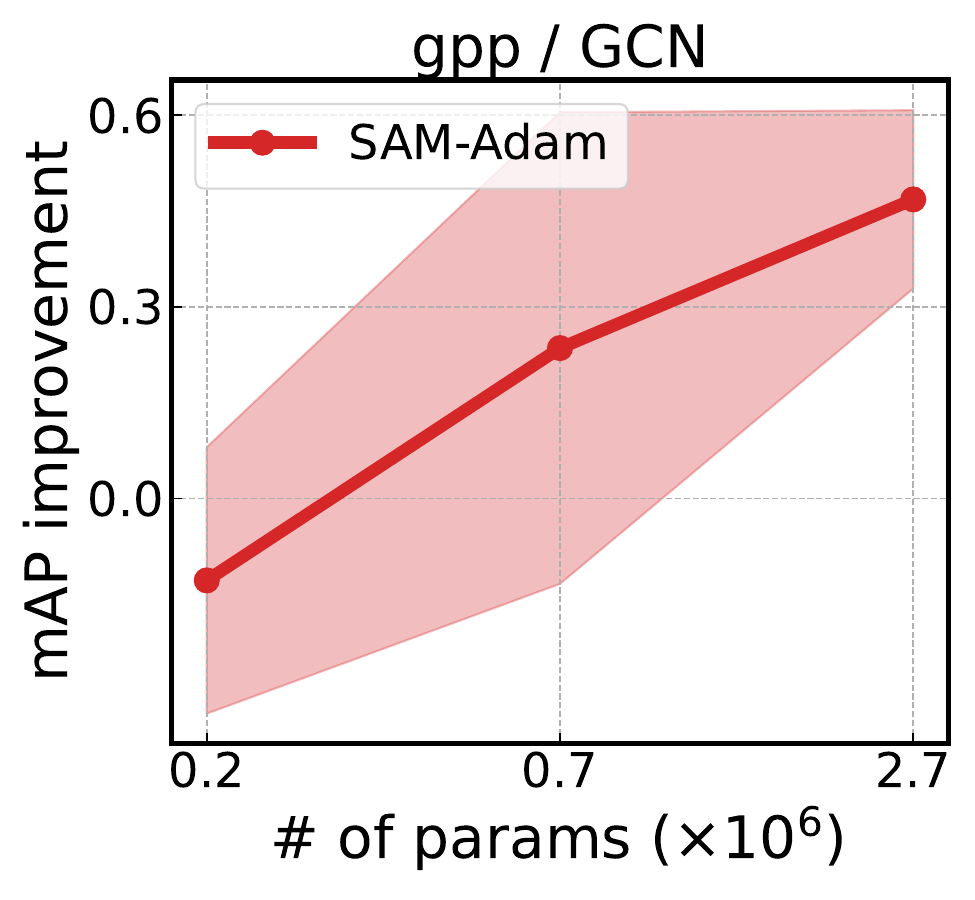}
      \includegraphics[width=0.24\linewidth]{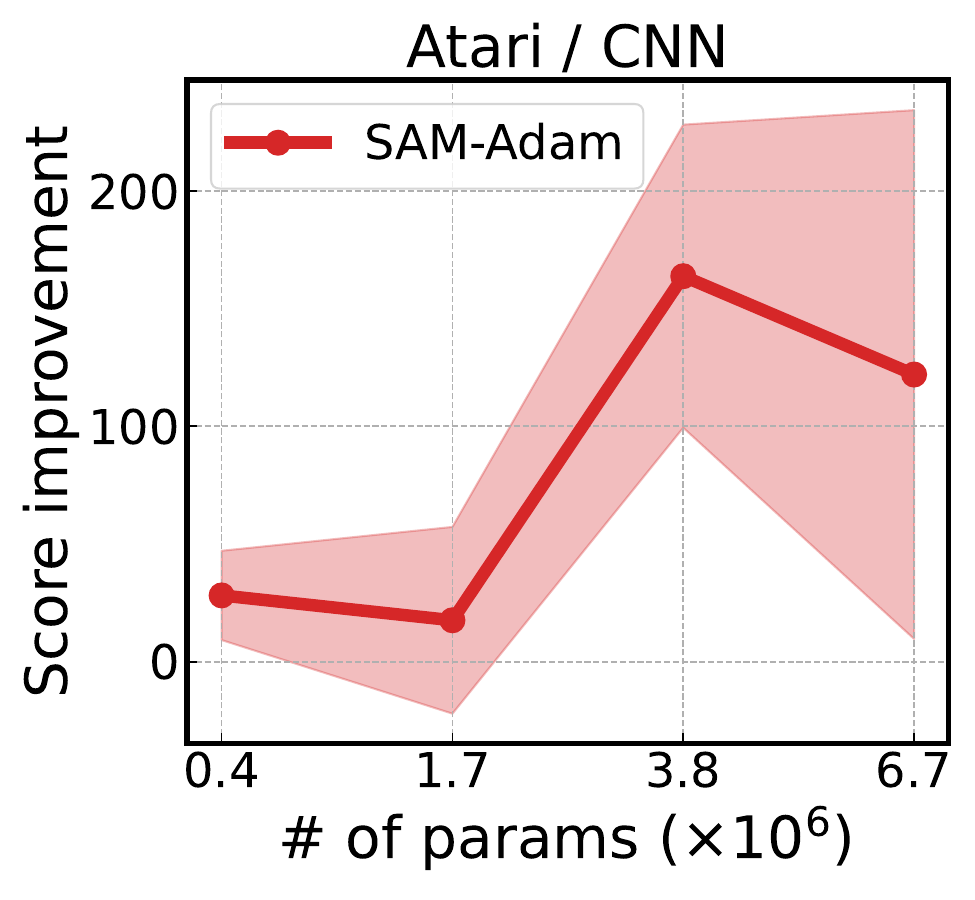}
  \end{subfigure}
    \vspace{-0.4em}
  \caption{
    Improvement in validation metrics by SAM.
    The generalization benefit of SAM tends to increase as the model becomes more overparameterized.
    We present the full results, including the absolute metrics for SAM and baseline optimizers in \cref{fig:result_overparam_acc_abs_app} of \cref{app:add-overparam_results}.
  }
  \label{fig:result_overparam_acc}
    \vspace{-0.7em}
\end{figure*}

SAM is introduced to find flat minima and thereby improve generalization performance in practice.
In this work, we are interested in whether and how this improvement is affected by overparameterization.
In order to understand any potential relationship between SAM and overparameterization, we first focus on precisely measuring the effect of overparameterization.
More specifically, we conduct a wide range of deep learning experiments (see \cref{tab:workloads} for the summary of all tested workloads), and observe how the generalization improvement made by SAM changes as with more parameters.

As a result, we find a strong and consistent trend that SAM improves with overparameterization in all tested cases (see \cref{fig:result_overparam_acc}).
To elaborate, initially, SAM does \emph{not} work much better than the non-sharpness-aware baseline optimizer (\ie, SGD or Adam family depending on the default choice) when the model is at a relatively low number of parameters;
it only starts to improve with more parameters and makes a clear distinction at a very large number of parameters.
We emphasize that this holds true for a wide variety of architectures (MLP, CNN, RNN, GCN, Transformer) and datasets of different domains (Synthetic, Vision, Language, Chemistry, Game) under a rigorous hyperparameter search (see \cref{app:exp_details} for the full experiment details).

This result possibly indicates that SAM is more effective, when (and possibly only when) applied to overparameterized models.
On the other hand, the increased generalization performance of SAM with more parameters renders a promising avenue, given that the modern neural network models are often heavily overparameterized \citep{zhang2022opt,dehghani2023scaling}.
We note that some evidence of the similar positive influence of overparameterization for SAM can be derived in the literature \citep{chenvision}, however, no prior work has conducted experiments or confirmed this phenomenon at any scale comparable to ours.\footnote{
As an additional result, we provide a theoretical analysis of the effect of overparameterization decreasing the test error of SAM in \cref{sec:theory-gen-main,app:sam-genbound,app:sam-genbound}.
Precisely, however, this result only mean for SAM and is not to be confused with the relative improvement against SGD as shown in \cref{sec:experiments-main}.
}

%% file: text/4_understanding.tex
\section{Understanding why SAM improves with overparameterization}
\label{sec:understanding}

Then why does overparameterization particularly favor SAM over non-sharpness-aware optimizers?
We address this question in this section to better understand the effect of overparameterization on SAM.
Precisely, we posit that it is potentially due to the complementarity between overparameterization enlarging the solution space and the implicit bias of SAM driving toward flat minima;
\ie, once there are more diverse solutions available (including both sharp and flat minima) by overparameterization, optimizers intrinsically biased toward flat solutions (such as SAM) will more likely find such solutions than unbiased optimizers (such as SGD).
We support this reasonable hypothesis by demonstrating the followings: 
(i) SAM finds simpler and flatter solutions than SGD with the enlarged solution space (\cref{sec:understanding-solution}), and 
(ii) the implicit bias of SAM becomes stronger with overparameterization (\cref{sec:understanding-implicitbias}); 
both of these take place only when the model is overparameterized.

\subsection{Enlarged solution space allows SAM to find simpler and flatter solutions}
\label{sec:understanding-solution}

\begin{figure*}[!t]
  \centering
  \hspace{-2em}
  \includegraphics[width=0.24\linewidth]{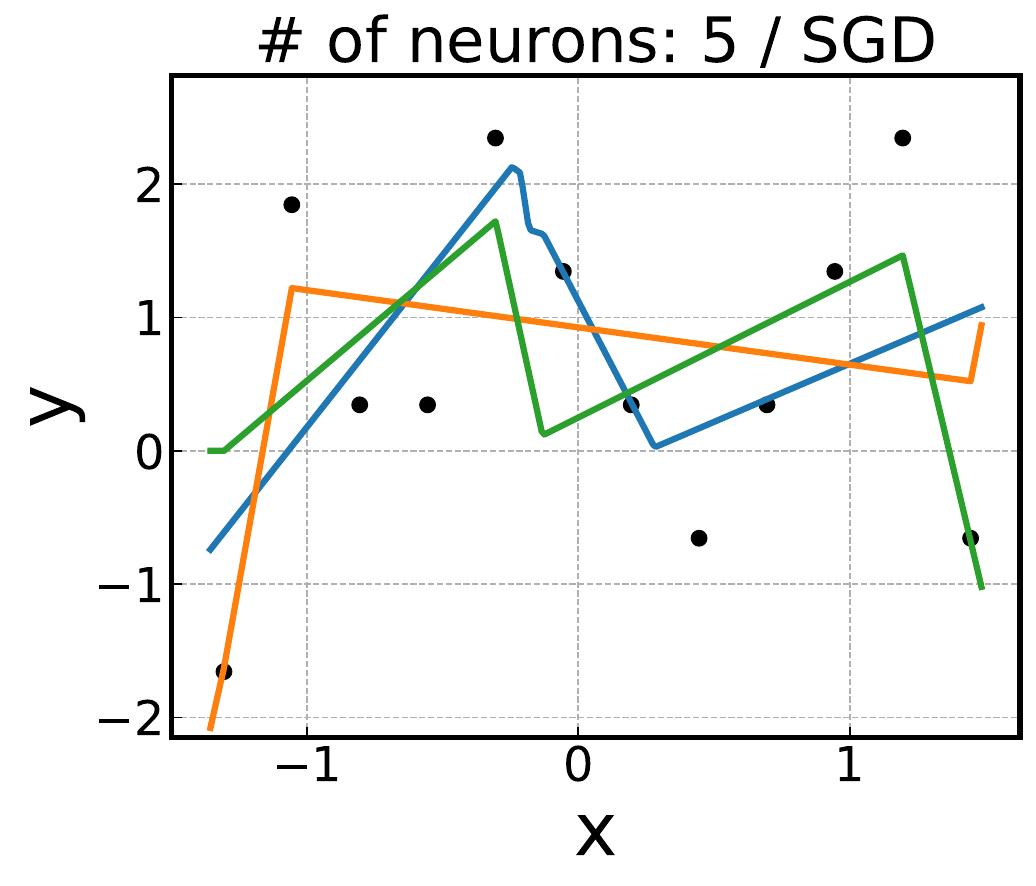}
  \includegraphics[width=0.24\linewidth]{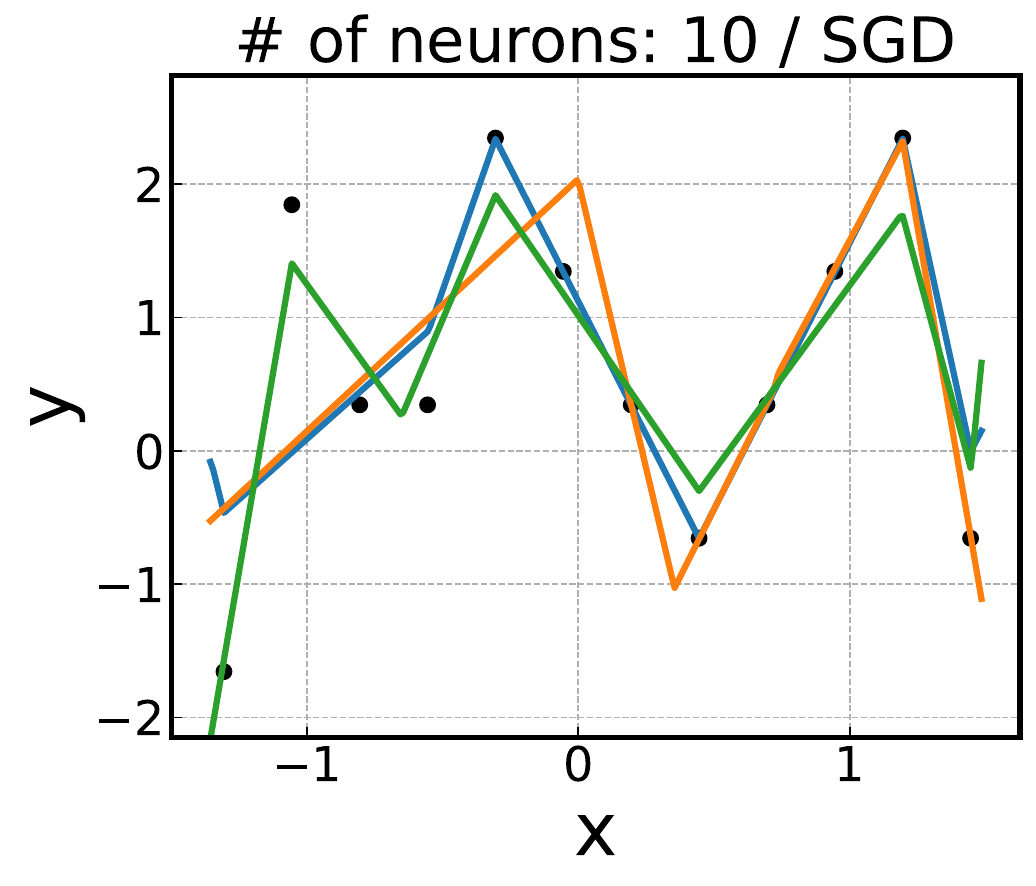}
  \includegraphics[width=0.24\linewidth]{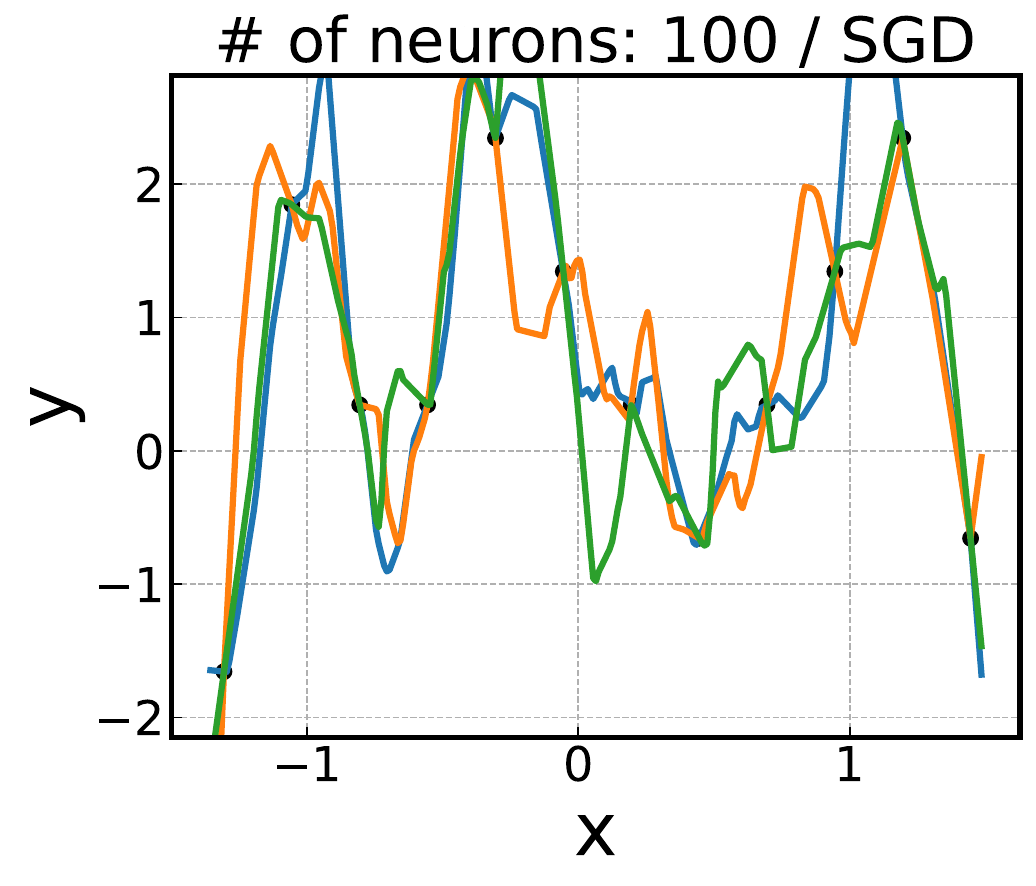}
  \includegraphics[width=0.24\linewidth]{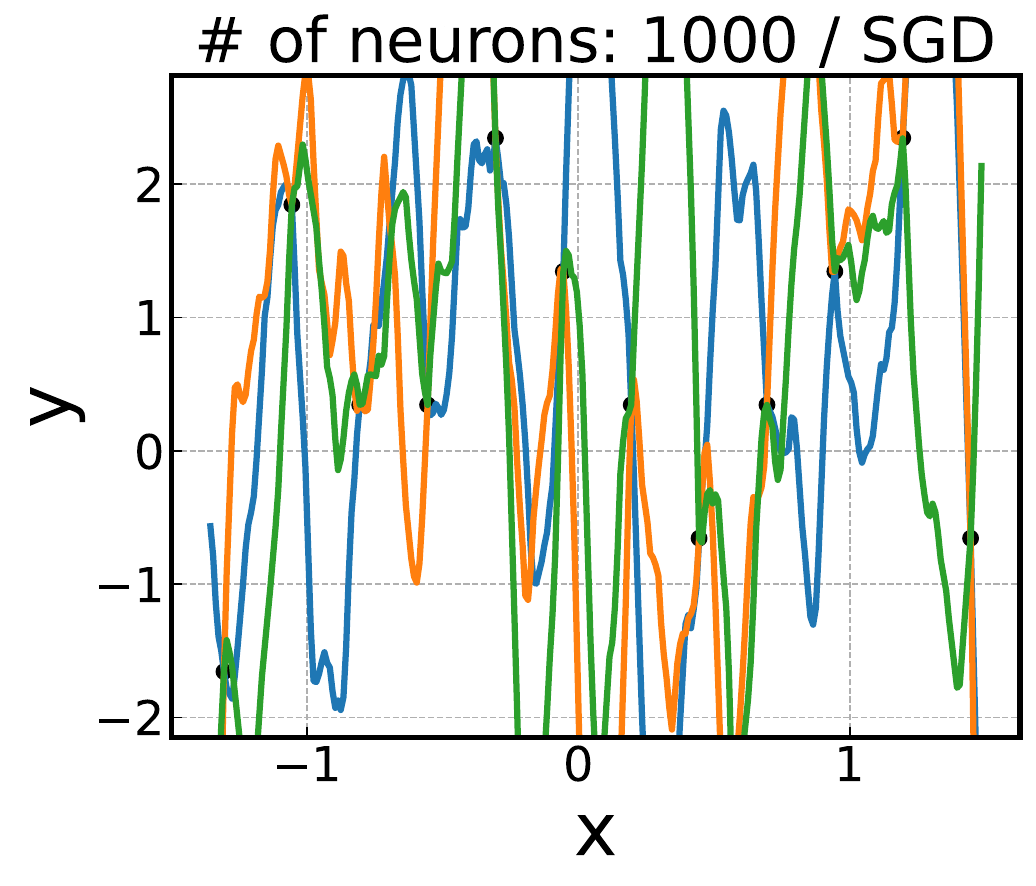}\\
  \hspace{-2em}
  \includegraphics[width=0.24\linewidth]{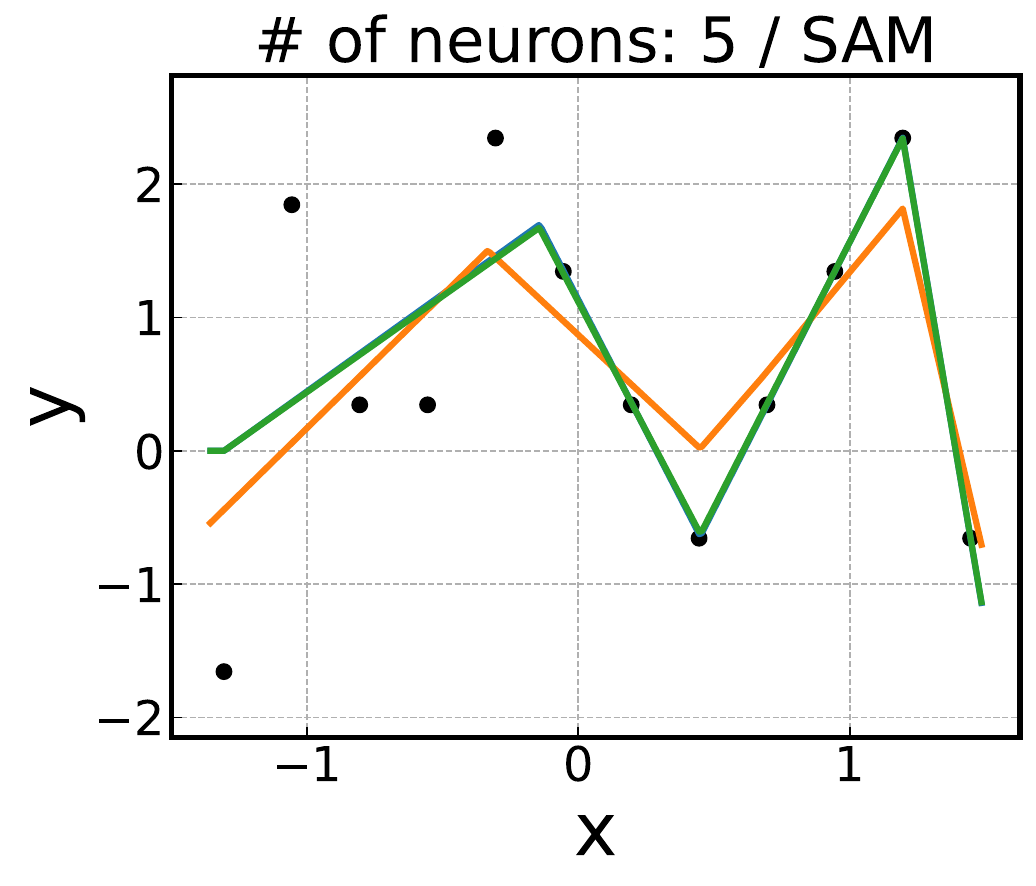}
  \includegraphics[width=0.24\linewidth]{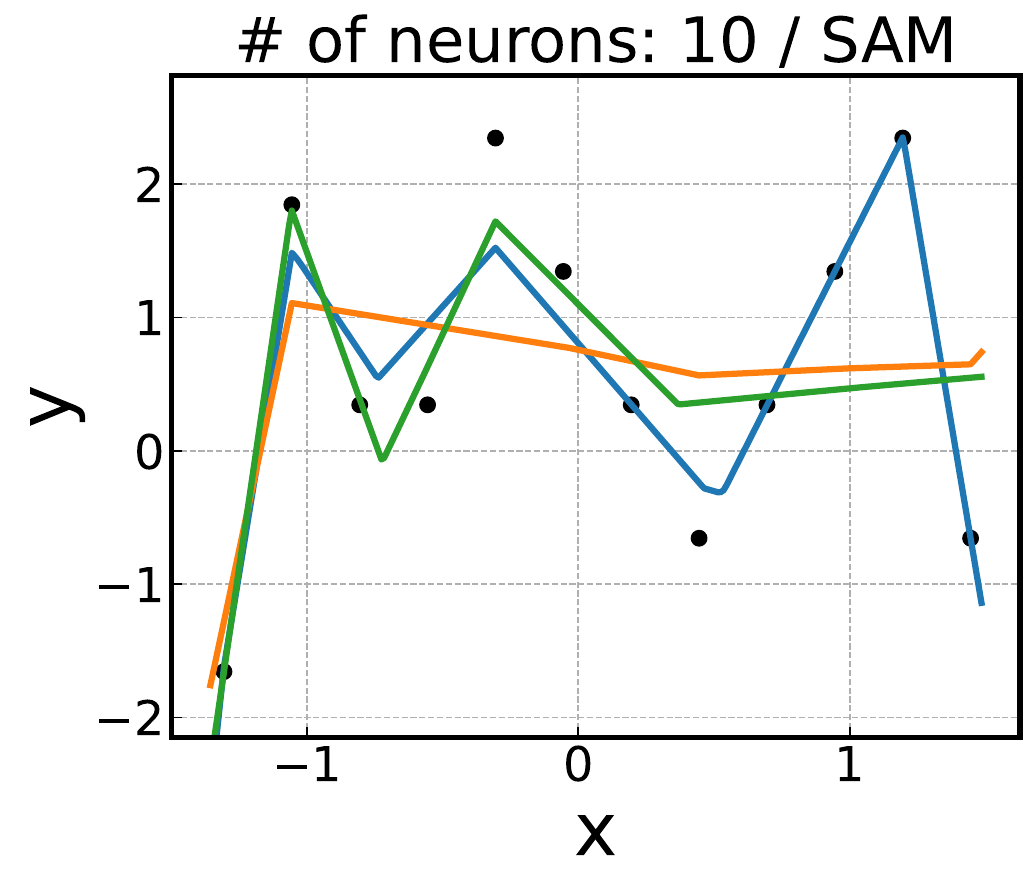}
  \includegraphics[width=0.24\linewidth]{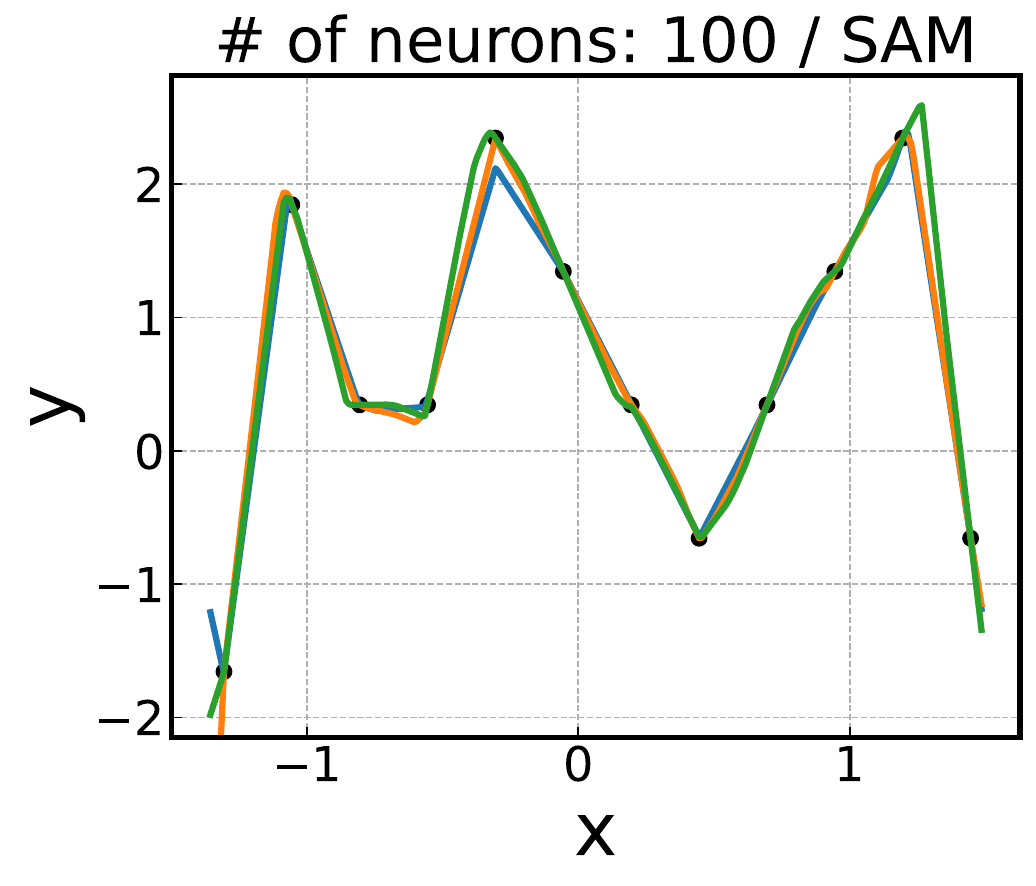}
  \includegraphics[width=0.24\linewidth]{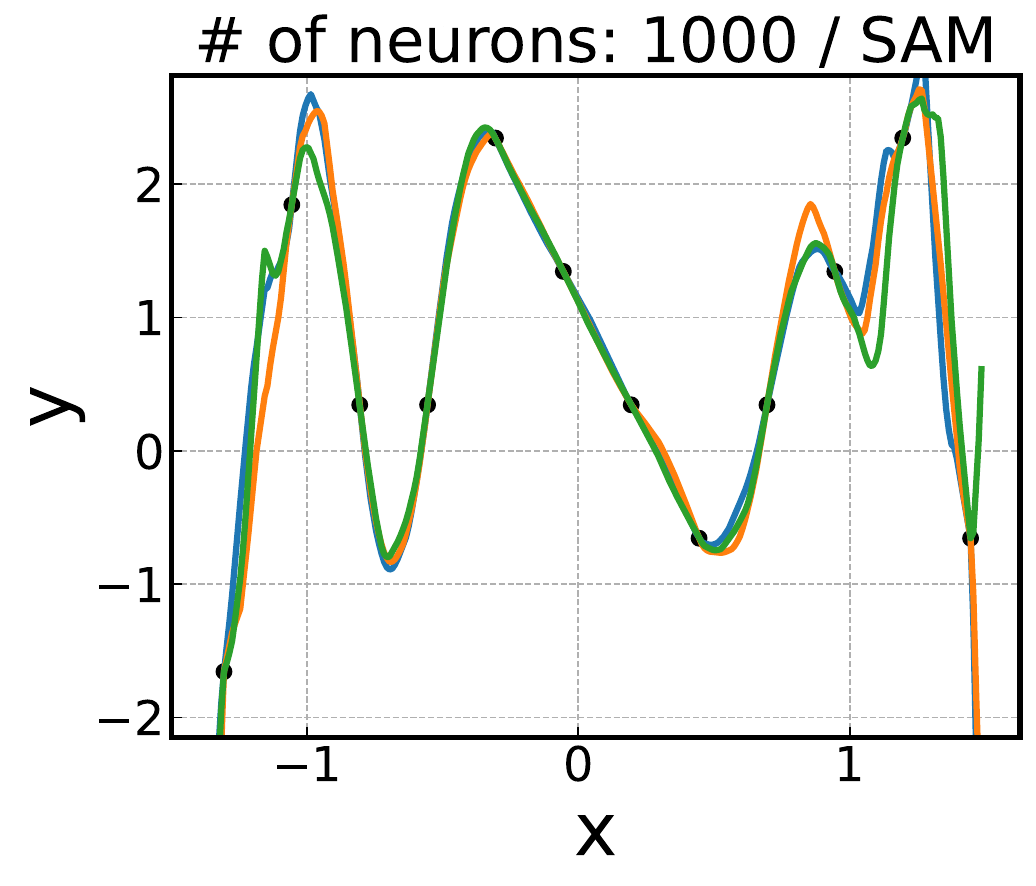}
  \caption{
  Solutions found by SGD (top) and SAM (bottom). 
  Both optimizers find similar solutions for under/moderately-parameterized models, whereas the solutions found by SAM are much simpler with less variance compared to those by SGD for overparameterized models. 
  Here, different colors correspond to different random seeds.
  }
  \label{fig:onehidden-function}
\end{figure*}

To corroborate our hypothesis, we start with a simple experiment where we train one-hidden-layer ReLU networks using SAM and SGD following \citet{andriushchenko2022towards};
we use $5$, $10$, $100$, and $1000$ hidden neurons for underparameterized to highly overparameterized cases;
we run three random seeds and compare solutions obtained by SAM and SGD in \cref{fig:onehidden-function}.

First, we find that the solutions found by SAM are not differentiated much from those of SGD when the model has no more than $10$ neurons.
Looking closely into the case of $10$ neurons, they all seem to be roughly $4$ to $6$ degrees of piecewise linear functions, \ie, the number of line segments for each solution is less than $10$, which is the maximum possible joints that this model can have in theory.
On the other hand, in the case of $100$ to $1000$ neurons, one can easily see that the solutions found by SAM are much simpler (and thus more likely to generalize) compared to those by SGD.

Next, we also track the optimization trajectories of both SAM and SGD.
The trajectories are plotted along PCA directions calculated from the converged minima following \citet{li2018visualizing}.
The results are illustrated in \cref{fig:onehidden-trajectories}.
We find that both SAM and SGD reach solutions in a similar basin when the model is under/moderately parameterized, whereas in the overparameterized case, they reach different solutions, \ie, SAM reaches a flatter solution, even though they all start from the same initial point.

These results support the idea that SAM has some implicit bias that drives itself toward a certain type of solutions (\eg, simple and flat) as previously shown in prior work \citep{andriushchenko2022towards,compagnoni2023sde,wensharpness}.
More importantly, however, these results newly reveal that \emph{overparameterization is a critical factor in facilitating this implicit behavior of SAM};
without it the space of potential solutions decreases, and SAM may not take effect.

\begin{figure*}[!t]
  \centering
  \includegraphics[width=0.24\linewidth]{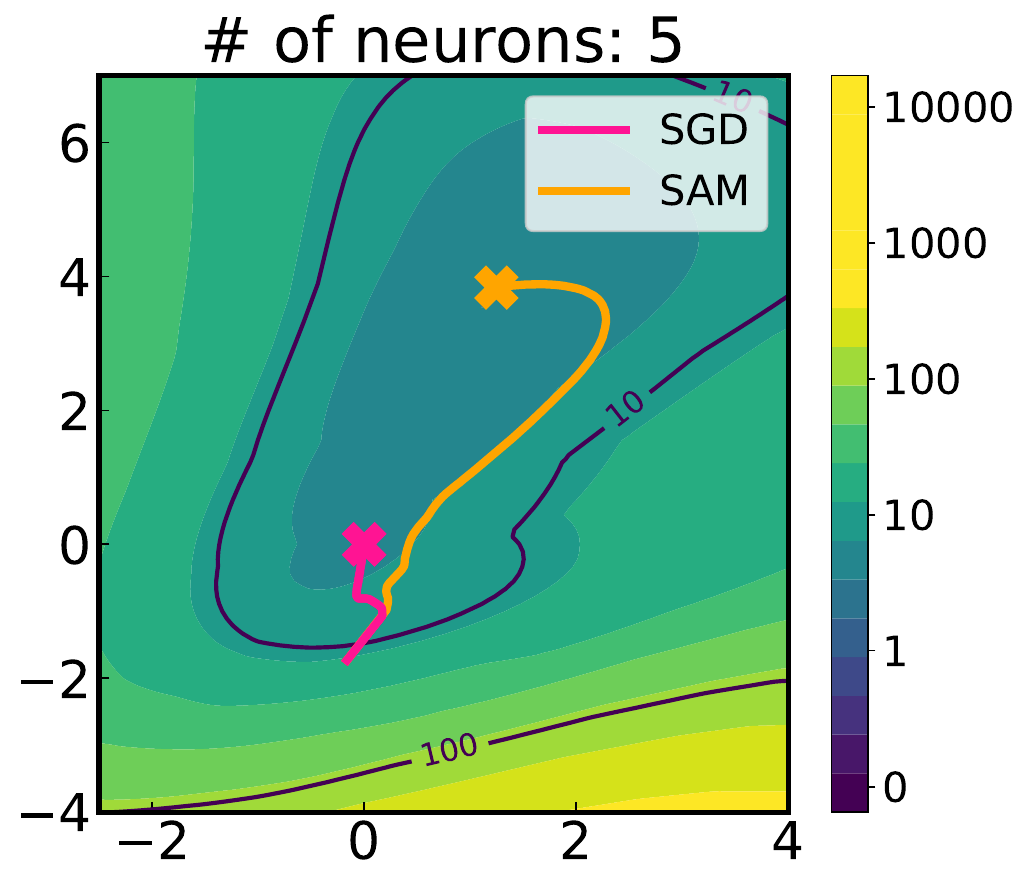}
  \includegraphics[width=0.24\linewidth]{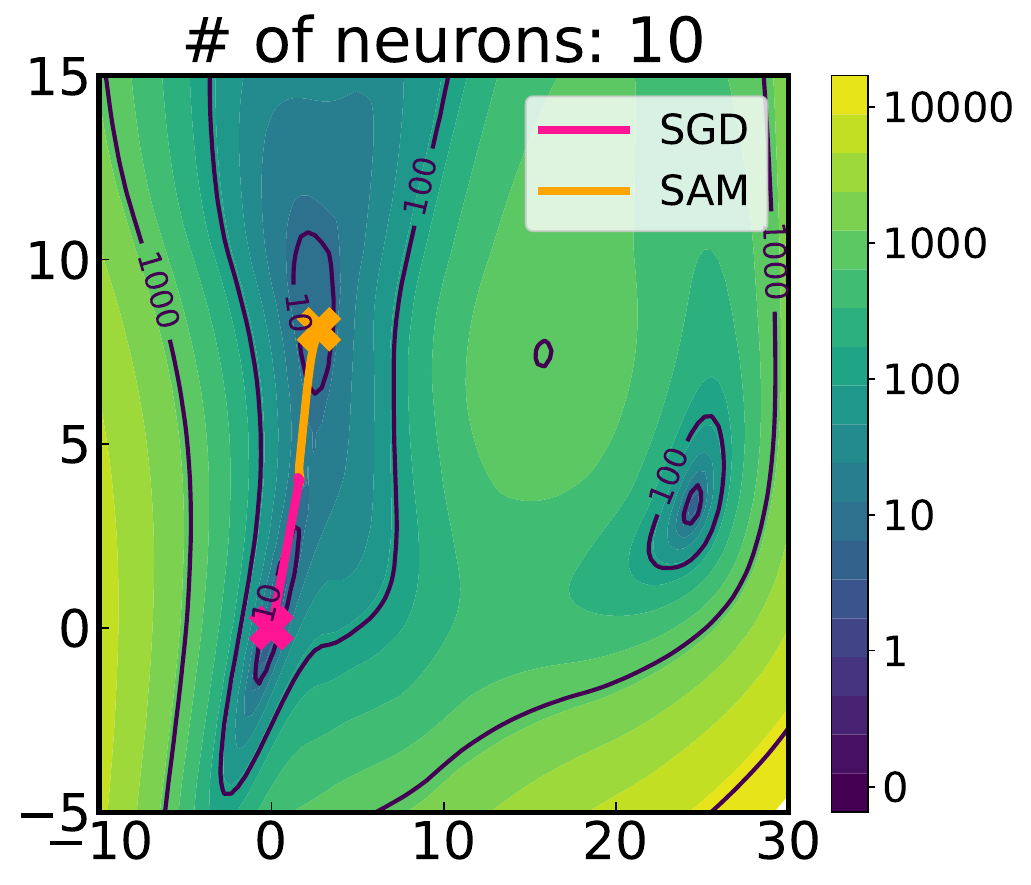}
  \includegraphics[width=0.24\linewidth]{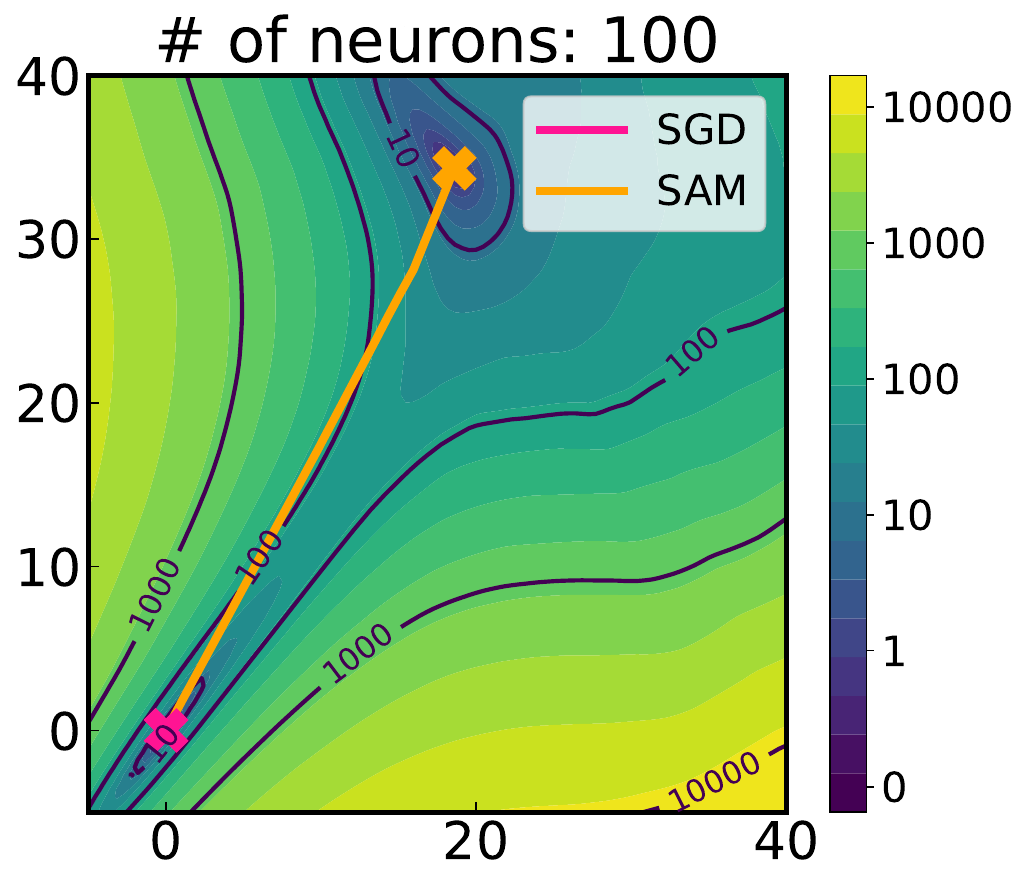}
  \includegraphics[width=0.24\linewidth]{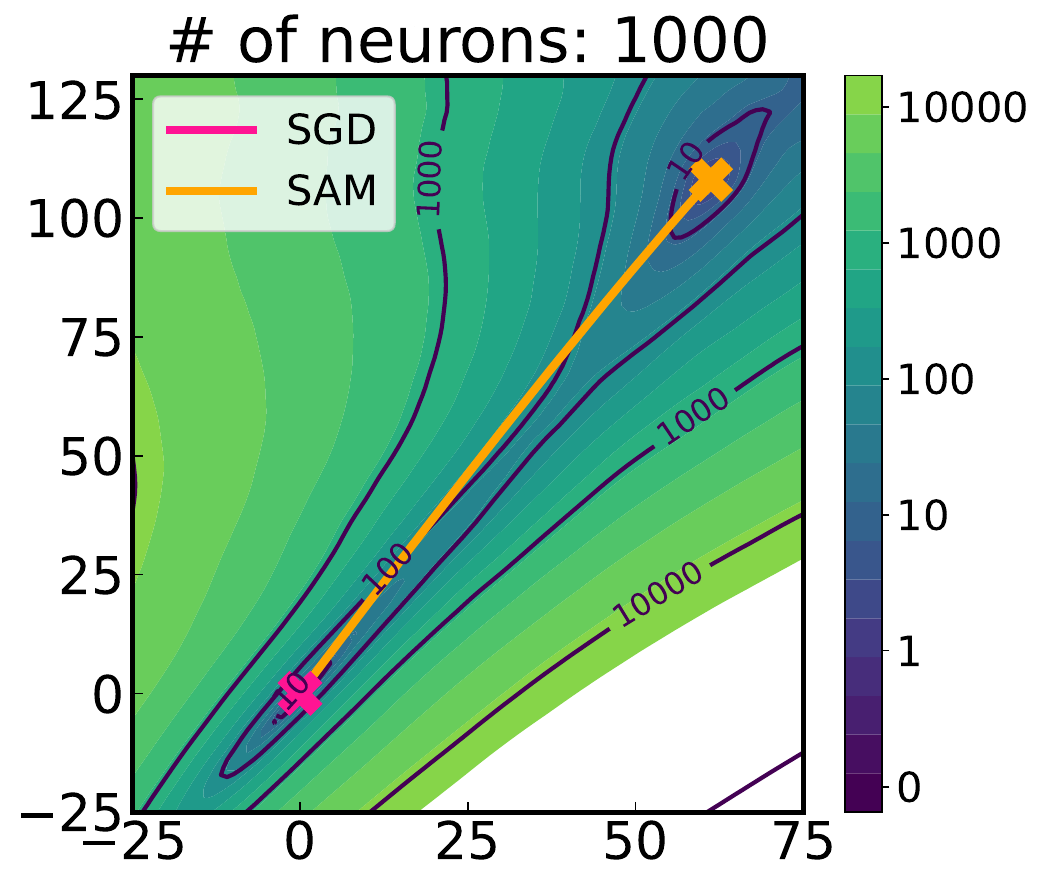}
  \caption{
  Optimization trajectories of SGD and SAM starting from the same initial point.
  SGD and SAM reach solutions in a similar basin for under/moderately-parameterized models, whereas they reach different solutions for overparameterized models, \ie, flatter region for SAM.
  }
  \label{fig:onehidden-trajectories}
\end{figure*}

\subsection{Implicit bias of SAM increases with overparameterization}
\label{sec:understanding-implicitbias}

While overparameterization can secure favorable conditions for SAM, it is not to be confused with guaranteeing the implicit bias of SAM taking effect.
In fact, we can further relate the implicit bias of SAM to the perturbation bound $\rho$ to bridge this gap.
Specifically, SAM can be interpreted as SGD on an implicitly regularized loss based on SDE (stochastic differential equation) modeling \citep{compagnoni2023sde}:
\begin{equation}
    \tilde{f}(x) \coloneqq f(x) + \rho \E\|\nabla f_\gamma (x)\|_2
\end{equation}
where $\gamma$ refers to some stochasticity.
This indicates that SAM becomes more regularized (\ie, the implicit bias is amplified) when $\rho$ increases.\footnote{This holds as long as $\rho$ is not too large, by which it might overshadow minimizing $f$ and implicitly bias the optimizer toward stationary points such as saddles and maxima.
Note that it reduces to standard SGD when $\rho=0$.
}

\begin{figure}[!t]
    \centering
    \centering
    \includegraphics[width=0.16\linewidth]{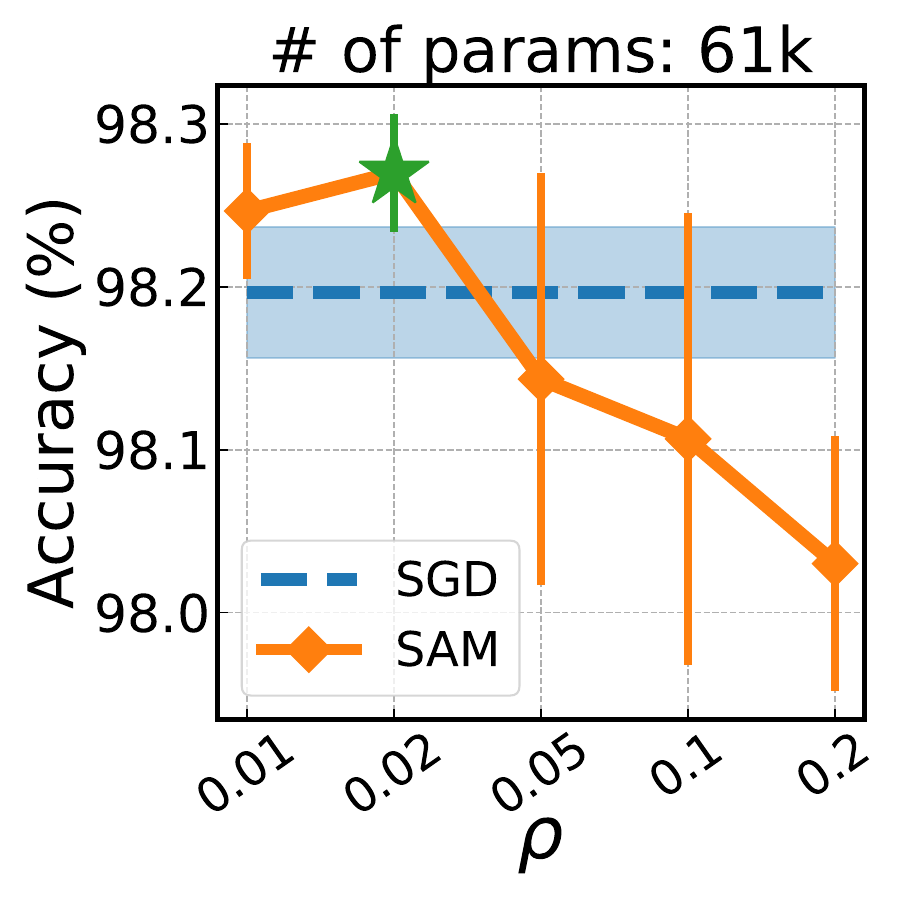}
    \includegraphics[width=0.16\linewidth]{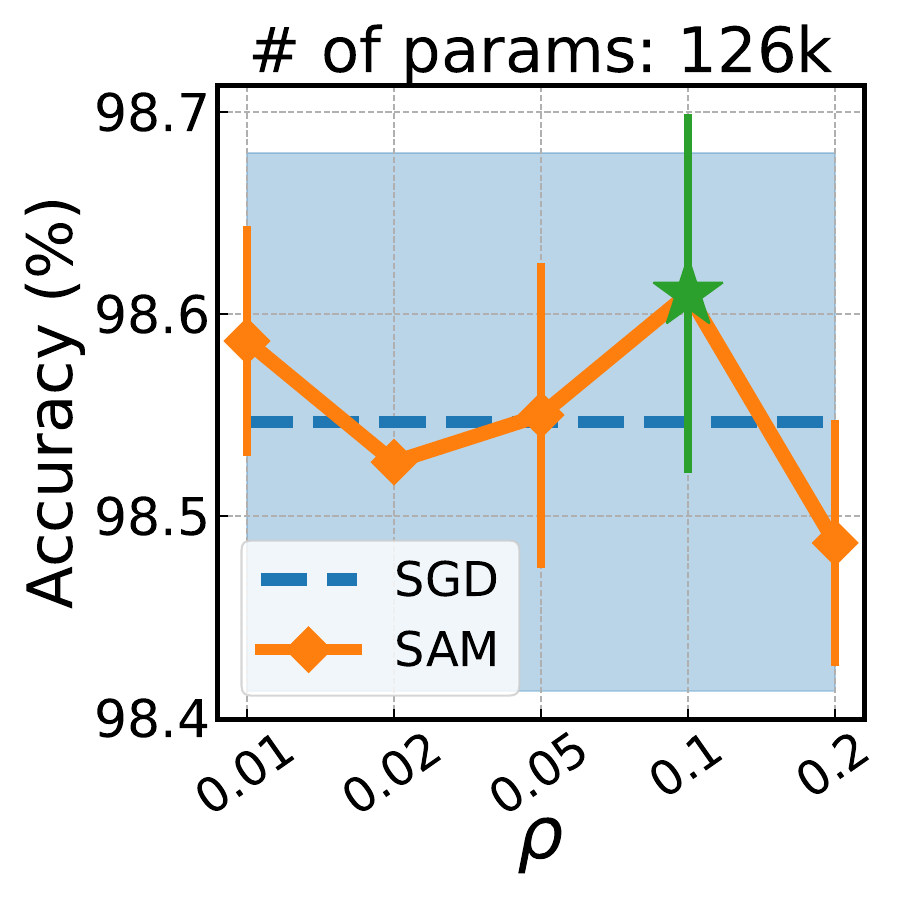}
    \includegraphics[width=0.16\linewidth]{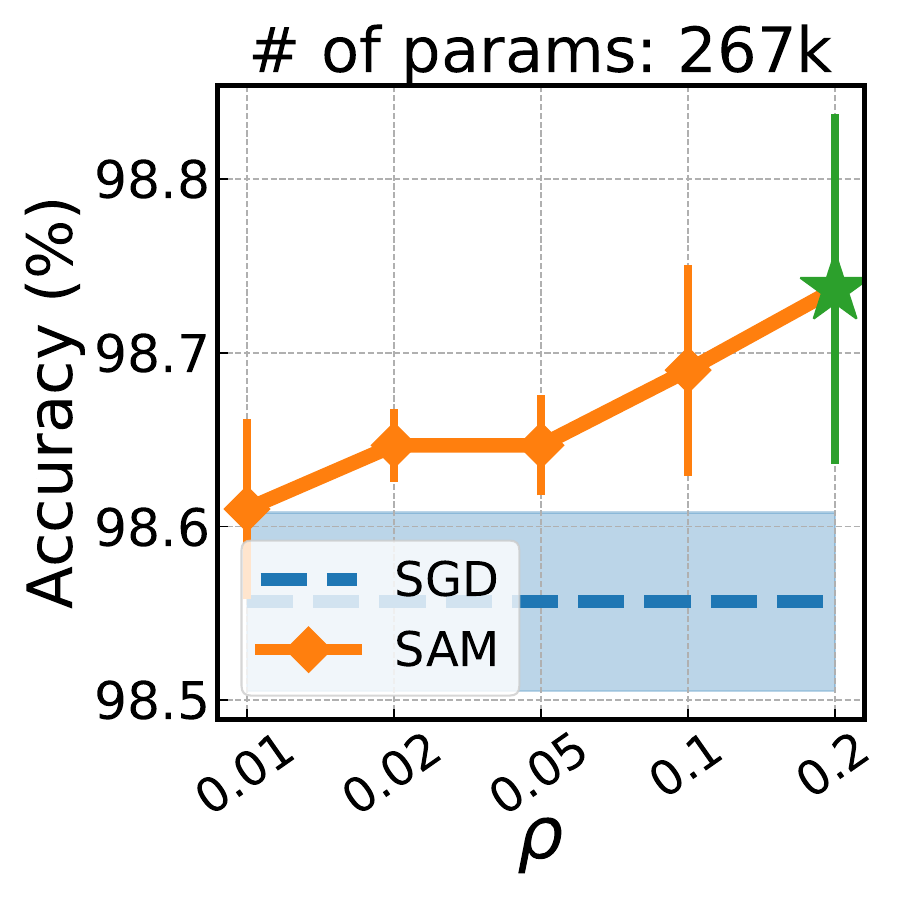}
    \includegraphics[width=0.16\linewidth]{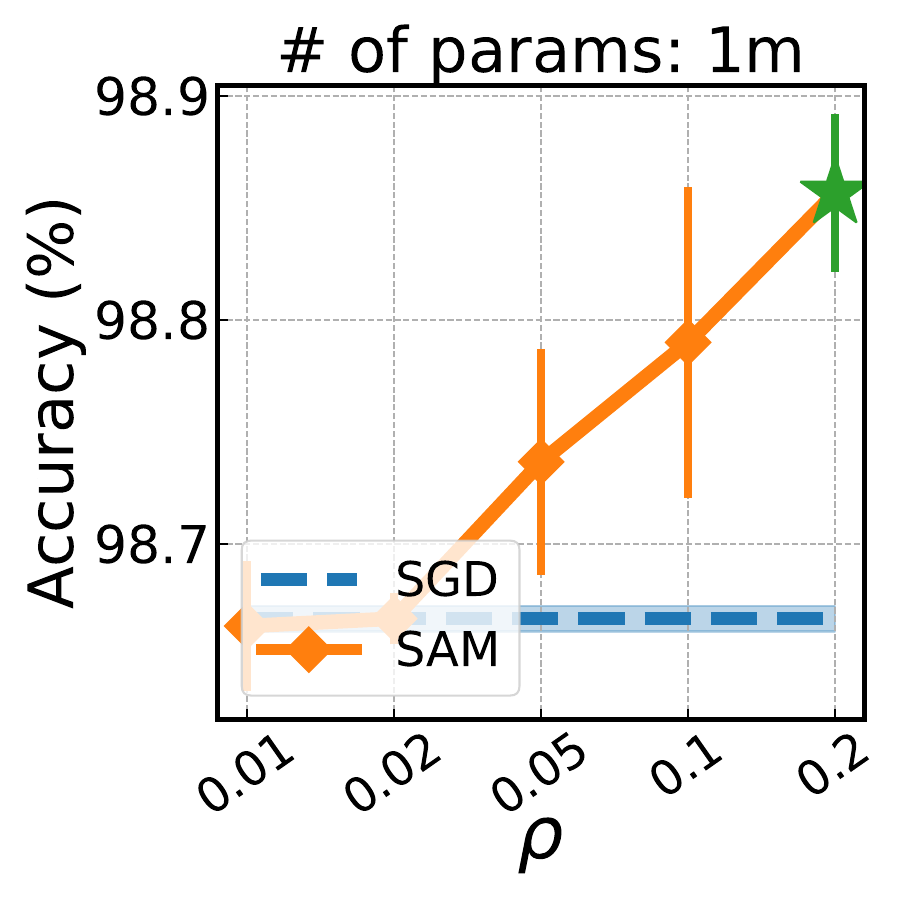}
    \includegraphics[width=0.16\linewidth]{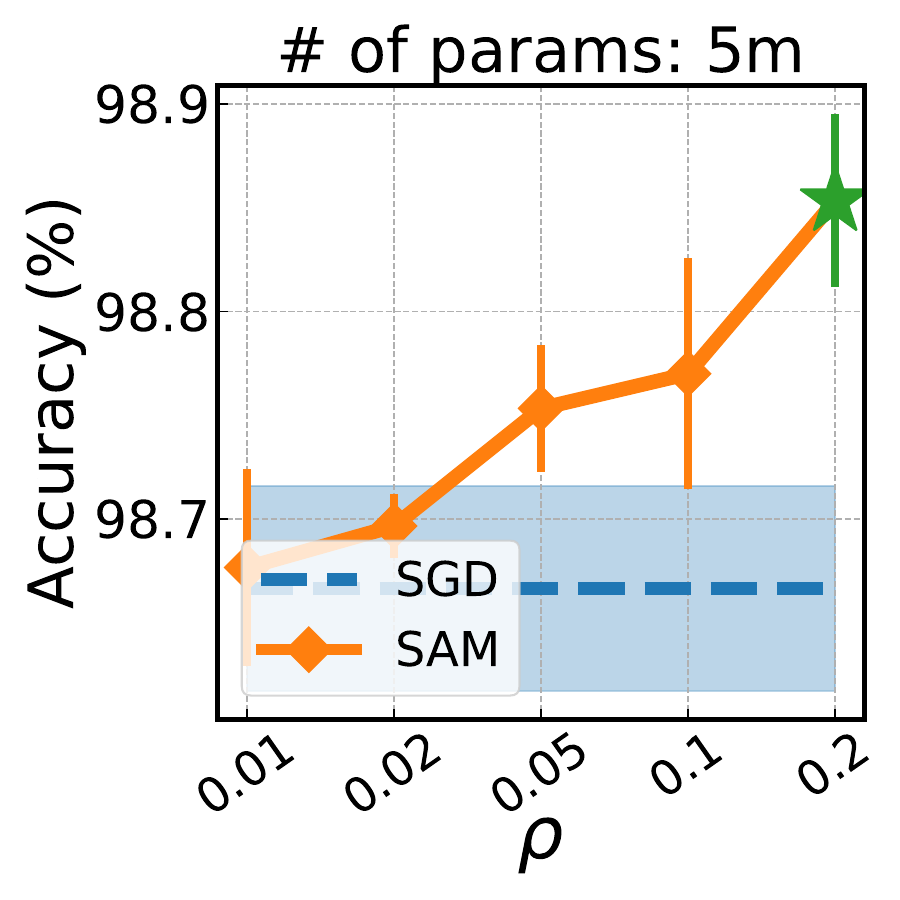}
    \includegraphics[width=0.16\linewidth]{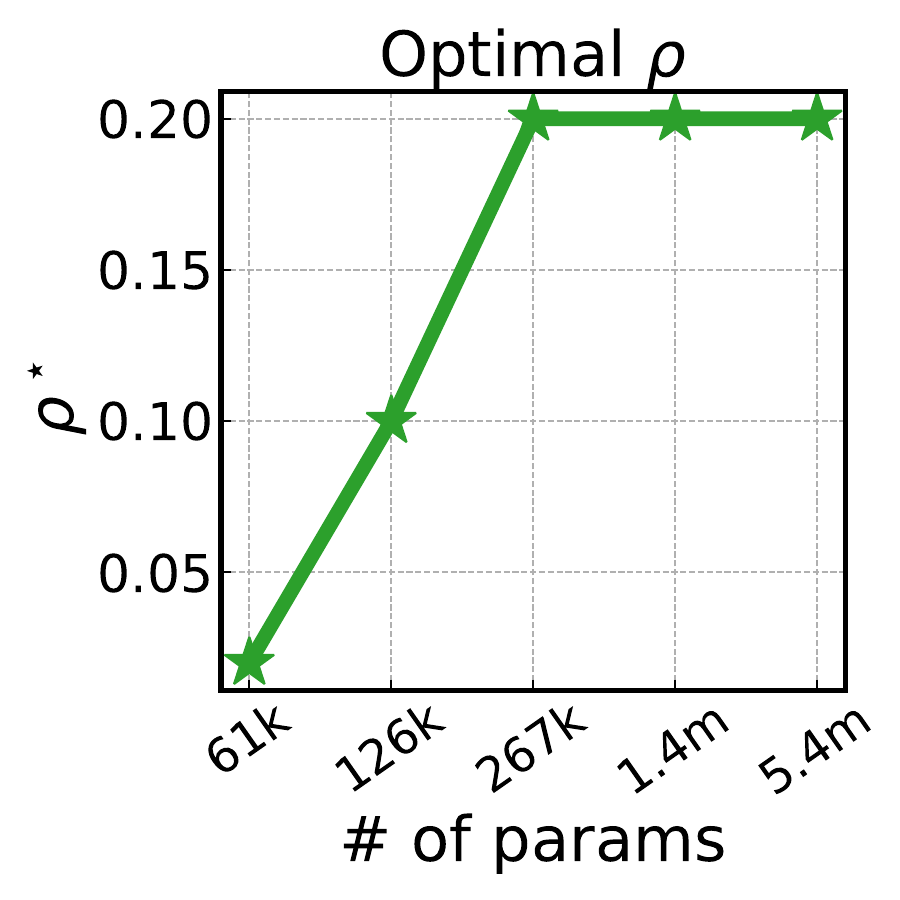}
    
    \includegraphics[width=0.16\linewidth]{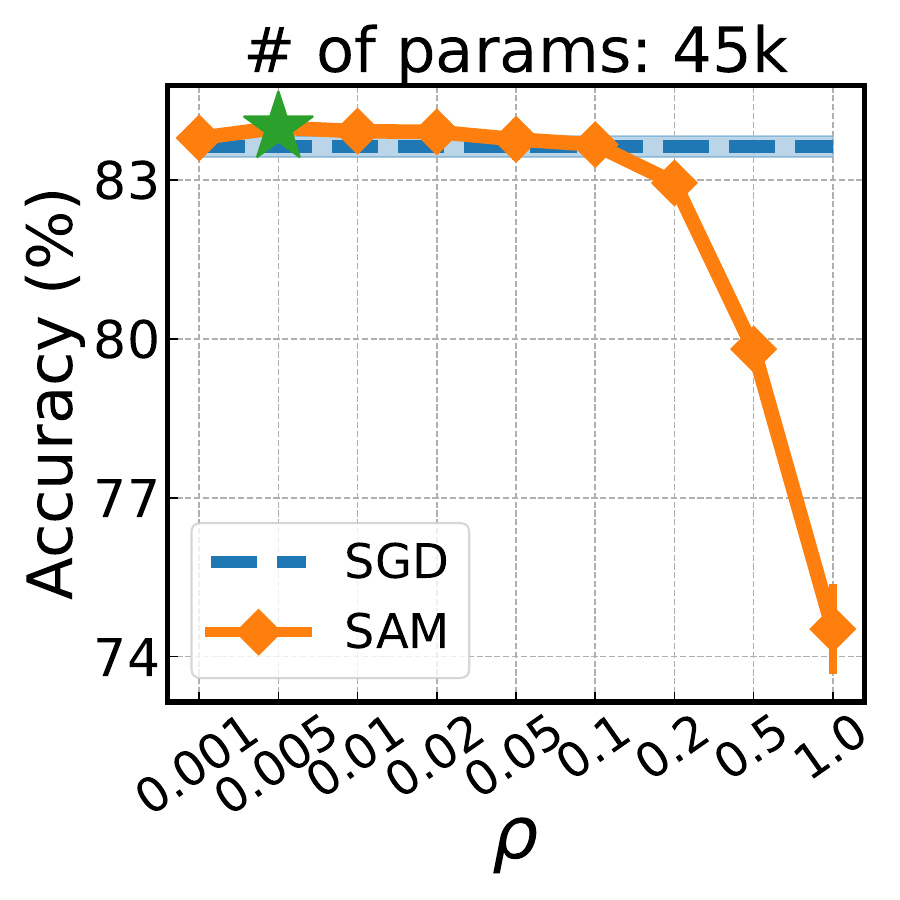}
    \includegraphics[width=0.16\linewidth]{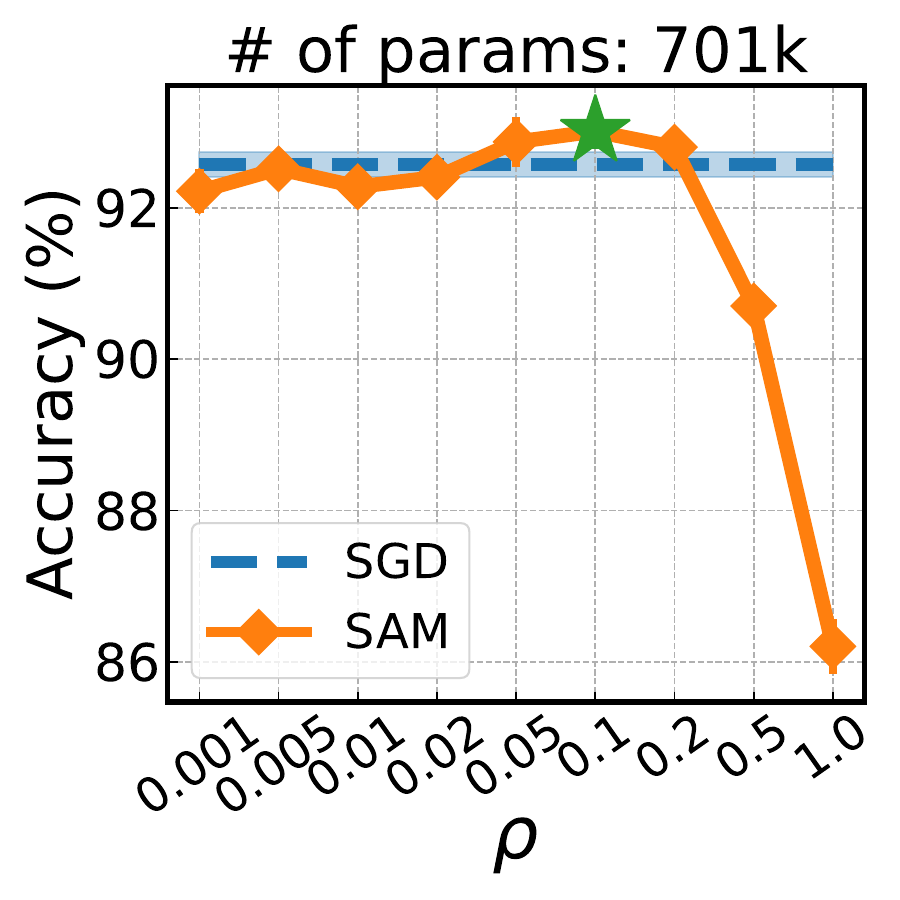}
    \includegraphics[width=0.16\linewidth]{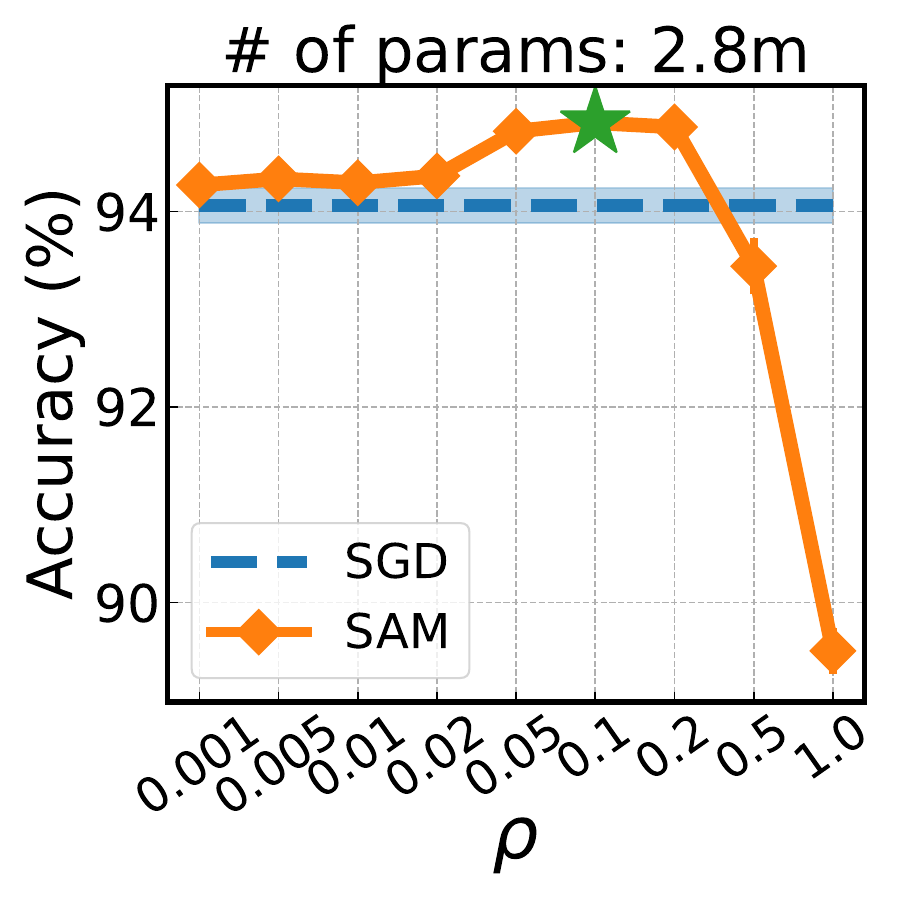}
    \includegraphics[width=0.16\linewidth]{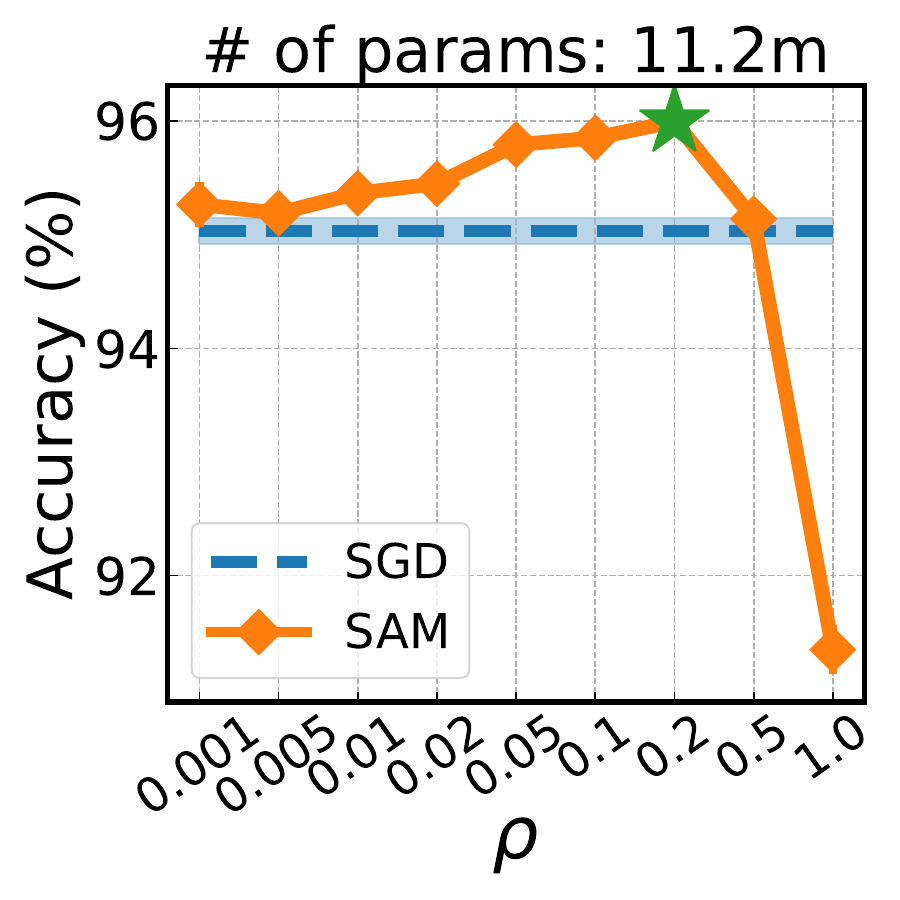}
    \includegraphics[width=0.16\linewidth]{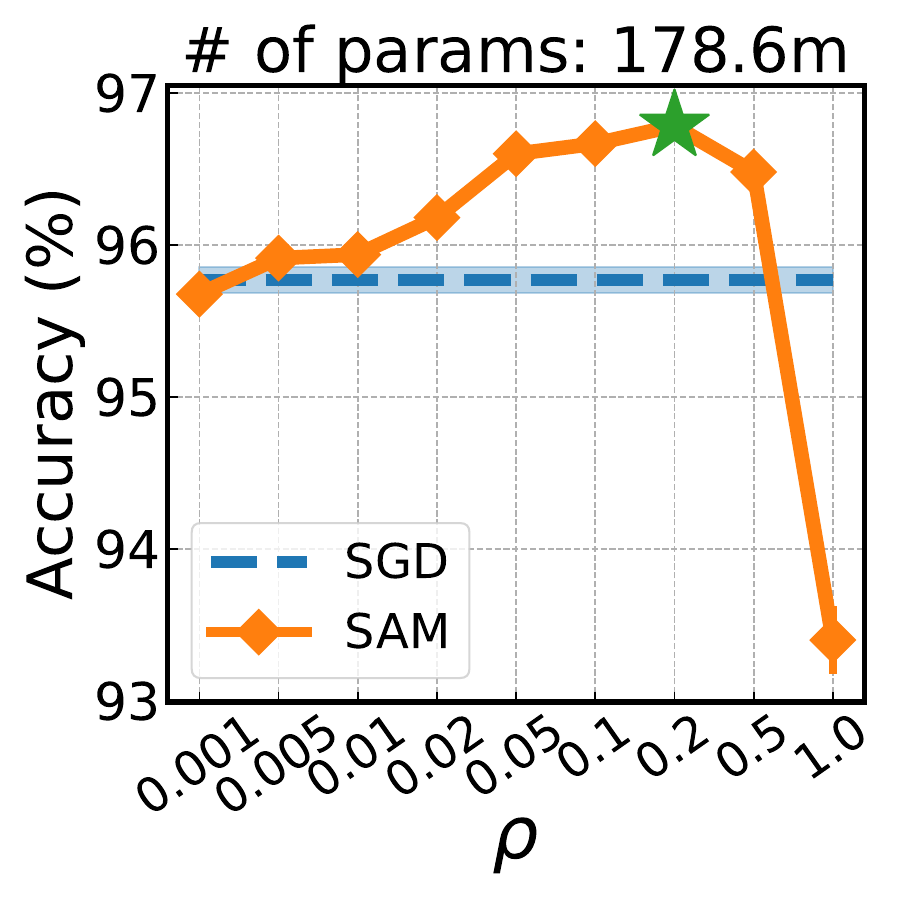}
    \includegraphics[width=0.162\linewidth]{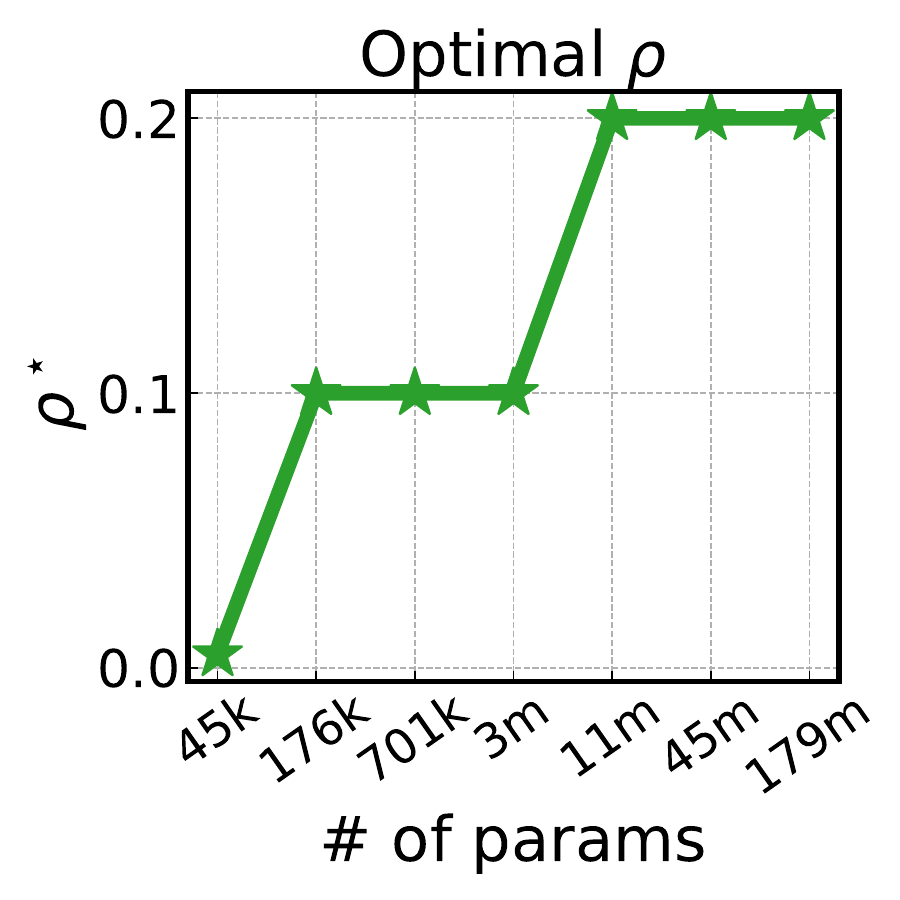}
  \caption{
    Validation accuracy versus $\rho$ for $3$-layer MLP trained on MNIST (top) and ResNet-18 trained on CIFAR-10 (bottom).
    $\rho^\star$ is located to be higher with more parameters.
  }
  \label{fig:result_overparam_each_rho}
\end{figure}

Our interest thus lies in seeing whether overparameterization has any effect on increasing $\rho$.
Since if that is the case, it indeed means that overparameterization puts more regularization on SAM.
We verify this by finding the empirically optimal perturbation bound $\rho^\star$ that yields the best generalization performance as we change the degree of overparameterization.
Specifically, we take a standard deep learning task and perform an extensive grid search to find $\rho^\star$.
The result is displayed in \cref{fig:result_overparam_each_rho}.

Indeed, it is observed that $\rho^\star$ tends to increase as the number of parameters increases;
\ie, seeing from left to right, $\rho$ value that yields highest accuracy (marked as green star \textcolor{ForestGreen}{$\star$}) tends to increase.
We confirm that this trend is consistently observed for various other workloads (See \cref{fig:result_overparam_each_rho_mnist,fig:result_overparam_each_rho_cifar,fig:result_overparam_each_rho_imagenet,fig:result_overparam_each_rho_sst} of \cref{app:add-rho} for more results).
This result is certainly encouraging since it supports that \emph{the generalization benefit of SAM via implicit regularization can indeed increase by overparameterization}.

Additionally, we can develop a conceptual account of why $\rho^\star$ increases with overparameterization.
First, if we consider the expected effect of perturbation $\epsilon \in\mathbb{R}^d$ of size $\rho$ on individual parameters simply as $\mathbb{E}_k[\epsilon_k^2]=||\epsilon||_2^2/d=\rho^2/d$, we can see that $\mathbb{E}_k[\epsilon_k^2] \rightarrow 0$ as $d\rightarrow\infty$, which implies that SAM would eventually have almost no effect on each parameter as the model scales unless $\rho$ is also increased.

Also, the Lipschitz bound on the gradients reveals that $\left\lVert \nabla f\left(x+\epsilon\right) - \nabla f(x)\right\rVert_2 \leq \beta \left\lVert x+\epsilon - x\right\rVert_2 = \beta \rho$, indicating that the SAM gradient becomes more similar to the original gradient as the model gets smoother (\ie, smaller smoothness constant $\beta$) with increasing size, requiring larger perturbation bound to achieve similar levels of perturbation effect.
These hold under the assumption that overparameterization makes the model smoother, which we empirically confirm in \cref{fig:result_empirical_lipschitz}.

%% file: text/5_practical.tex
\section{Further merits and caveats of overparameterization}
\vspace{-0.5em}
\label{sec:practical}

In this section, we present further merits and some caveats of overparameterization.
Specifically, we show that the overparameterization benefit of SAM continues to exist and becomes more evident under label noise or sparsity.
We also discover that sufficient regularization is required to attain the benefit.
These results could serve as a guidance to employ SAM in practice.

\paragraph{Overparameterization secures the robustness of SAM to label noise}
\label{sec:experiments-main-label_noise}

In practice, deep learning models are often trained on noisy data \citep{song2022learning}.
To examine whether the overparameterization benefit for SAM continues to exist in this scenario, we introduce some label noise to training data \citep{angluin1988learning,natarajan2013learning} and see how SAM responds.
The results are reported in \cref{fig:resut_labelnoise_noiserate}.
Overall, we find SAM benefits from overparameterization significantly more than SGD in the presence of label noise.
Precisely, the accuracy improvement made by SAM keeps on increasing as the model has more parameters, whereas the improvement over SGD is marginal for less parameterized models.
Notably, this trend is more pronounced with a higher noise level;
\eg, it rises from $5\%$ to nearly $50\%$ at the highest noise rate.
Notably, it is previously known that SAM is robust to label noise compared to SGD \citep{foret2020sharpness,baek2024sam, huang2023robust, zou2024towards, kim2023fantastic}, and yet, this result newly reveals that overparameterization plays a profound role in securing the robustness of SAM.

\begin{figure}[!t]
      \centering
      \begin{subfigure}{0.19\linewidth}
          \includegraphics[width=\linewidth]{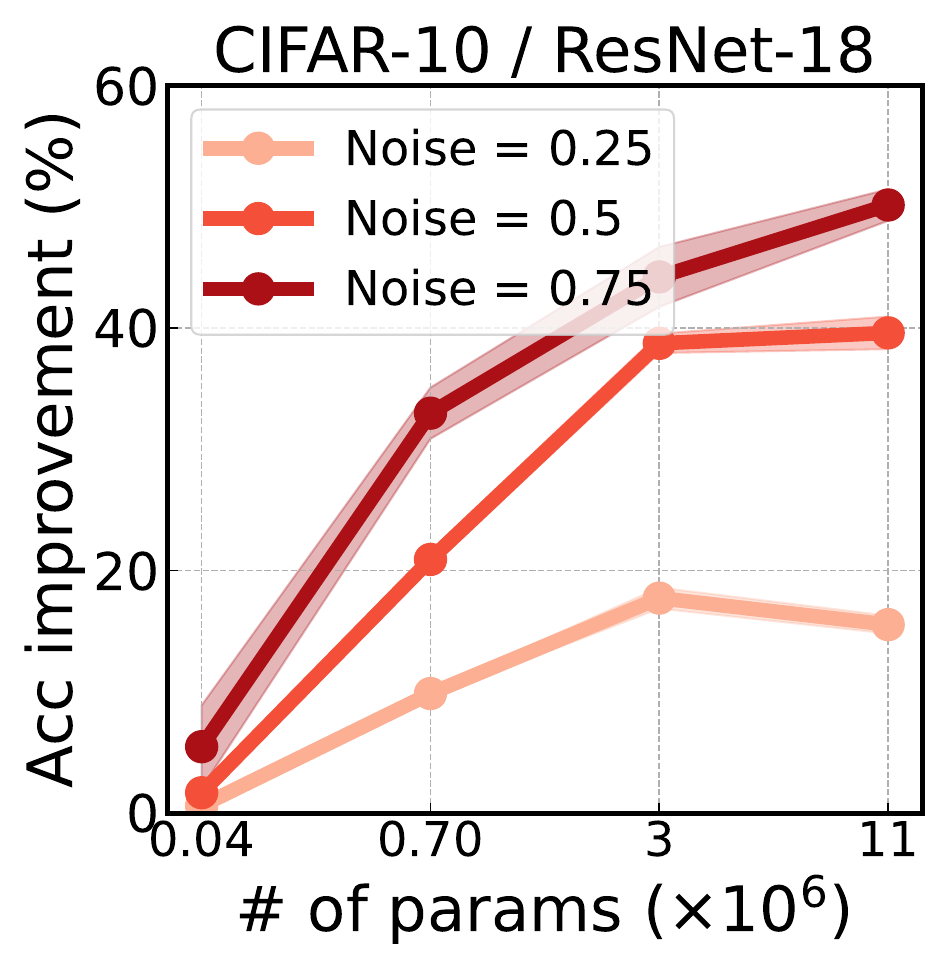}
          \caption{Label noise}
          \label{fig:resut_labelnoise_noiserate}
      \end{subfigure}
      \hspace{1em}
      \begin{subfigure}{0.199\linewidth}
      \includegraphics[width=\linewidth]{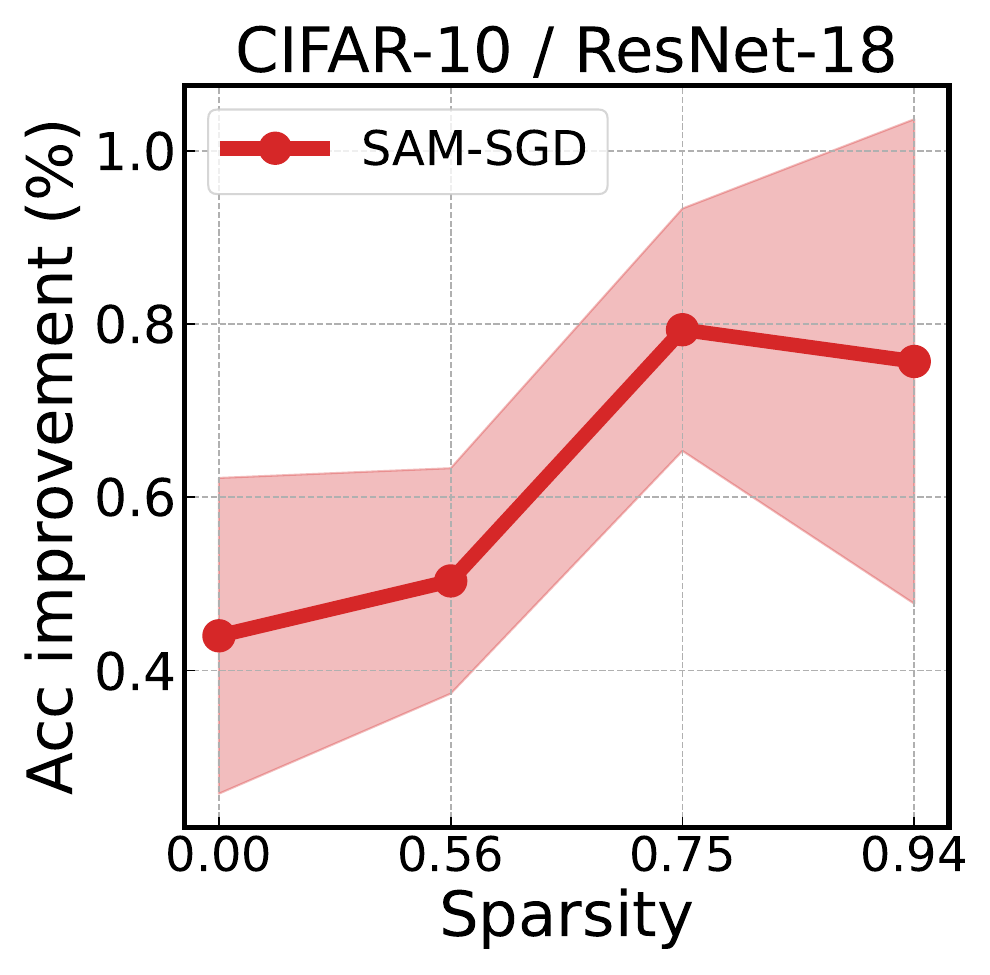}
          \caption{Sparsity}
          \label{fig:result_largesparse_cifar_resnet_dense16}
      \end{subfigure}
      \hspace{1em}
      \begin{subfigure}{0.19\linewidth}
          \includegraphics[width=\linewidth]{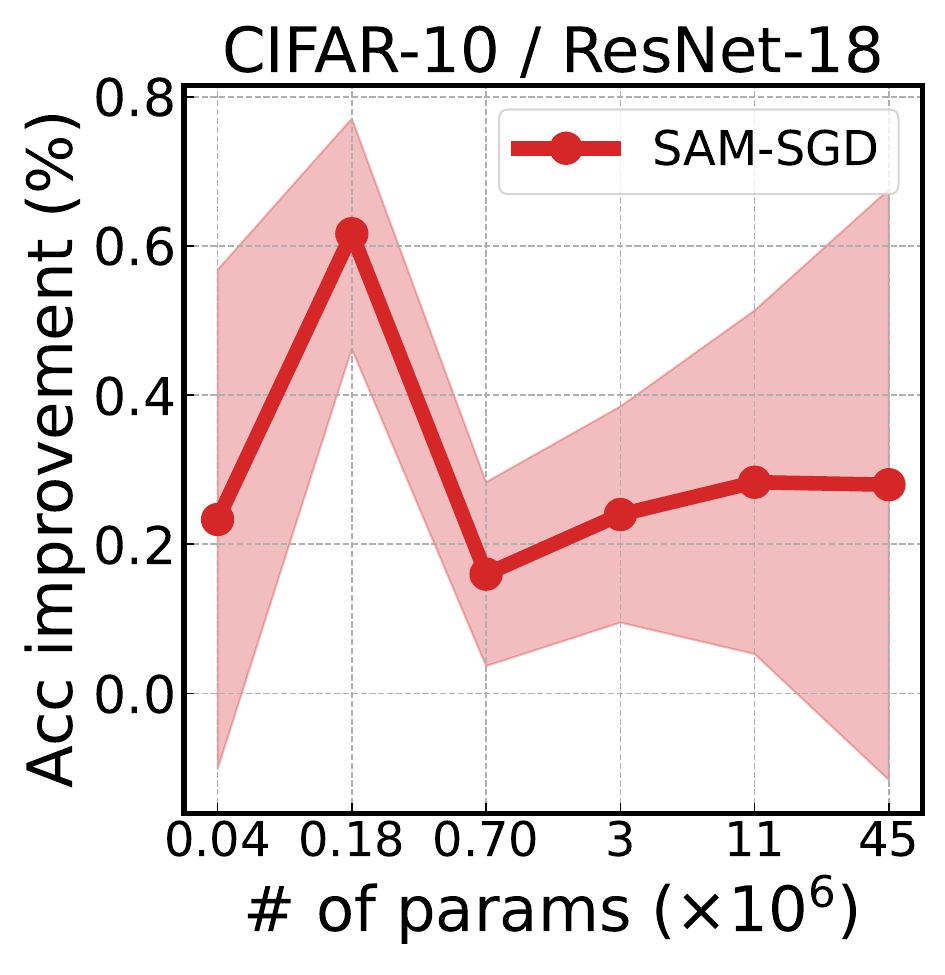}
            \caption{w/o weight decay}
            \label{fig:result_overparam_diff_resnet_wo_wd}
      \end{subfigure}
      \begin{subfigure}{0.164\linewidth}
          \includegraphics[width=\linewidth,trim={1em 0 1em 0},clip]{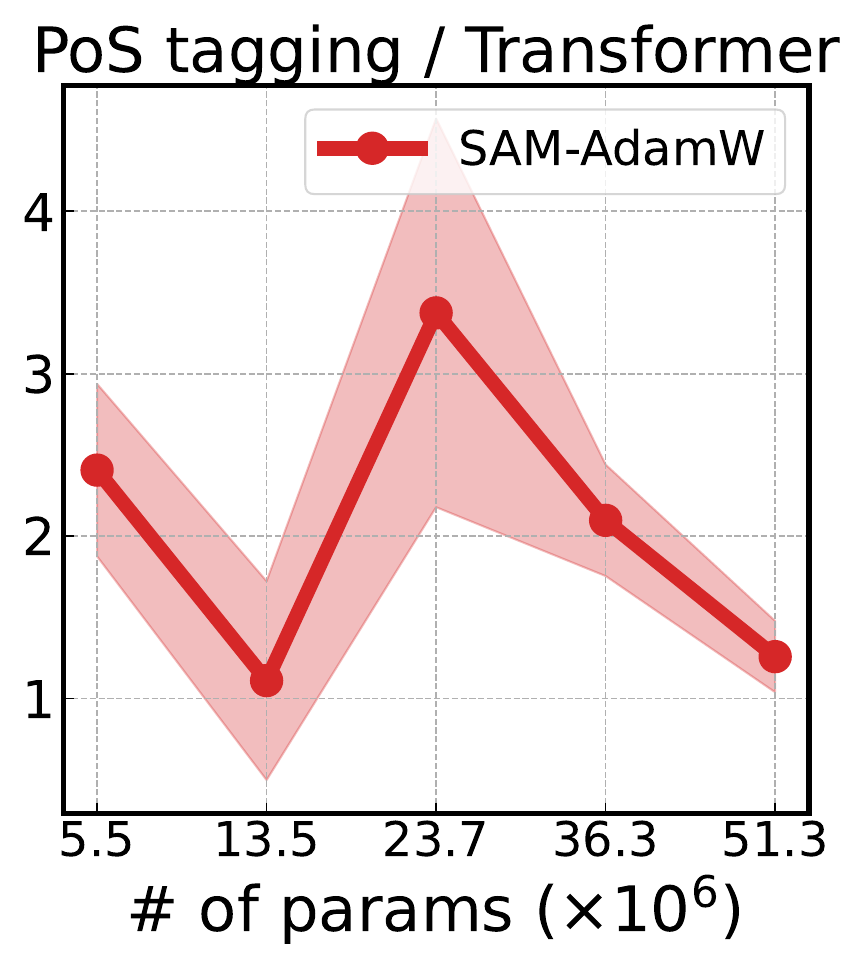}
            \caption{w/o early stop.}
            \label{fig:result_overparam_diff_pos_wo_es}
      \end{subfigure}
      \begin{subfigure}{0.167\linewidth}
          \includegraphics[width=\linewidth,trim={1.3em 0 0 0},clip]{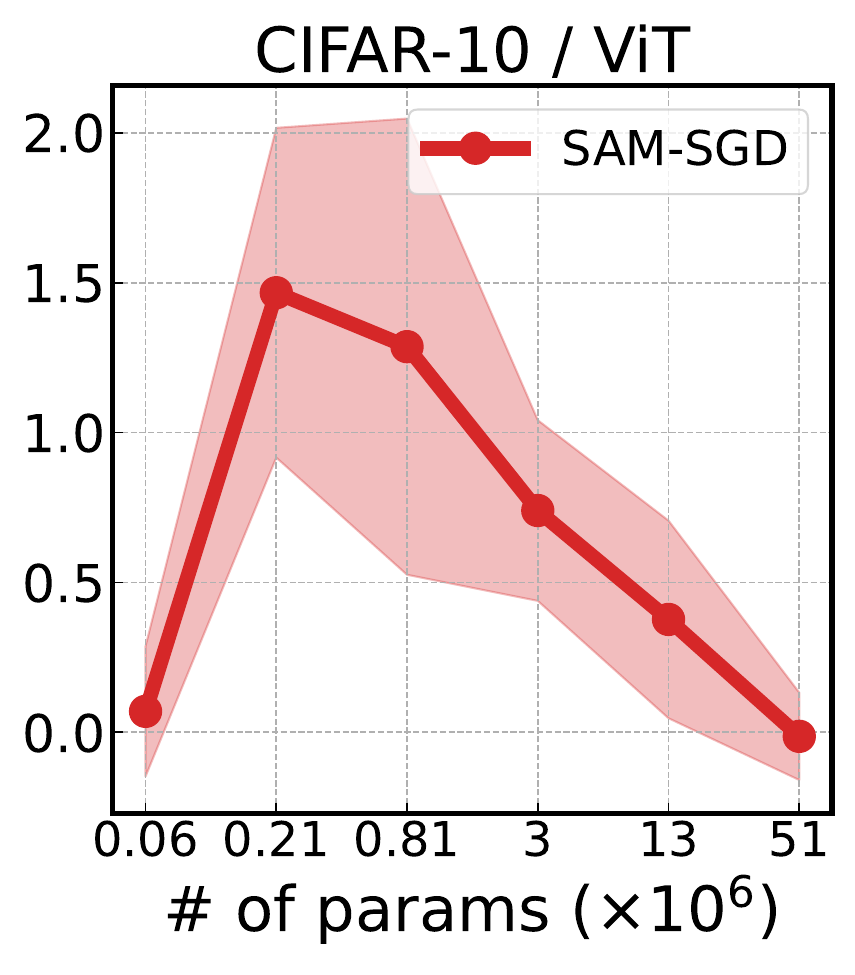}
            \caption{w/o induc. bias}
            \label{fig:result_overparam_diff_cifar_vit}
      \end{subfigure}
      \caption{
        Effect of (a) label noise, (b) sparsity, and (c-e) regularization on SAM.
        (a) The benefit of SAM is more pronounced with a higher noise level.
        (b) The improvement by SAM tends to increase in large sparse models compared to their small dense counterparts.
        (c-e) SAM does not always benefit from overparameterization without sufficient regularization.
        See \cref{fig:result_labelnoise,fig:result_largesparse,fig:result_largesparse_rho,fig:result_overparam_overfit} in \cref{app:add_practical} for more results.
      }
\end{figure}

\paragraph{SAM benefits from sparse overparameterization.}
\label{sec:experiments-main-large_sparse}

There has been a recent interest in employing sparsity to train large models to alleviate the computation and memory costs \citep{hoefler2021sparsity,mishra2021accelerating}.
To test the effect of overparameterization on SAM under this setting, we introduce a varying degree of sparsity to an overparameterized model at initialization \citep{leesnip} such that the number of parameters matches the original dense model.
The results are reported in \cref{fig:result_largesparse_cifar_resnet_dense16}.
We observe that the generalization improvement tends to increase as the model becomes more sparsely overparameterized;
more precisely, the average accuracy improvement increases from $0.4\%$ in the small dense model to around $0.8\%$ in the large sparse model.
This result suggests that one can consider taking sparsification more actively when employing SAM.

\paragraph{Sufficient regularization is needed to secure the benefit of overparameterization.}
\label{sec:experiments-main-regularization}

We also investigate whether the overparameterization benefit for SAM continues to exist when models are prone to overfitting due to insufficient regularization \citep{ying2019overview}.
Specifically, we evaluate three cases:
(a) without weight decay,
(b) without early stopping, and
(c) without sufficient inductive bias.\footnote{We train ViTs that are not pre-trained on a massive dataset, which is known to lack inductive biases inherent to CNNs and thus more prone to overfitting \citep{lee2021vision,chen2022transmix}.}
The results are reported in \cref{fig:result_overparam_diff_resnet_wo_wd,fig:result_overparam_diff_pos_wo_es,fig:result_overparam_diff_cifar_vit}.
We observe that the generalization improvement does not increase by simply adding more parameters.
The results indicate that some level of regularization is required in practice to attain the overparameterization benefit for SAM.

%% file: text/6_theory.tex
\section{Other effects of overparameterization: Theoretical aspects} \label{sec:theory}

Thus far, we focused on empirically exploring how increasing number of parameters influences SAM, and discovered critical improvements in its generalization benefits.
However, existing theoretical analyses on overparameterization also hint at other types of positive influences on different aspects of SAM such as convergence \citep{ma2018power,vaswani2019fast} and implicit bias \citep{neyshabur2017implicit,zhang2017understanding}.
Despite this, we find that there is little work on explicitly verifying whether these influences extend to SAM, however.

To fill this gap, we develop theoretical analyses of the effect of overparameterization on SAM\footnote{We use an unnormalized version of SAM: $x_{t+1} = x_{t} - \eta \nabla f \left(x_t+\rho\nabla f(x_t)\right)$, an empirically similar variant of SAM often adopted to simplify proofs \citep{andriushchenko2022towards, compagnoni2023sde}.} in this section.
Specifically, we show that (i) linearly stable minima for SAM have more uniform Hessian moments compared to SGD (\cref{sec:sam-stability}), and (ii) SAM can converge much faster (\cref{sec:convergence}), all when the model is overparameterized.

To characterize overparameterization, we adopt a widely accepted definition: a model is overparameterized if it possesses more parameters than necessary to fit the entire training data or achieve zero training loss \citep{ma2018power, belkin2018understand, belkin2019reconciling, neyshabur2019role, nakkiran2020deep, nakkiran2020optimal}---that is, any model capable of interpolation.
We formalize this via the following \emph{interpolation} assumption:
\begin{definition} \label{def:interpolation}
    (Interpolation) Let $f(x)=\sum^n_{i=1} f_i(x)$. There exists $x^\star$ s.t. $f_i(x^\star)=0$ and $\nabla f_i(x^\star)=0$ for $i=1,\hdots,n$.
\end{definition}
Crucially, this implies that there exists a fixed point $x^\star$ for stochastic gradient-based optimizers, which comes as an important property in the following two sections.

We leave a clear note here that the aim of these analyses is to complement, rather than directly support \cref{sec:experiments-main,sec:understanding}, by outlining theoretically guaranteed benefits of overparameterization on SAM.
We discuss more about the limitations later in \cref{sec:discussion}.

\subsection{SAM escapes sharp minima with non-uniform Hessian}
\label{sec:sam-stability}

Here we demonstrate that SAM escapes until it encounters minima of a certain level of flatness and uniform Hessian moments that are stricter compared to SGD.
To this end, we employ linear stability analysis \citep{wu2018sgd, wu2022alignment}, which aims to derive specific conditions a minimum should satisfy in order for a given optimizer to remain stable and not escape from it.

We first define linear stability as follows:
\begin{definition} \label{def:linear-stability}
    (Linear stability)
    Consider a general iterative first-order optimizer $x_{t+1} = x_t - G(x_t)$.
    A minimizer $x^\star$ is called linearly stable if there exists a constant $C$ such that
    $$\mathbb{E}[\| \tilde{x}_t-x^\star\|^2] \leq C \| \tilde{x}_0-x^\star\|^2$$
    for all $t > 0$ under the linearized dynamic near $x^\star$: $\tilde{x}_{t+1} = \tilde{x}_t -  \nabla G(x^\star) (\tilde{x}_t - x^\star)$, \ie, if it does not deviate far from $x^\star$ once arrived near a fixed point.
\end{definition}
Here, the linearized dynamic $\tilde{x}_t$ appears when the iterate $x_t$ approaches sufficiently near $x^\star$ such that the loss becomes approximately quadratic, with the existence of the fixed point $x^\star$ implied by the interpolation assumption in \cref{def:interpolation}.

With this, we provide the stability condition that minima should satisfy for a stochastic SAM to converge in the following theorem:
\begin{theorem} \label{thm:sam-stability} 
    Let us assume $x^\star=0$ without loss of generality.
    Then $x^\star$ is linearly stable for a stochastic SAM if the following is satisfied:
    \begin{equation} \label{eq:sam-stability-condition}
    \begin{split}
     \lambda_{\textup{max}} & \left((I - \eta H - \eta \rho H^2)^2 + \eta(\eta-2\rho) (M_2-H^2) \right.\\
     & \hspace{1em} \left.  + 2\eta^2\rho (M_3-H^3) + \eta^2\rho^2 (M_4 -H^4) \right) \leq 1
    \end{split}
    \end{equation}
    where $H = \frac{1}{n} \sum_{i=1}^n H_i$ and $M_k = \frac{1}{n}\sum_{i=1}^n H_i^k$ are the average Hessian and the $k$-th moment of the Hessian at $x^\star$ over $n$ training data.
    Subsequently as a necessary condition of (\ref{eq:sam-stability-condition}) it follows that
    \begin{align} \label{eq:stability-necessary}
        \begin{split}
            & 0 \leq a (1 + \rho a) \leq \frac{2}{\eta}, \quad 0 \leq s_2^2 \leq \frac{1}{\eta (\eta - 2 \rho)}, \quad \\
            & 0 \leq s_3^3 \leq \frac{1}{2 \eta^2 \rho}, \quad 0 \leq s_4^4 \leq \frac{1}{\eta^2 \rho^2},
        \end{split}
    \end{align}
    where $a = \lambda_{\text{max}}(H), s_k = \lambda_{\text{max}}((M_k - H^k)^{1/k})$ are the sharpness and the non-uniformity of the Hessian measured with the $k$-th moment, respectively.
\end{theorem}

The detailed proof of the theorem is provided in \cref{app:sam-stability}.

Our result (\ref{eq:stability-necessary}) suggests that SAM requires less sharp minima and more uniformly distributed Hessian moments to achieve linear stability (provided that $\rho > 0$) compared to those of SGD \citep{wu2018sgd}, \ie, when $\rho\rightarrow 0$ in (\ref{eq:stability-necessary}).
While a similar result is shared by a concurrent work of \citet{behdin2023msam}, we further ensure that higher-order terms of Hessian moments are bounded, and interestingly, it becomes tighter for a larger $\rho$.
To corroborate our result, we measure the empirical sharpness and non-uniformity of Hessian.
The results are reported in \cref{fig:landscape,fig:stability_uniformity}.

\begin{figure*}[!t]
  \begin{subfigure}{0.32\linewidth}
      \centering
      \includegraphics[width=0.48\linewidth,trim={1.3cm 0cm 0.8cm 0cm},clip]{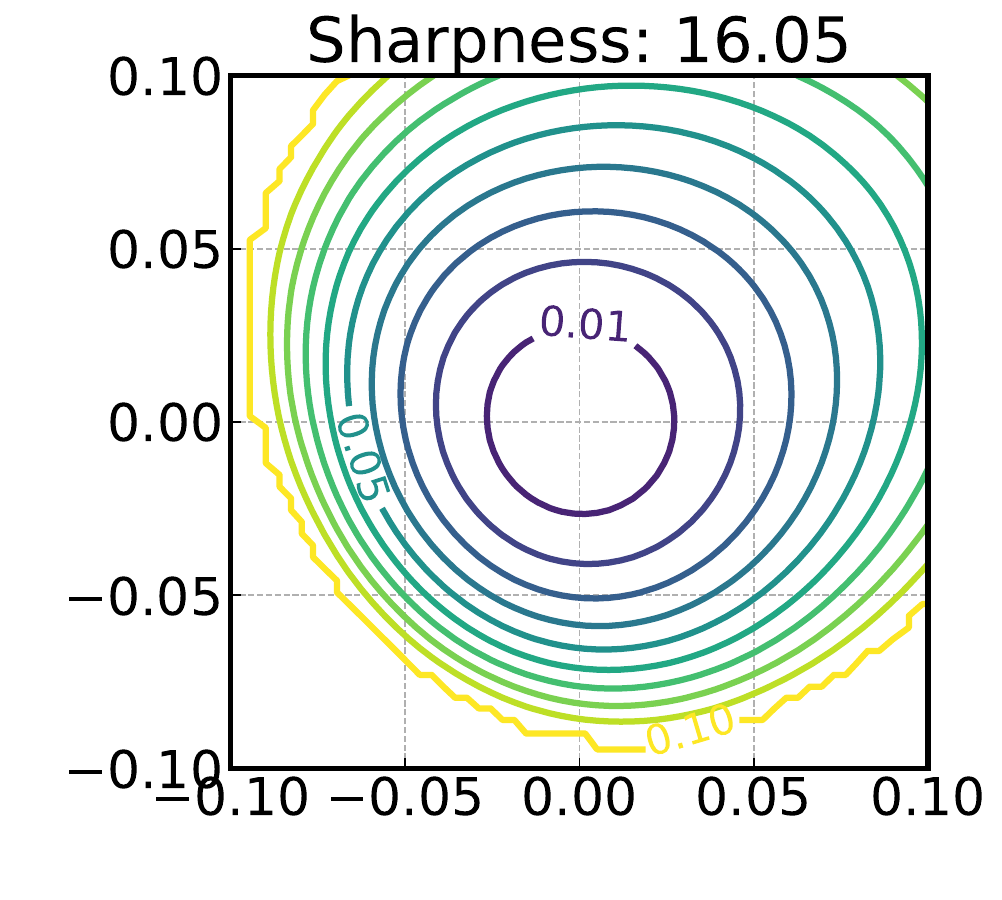}
      \includegraphics[width=0.48\linewidth,trim={1.3cm 0cm 0.8cm 0cm},clip]{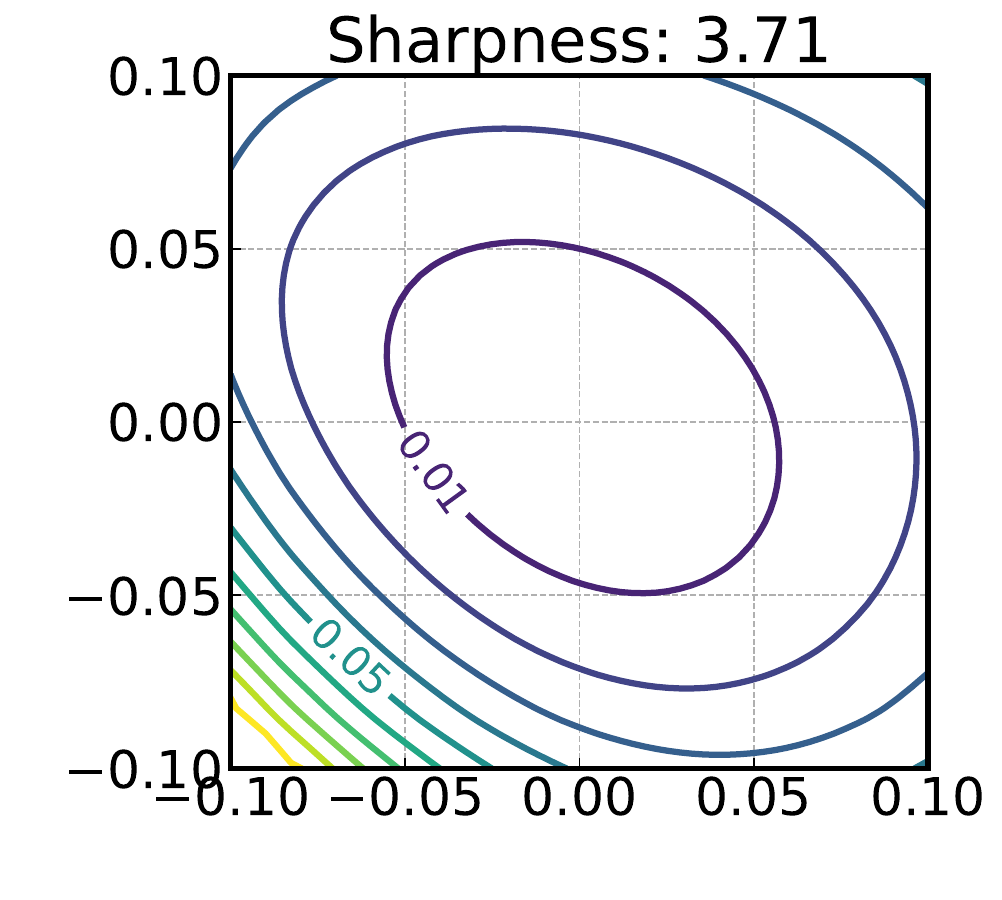}
      \caption{Landscape (SGD vs. SAM)}
      \label{fig:landscape}
  \end{subfigure}
  \hspace*{\fill}
  \begin{subfigure}{0.162\linewidth}
      \centering
      \includegraphics[width=\linewidth,trim={0.1cm 0.05cm 0.1cm 0},clip]{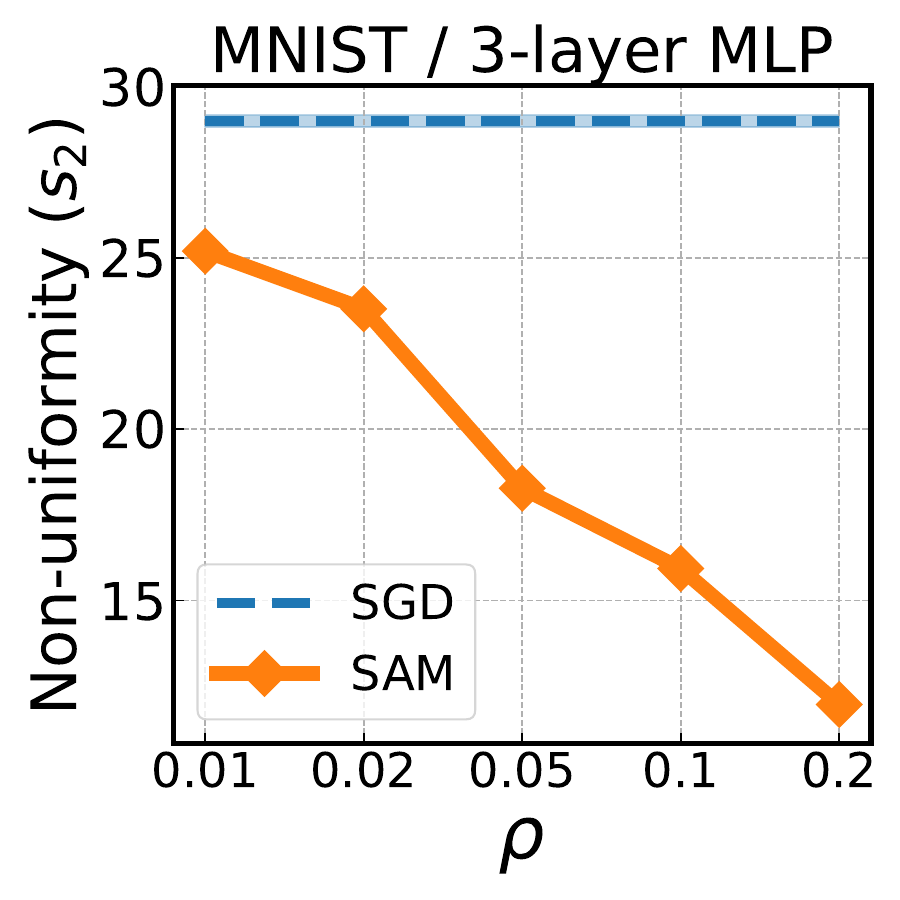}
      \vspace{-1.35em}
      \caption{Non-uniformity}
      \label{fig:stability_uniformity}
  \end{subfigure}
  \begin{subfigure}{0.5\linewidth}
      \centering
      \includegraphics[width=0.32\linewidth, trim={1em 1.2em 1em -3em}, clip]{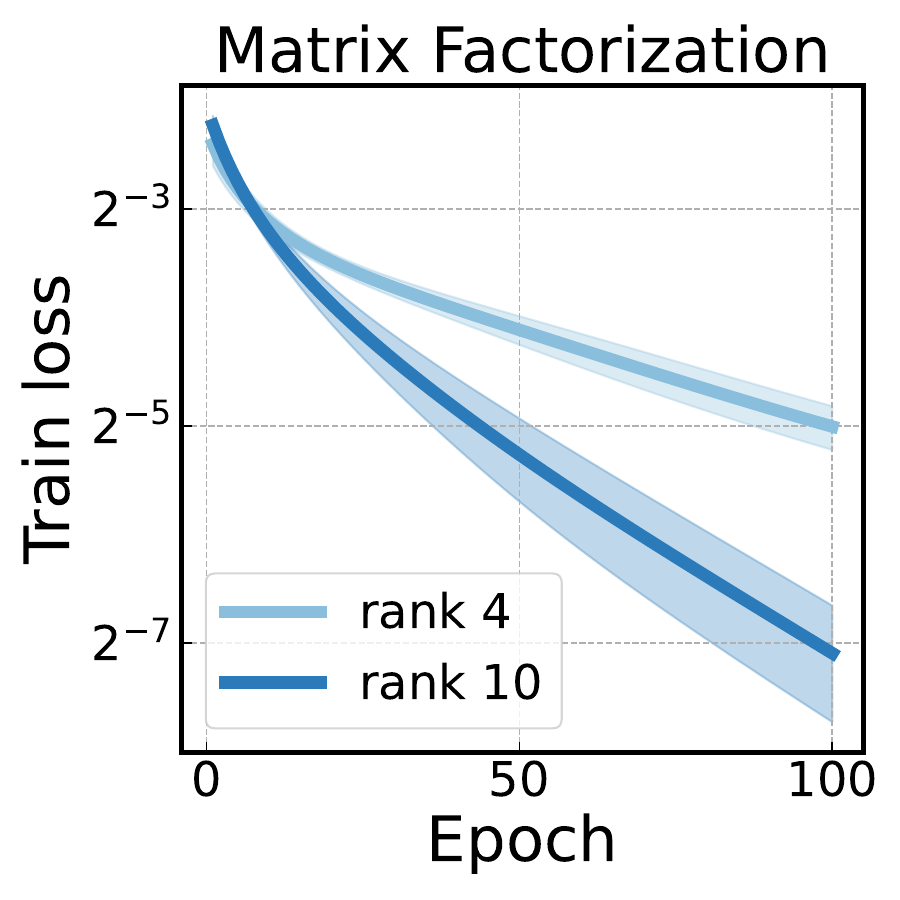}
      \includegraphics[width=0.32\linewidth, trim={1em 1.2em 1em 0}, clip]{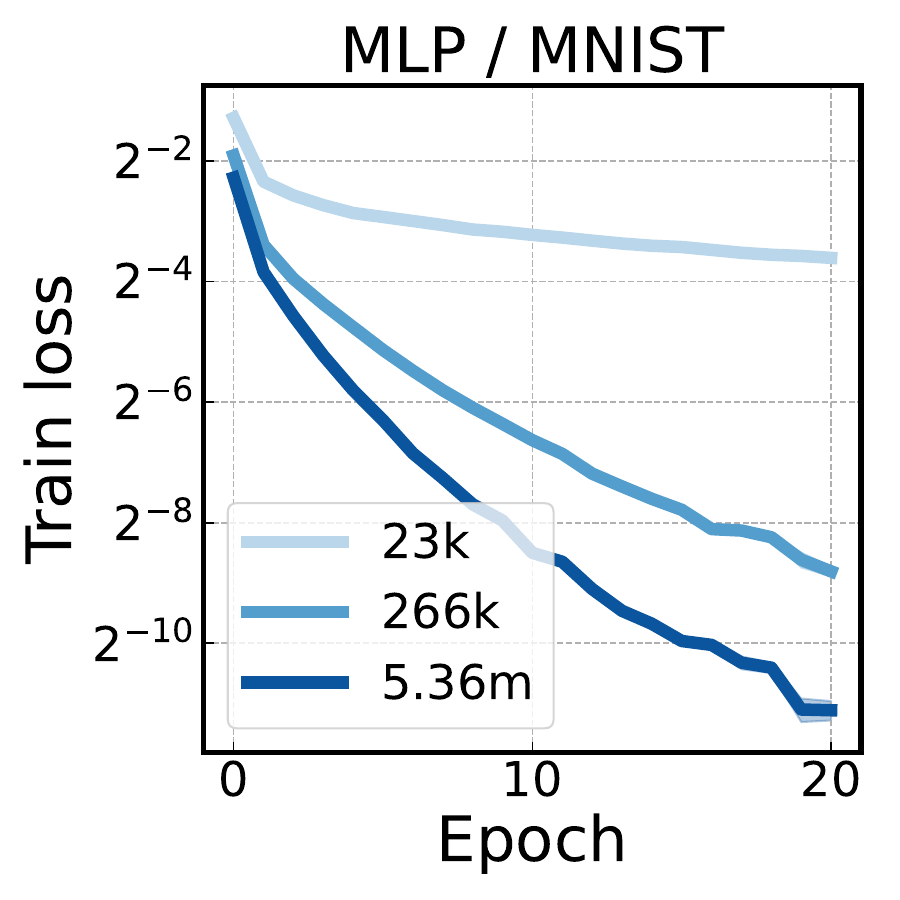}
      \includegraphics[width=0.32\linewidth, trim={1em 1.2em 1em 0}, clip]{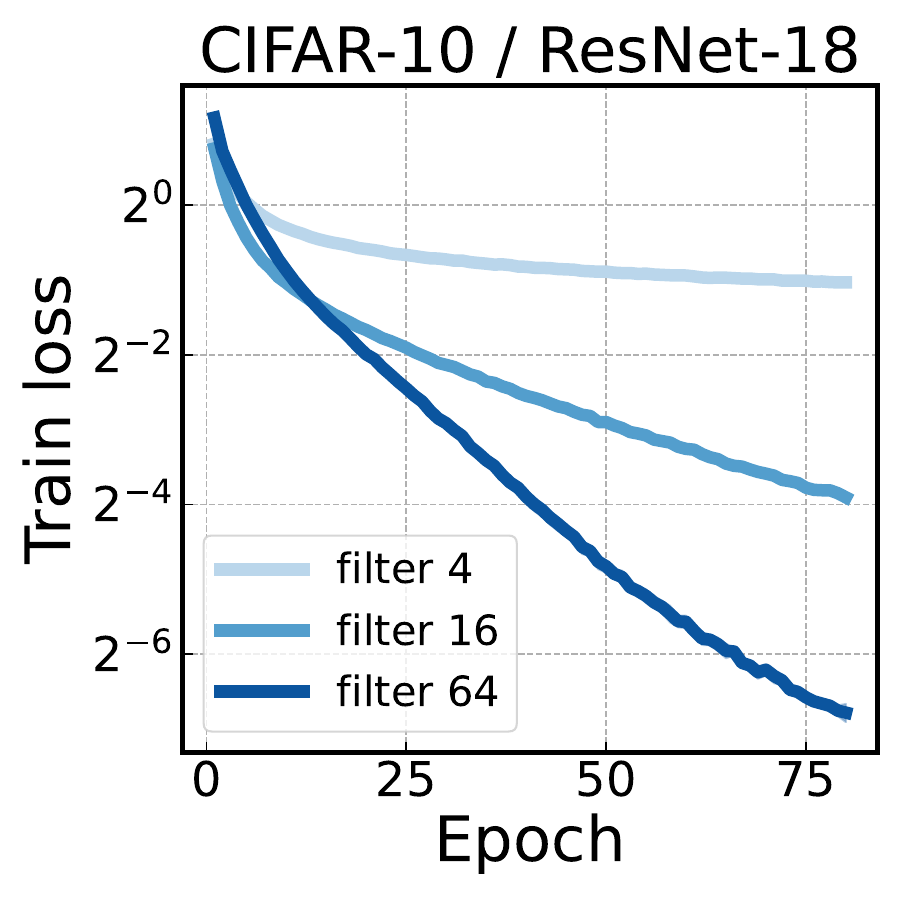}
      \caption{Convergence}
      \label{fig:convergence}
  \end{subfigure}
  \caption{
    (a) Loss landscapes of SGD (left) and SAM (right) along with the corresponding sharpness $a = \lambda_{max}(H)$. SAM converges to flatter minima with lower sharpness compared to SGD.
    (b) Non-uniformity of Hessian for SGD and SAM. SAM has a more uniform Hessian distribution than SGD.
    (c) Convergence properties of SAM. As model becomes overparameterized, SAM converges much faster and closer to a linear rate.
    See \cref{app:exp_details_theory} for the experiment details.}
  \label{fig:stability_experiments}
  \vspace{-1.0em}
\end{figure*}

\subsection{Stochastic SAM converges much faster with overparameterization}
\label{sec:convergence}

Prior works have revealed the power of overparameterization for stochastic optimization methods to accelerate convergence \citep{ma2018power,vaswani2019fast,meng2020fast}.
We prove that this benefit also extends to a stochastic SAM.

Besides the interpolation assumption we defined earlier in \cref{def:interpolation}, let us start by providing some assumptions used below.
\begin{definition}
    (Smoothness) $f$ is $\beta$-smooth if there exists  $\beta>0$ s.t. $\| \nabla f (x) -\nabla f (y) \| \leq \beta \| x-y \|$ for all $x,y  \in \mathbb R ^d$.
\end{definition}
\begin{definition}
    (Polyak-Lojasiewicz) $f$ is $\alpha$-PL if there exists $\alpha > 0$ s.t. $\| \nabla f (x) \|^2 \geq \alpha (f(x)-f(x^\star))$ for all $x  \in \mathbb R ^d$.
\end{definition}

The smoothness and the Polyak-Lojasiewicz (PL) assumptions are standard and used frequently in optimization \citep{gower2020variance, meng2020fast, nutini2022let, karimi2016linear}.
The smoothness assumption is satisfied for any neural network with smooth activation and loss function with bounded inputs \citep{andriushchenko2022towards}, and the PL condition is argued to be satisfied when the model is overparameterized \citep{belkin2021fit, liu2022loss}, which we empirically verify in \cref{fig:resut_empirical_pl} of \cref{app:emp-measure}.

Under these assumptions, we present the following convergence theorem of a stochastic SAM:

\begin{theorem}\label{thm:main-PL-stochSAM}
Suppose each $f_i$ is $\beta$-smooth, $f$ is $\lambda$-smooth and $\alpha$-PL, and interpolation holds.
For any $\rho\leq \frac{1}{(\beta/\alpha + 1/2)\beta}$, a stochastic SAM that runs for $t$ iterations with constant step size $\eta^\star \defeq \frac{\alpha-(\beta + \alpha/2)\beta\rho}{2\lambda\beta(\beta\rho+1)^2}$ gives the following convergence guarantee:
\begin{equation*}\label{eq:theorem-PL-stochastic-SAM}
\ex{x_t}{f(x_t)}\leq
\left(1 - \frac{\alpha-(\beta + \alpha/2)\beta\rho}{2} \, \eta^\star\right)^t\,f(x_0).
\end{equation*}
\end{theorem}

We provide the full proof in Appendix \ref{app:prooflinconv}, which also contains result for the more general case of a mini-batch SAM.

This result shows that with overparameterization, a stochastic SAM can converge as fast as the deterministic gradient method at a linear convergence rate, which is much faster than the well-known sublinear rate of $\mathcal{O}(1/t)$ for SAM \citep{andriushchenko2022towards}.
Also, our analysis suggests that convergence is guaranteed without the bounded variance assumption and diminishing step size under overparameterization, while without overparameterization, convergence does not hold \citep{andriushchenko2022towards}.
This suggests that overparameterization can significantly ease the convergence of SAM.
We corroborate our result empirically as well, by measuring how training proceeds with overparameterization in realistic settings.
The results are plotted in \cref{fig:convergence}.

%% file: text/7_discussion.tex
\section{Conclusion} \label{sec:discussion}

In this work, we have disclosed the \textit{critical influence of overparameterization on SAM} from empirical and theoretical perspectives.
We started with an extensive evaluation to display a highly consistent trend that the generalization benefit of SAM increases with overparameterization, without which SAM may not take effect (\cref{sec:experiments-main}).
This led us to come up with a reasonable hypothesis to explain the benefit in terms of increased solution space and implicit bias (\cref{sec:understanding}).
In addition, we presented further merits and caveats of overparameterization in practice (\cref{sec:practical}).
Finally, we developed theoretical advantages of overparameterization for SAM on linear stability, convergence, and generalization (\cref{sec:theory}).
We believe these findings can bridge between overparameterization and SAM, which has been rather unattended in the literature as of yet.
Nevertheless, we discuss limitations, ideas for potential future work as well as practical implications of our results below.

\paragraph{Theoretical account of \cref{sec:experiments-main}}
The consistent trend observed in \cref{sec:experiments-main} certainly hints at the presence of a fundamental process underneath, and yet, our study does not offer a precise theory to support this phenomenon.
This is largely because modeling the generalization of SAM under varying degrees of overparameterization challenges the boundaries of existing theoretical frameworks currently available in the literature.
Nevertheless, drawing upon recent advancements in understanding overparameterization and generalization, we have developed plausible hypotheses to directly address this phenomenon (\cref{sec:understanding}). 
We also employed rigorous theoretical frameworks to examine the effects of overparameterization on various other aspects of SAM, reinforcing the general trend of overparameterization benefits (\cref{sec:theory}).
We believe these efforts offer valuable insights and preliminary foundations that could be instrumental in achieving a comprehensive theoretical account of \cref{sec:experiments-main} in the future.

\paragraph{Other sharpness minimization schemes}
Our theoretical results in \cref{sec:theory} are based on an unnormalized version of SAM.
This is largely driven by two reasons:
(i) it appears to render minimal practical difference from the original SAM, and more crucially,
(ii) it simplifies analyses as widely adopted in initial studies \citep{andriushchenko2022towards, compagnoni2023sde}.
However, more recently, works such as \citet{dai2023crucial, si2023practical} have highlighted the theoretical significance of the normalization step.
We plan to extend our analysis to better reflect the effect of normalization in future work.
Additionally, given that different sharpness minimization schemes can make a difference in the found minima and resulting performance \citep{kaddour2022flat, dauphin2024neglected}, extension of our analyses to other non-SAM sharpness minimization schemes \citep{izmailov2018averaging,orvieto2022anticorrelated} and studying how they compare to SAM under overparameterization would be a promising avenue for future work.
Nonetheless, we consider these results an initial exploration of the impact of overparameterization on SAM, setting the stage for future research.

\paragraph{More ablation study}

In addition to label noise, sparsity, and regularization from \cref{sec:practical}, we investigate the influence of other factors on the increased benefit of SAM in \cref{app:ablation}.
Specifically, in \cref{app:depth}, we explore the effect of increasing the depth instead of the width, where we find that the advantages differ across architectures with MLPs appearing to benefit more significantly than ResNets.
We suspect that this may result from the complex interplay of various intricate factors and decisions involved in increasing depth in modern architecture.
Also, in \cref{app:exp-linear}, motivated by recent studies suggesting that overparameterized models can behave like linearized models \citep{jacot2018neural, chizat2019lazy}, we test if the increased benefit of SAM is due to linearization.
Our observations show that SAM underperforms SGD in the linearized regimes by more than $-10\%$, indicating that overparameterization itself, rather than linearization,
is likely the key factor behind the increased effectiveness of SAM.

\paragraph{Potential to modern deep learning}

Our key observations in \cref{sec:experiments-main} indicate a great potential to use SAM in the modern landscape of large-scale training \citep{kaplan2020scaling,belkin2021fit}.
Also, our results in \cref{sec:practical} further highlight its potential in the current trend where foundation models are often trained with noisy data \citep{radford2021learning,schuhmann2022laion} or to employ sparsity \citep{frantar2023scaling,jiang2024mixtral}.
In this regard, we can possibly anticipate that the overparameterization benefit might hold even when training billion-scale foundation models \citep{zhang2022opt,dehghani2023scaling}, which we leave to explore as future work.
It would also be interesting to study how popular settings for training foundation models other than label noise or sparsity affect the benefit, such as quantization \citep{gholami2022survey}, dataset pruning \citep{agiollo2024approximating}, differential privacy \citep{yu2021differentially}, or human alignment \citep{ouyang2022training}.

%% file: text/acknowledgement.tex
\section*{Acknowledgement}
This work was partly supported by the Institute of Information \& communications Technology Planning \& Evaluation (IITP) grant funded by the Korean government (MSIT) (IITP-2019-0-01906, Artificial Intelligence Graduate School Program (POSTECH) and RS-2022-II220959, (part2) FewShot learning of Causal Inference in Vision and Language for Decision Making), the National Research Foundation of Korea (NRF) grant funded by the Korean government (MSIT) (2022R1F1A1064569, RS-2023-00210466, RS-2023-00265444).
Sungbin Shin was supported by Kwanjeong Educational Foundation Scholarship.
M.A. was supported by the Google Fellowship and Open Phil AI Fellowship.

%% file: text/appendix.tex
\section{Experimental details}

We present the experimental details of \cref{sec:experiments-main,sec:understanding,sec:practical,sec:theory}.
Most of the experiments are conducted with a single RTX3090 GPU with $24$GB VRAM while some experiments requiring larger memory are conducted with multiple RTX3090 GPUs.
The code to reproduce the results of this work is implemented with JAX \citep{jax2018github} and Flax \citep{flax2020github}, which is available at \url{https://github.com/LOG-postech/SAM-overparam}.

\subsection{Experiments for Section \ref{sec:experiments-main}}

\label{app:exp_details}

\begin{table*}[!th]
    \renewcommand{\arraystretch}{1.7}
    \centering
    \resizebox{\linewidth}{!}{
    \begin{tabular}{lllllllll}
      \toprule
      Workload & Epochs/steps & Learning rate / decay & Weight decay & Batch size & $\rho$ search & Base optimizer &  \\
      \midrule
      Synthetic         & 100 epochs   & $0.1$ / step & $0.0$ & $128$ & $\Big\{\substack{0.001, 0.01, 0.05, 0.07, 0.1, \\0.2, 0.3, 0.5, 0.7, 1.0, 2.0}\Big\}$ & SGD \\ 
      MNIST/MLP         & $100$ epochs   & $0.1$ / step & $0.0001$ & $128$ & $\{0.01, 0.02, 0.05, 0.1, 0.2\}$ & SGD with momentum $0.9$ \\ 
      CIFAR-10/ResNet-18  & $200$ epochs & $0.1$ / step & $0.0005$ & $128$ & $\Big\{\substack{0.001, 0.005, 0.01, 0.02, \\0.05, 0.1, 0.2, 0.5, 1.0}\Big\}$ & SGD with momentum $0.9$\\ 
      ImageNet/ResNet-50  & $90$ epochs & $0.1$ / cosine & $0.0001$ & $512$ & $\{0.01, 0.02, 0.05, 0.1, 0.2\}$ & SGD with momentum $0.9$\\ 
      PoS tagging  & $75000$ steps & $0.05$ / inverse sqrt & $0.1$ & $64$ & $\{0.01, 0.02, 0.05, 0.1, 0.2, 0.3, 0.5\}$ & AdamW ($\beta_1 = 0.9, \beta_2 = 0.98$) \\ 
      Sentiment classification  & $30$ epochs & $0.1$ / constant & $3e\text{-}6$ & $64$ & $\{0.01, 0.02, 0.05, 0.1, 0.2, 0.3, 0.5\}$ & SGD with momentum $0.8$ \\ 
      Graph property prediction  & $10^5$ steps & $0.001$ / constant & $0.0$ & $256$ & $\{0.01, 0.02, 0.05, 0.1, 0.2\}$ & Adam ($\beta_1 = 0.9, \beta_2 = 0.999$) \\ 
      Atari game  & $10^7$ steps & $2.5e\text{-}4$ / linear & $0.0$ & $256$ & $\{0.01, 0.02, 0.05, 0.1, 0.2\}$ & Adam ($\beta_1 = 0.9, \beta_2 = 0.999$) \\ 
      CIFAR-10/ViT  & $200$ epochs & $0.1$ / cosine & $0.0001$ & $128$ & $\{0.01, 0.02, 0.05, 0.1, 0.2\}$ & SGD with momentum $0.9$\\ 
      \bottomrule
  \end{tabular}
  }
  \caption{
    Hyperparameters for each workload.
  }
  \label{tab:hyperparams}
\end{table*}

\begin{table*}[!th]
    \renewcommand{\arraystretch}{1.4}
    \centering
    \resizebox{0.8\linewidth}{!}{
    \begin{tabular}{lll}
      \toprule
      Workload & Scaling factor & Values \\
      \midrule
      Synthetic         & \# of neurons   & $\{k*100 \vert 1 \leq k \leq 10\}$ \\ 
      MNIST/MLP         & \# of neurons   & $\{[300*p, 100*p] \vert p \in \{0.25, 0.5, 1, 4, 10\}\}$ \\ 
      CIFAR-10/ResNet-18  & \# of convolutional filters & $\{2^k \vert 2 \leq k \leq 8\}$ \\ 
      ImageNet/ResNet-50  & \# of convolutional filters & $\{16*k \vert 1 \leq k \leq 5\}$ \\ 
      PoS tagging  & dimension of hidden states & $\{128*k \vert 1 \leq k \leq 5\}$  \\ 
      Sentiment classification  & dimension of hidden states & $\{2^k \vert 5 \leq k \leq 9\}$  \\ 
      Graph property prediction  & \# of neurons & $\{2^k \vert 7 \leq k \leq 9\}$ \\ 
      Atari game  & \# of convolutional filters & $\{16*k \vert 1 \leq k \leq 4\}$ \\ 
      CIFAR-10/ViT  & dimension of hidden states  & $\{2^k \vert 5 \leq k \leq 10\}$ \\ 
      \bottomrule
  \end{tabular}
  }
  \caption{
    Model scaling factors and values for each workload.
  }
  \label{tab:scaling}
\end{table*}

For all the experiments in \cref{sec:experiments-main}, we run the experiments with the same configurations over three different random seeds. 
We visualize the average and standard error (\ie, std$/\sqrt{n_{\text{seed}}}$) as a line plot and a shaded region surrounding it.
Many of our experiments and the hyperparameter values are based on examples provided by Flax \citep{flax2020github} official repository.\footnote{\url{https://github.com/google/flax/tree/main/examples}}
The hyperparameter values and how the models are scaled for each workload are summarized in \cref{tab:hyperparams,tab:scaling}, respectively.
We present the additional details for individual workloads below.

\paragraph{Synthetic Regression / 2-layer MLP}

We follow the student-teacher setting from \citet{advani2020highdi} where the teacher is a randomly initialized $2$-layer ReLU network with $200$ neurons and the student is a $2$-layer ReLU network with a different number of neurons.
Each element for the input $x \in \R^{100}$ is sampled from a standard normal distribution while the target $y \in \R$ is calculated as the output of the teacher network added by Gaussian noise sampled from a standard normal distribution.
The models are trained on $20400$ training data, which is roughly the same as the number of parameters in the teacher model, and tested on the $5100$ data, which is a quarter of the number of the training data.

\paragraph{MNIST / 3-layer MLP}

We train LeNet-300-100 \citep{lecun1998gradient} for the MNIST \citep{lecun2010mnist}.
The learning rate decays by $0.1$ after $50\%$ and $75\%$ of the total epochs.
We scale the models while preserving the relative proportions of the number of neurons in each layer as $3:1$.

\paragraph{CIFAR-10 / ResNet-18}

We train ResNet-18 \citep{he2016deep} for the CIFAR-10 \citep{krizhevsky2009learning}.
We choose the hyperparameters as similar to \citet{andriushchenko2022towards}.
The learning rate decays by $0.1$ after $50\%$ and $75\%$ of the total epochs.

\paragraph{ImageNet / ResNet-50}

We train ResNet-50 \citep{he2016deep} for the ImageNet \citep{deng2009imagenet}.
We choose the hyperparameters as similar to \citet{du2021efficient} and use a linear warmup of $5000$ steps.
We additionally experiment with $\rho = 0.005$ for the two smallest models.

\paragraph{PoS tagging/ Transformer}

We train Encoder-only Transformer \citep{vaswani2017attention} for the Universal Dependencies \citep{nivre2016universal} -- Ancient Greek.
We use a linear warmup of $8000$ steps.
We evaluate the validation accuracy once every $1000$ step and report the best value except for the experiment in \cref{fig:result_overparam_diff_pos_wo_es}.
The dimension of MLP and the number of attention heads are scaled as $4\times$ and $1/64 \times$ of the dimension of the hidden states following the Flax example.

\paragraph{SST / LSTM}

We train LSTM \citep{hochreiter1997long} for SST2  \citep{socher2013recursive} where the task is a binary classification (positive/negative) of the movie reviews.
We evaluate the validation accuracy for every epoch and report the best value.
The embedding size is scaled as $300/256 \times$ of the dimension of hidden states following the Flax example.

\paragraph{Graph property prediction / GCN}

We train $2$-layer Graph Convolutional Networks \citep{kipf2016semi} for the ogbg-molpcba \citep{hu2020open}.
Here, the input is a graph of a molecule where nodes and edges each represent atoms and chemical bonds.
The task is a binary classification of whether a molecule inhibits HIV replication or not.

\paragraph{Atari game / CNN}

We train $5$-layer CNNs for the Atari Breakout-v5 game \citep{mnih2013playing}.
We train the Actor-Critic networks \citep{konda1999actor} with proximal policy optimization \citep{schulman2017proximal}.
We evaluate the validation score once every $100$ step and report the best value.
We also use gradient clipping of $0.5$ for all models.

\paragraph{CIFAR-10 / ViT}

For the experiment in \cref{fig:result_overparam_diff_cifar_vit}, we train $6$-layer Vision Transformers \citep{dosovitskiy2020image} for the CIFAR-10 \citep{krizhevsky2009learning} using the patch size of $4 \times 4$.
We scale the dimension of MLP and the number of attention heads as $2 \times$ and $1/32 \times$ of the dimension of hidden states.

\subsection{Experiments for Section \ref{sec:understanding}}

\label{app:exp_details_understanding}

For the experiments in \cref{fig:onehidden-function,fig:onehidden-trajectories}, we follow the setting in \citet{andriushchenko2022towards}.\footnote{\url{https://github.com/tml-epfl/understanding-sam/tree/main/one_layer_relu_nets}}
Specifically, we train one-hidden-layer ReLU networks where each data has input $x \in \R$ and target $y \in \R$.
Here, the networks are trained on the quadratic loss with mini-batch SGD or SAM with $\rho = 0.2$ where we randomly choose $6$ data points every iteration.
Additionally, the optimization trajectories in \cref{fig:onehidden-trajectories} are plotted following \citet{li2018visualizing}.\footnote{\url{https://github.com/tomgoldstein/loss-landscape}}
Specifically, the trajectories are plotted along the PCA directions calculated from inital point and the converged minima of SGD and SAM.

\subsection{Experiments for Section \ref{sec:theory}}

\label{app:exp_details_theory}

\paragraph{Linear stability}
For the experiments of \cref{fig:stability_uniformity,fig:landscape}, we follow the setting in \citet{wu2018sgd}.
Specifically, we set up $3$-layer MLP having $[3000, 1000]$ hidden neurons with squared loss, so that the local quadratic approximation becomes precise, and train the networks on MNIST.
We use $1000$ random samples to calculate the non-uniformity, and all models are trained to reach near zero loss.
The networks are trained with a constant learning rate of $0.1$ without weight decay or momentum.

\paragraph{Convergence -- Matrix Factorization}
For the matrix factorization experiment in \cref{fig:convergence}, we solve the following non-convex regression problem: $\min_{W_1, W_2} \E_{x\sim \mathcal{N}(0, I)}{\|W_2W_1x-Ax\|^2}$ where the objective function is smooth and satisfies the PL-condition \citep{loizou2021stochastic}. 
We choose $A\in\R^{10\times6}$ and generate 1000 training samples, which are used for training a rank $k$ linear network with two matrix factors $W_1\in\R^{k\times6}$ and $W_2\in\R^{10\times k}$. Here, interpolation is satisfied when rank $k=10$.
We train two linear networks with $k\in\{4, 10\}$ for $100$ epochs with a constant learning rate of $0.0005$ and compare the convergence speed.

\section{Absolute validation metric for Section \ref{sec:experiments-main}} \label{app:add-overparam_results}

We present the absolute validation metrics of SAM and SGD in \cref{fig:result_overparam_acc_abs_app}.
There is a consistent trend that SAM improves with overparameterization in all tested cases.

\begin{figure*}[!th]
  \centering
  \includegraphics[width=0.24\linewidth]{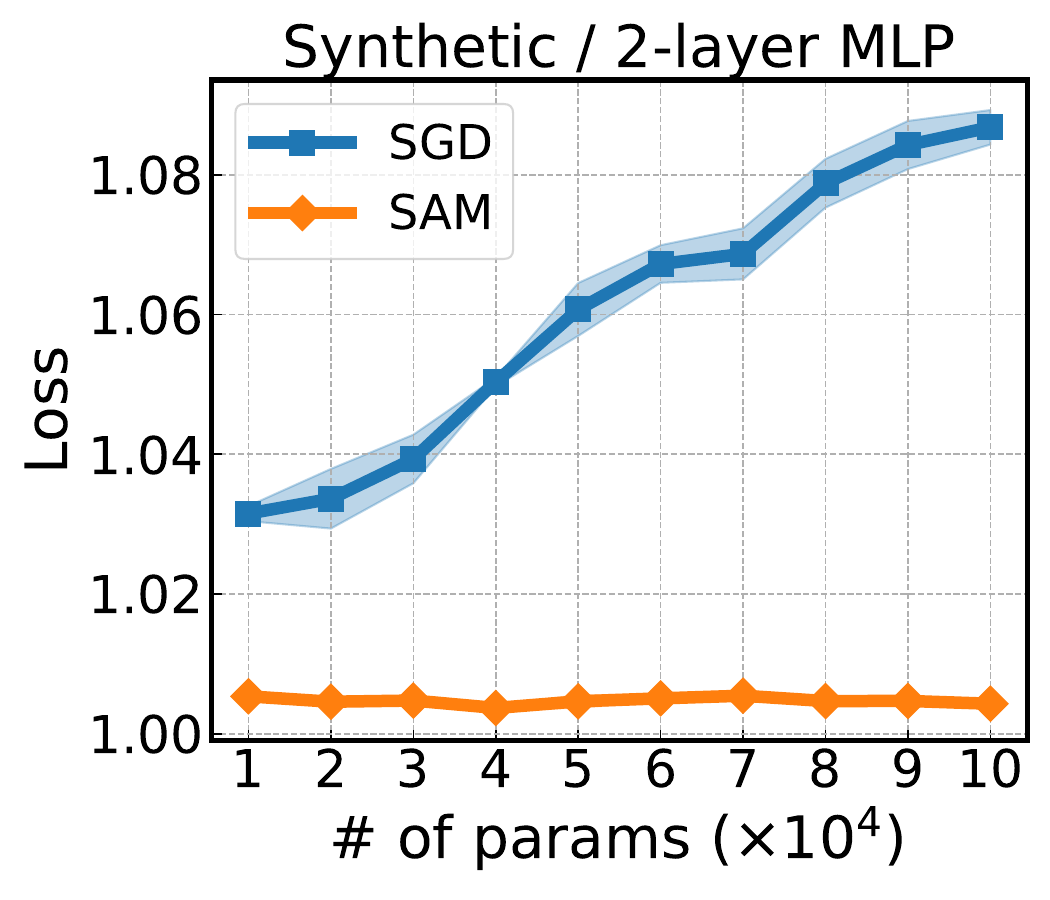}
  \includegraphics[width=0.24\linewidth]{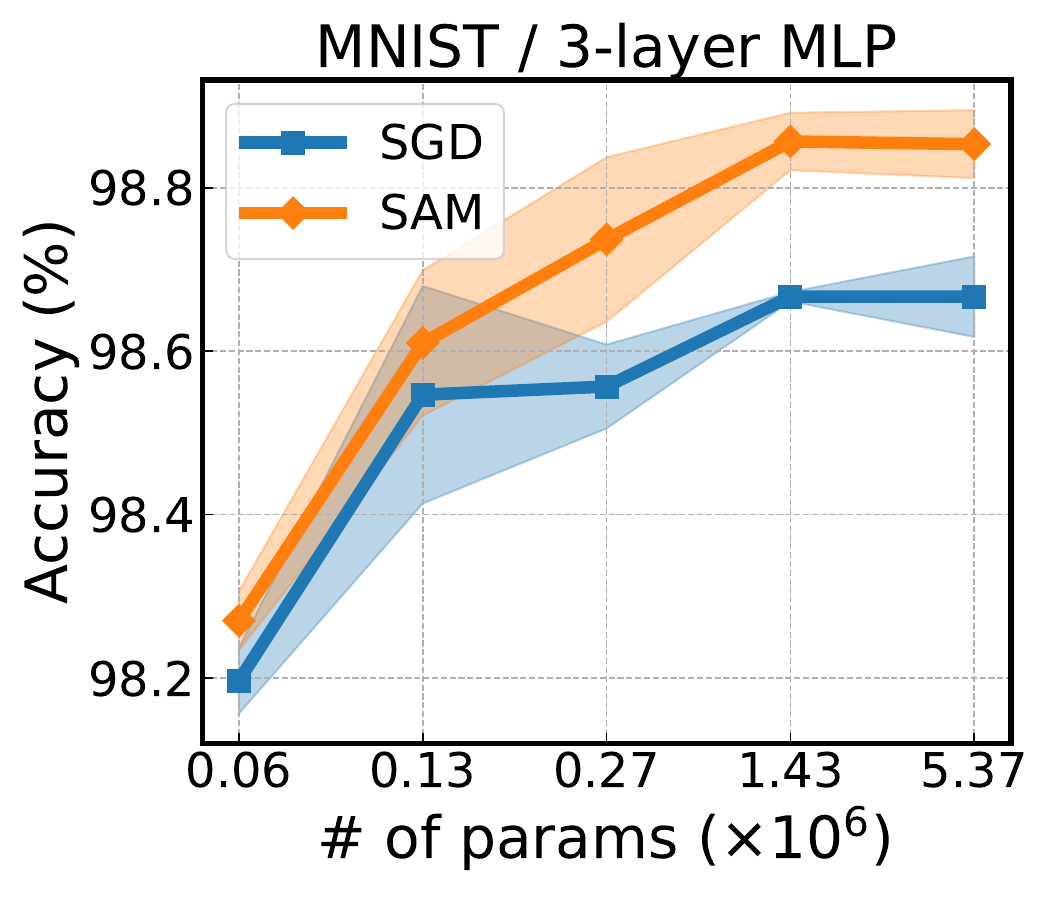}
  \includegraphics[width=0.24\linewidth]{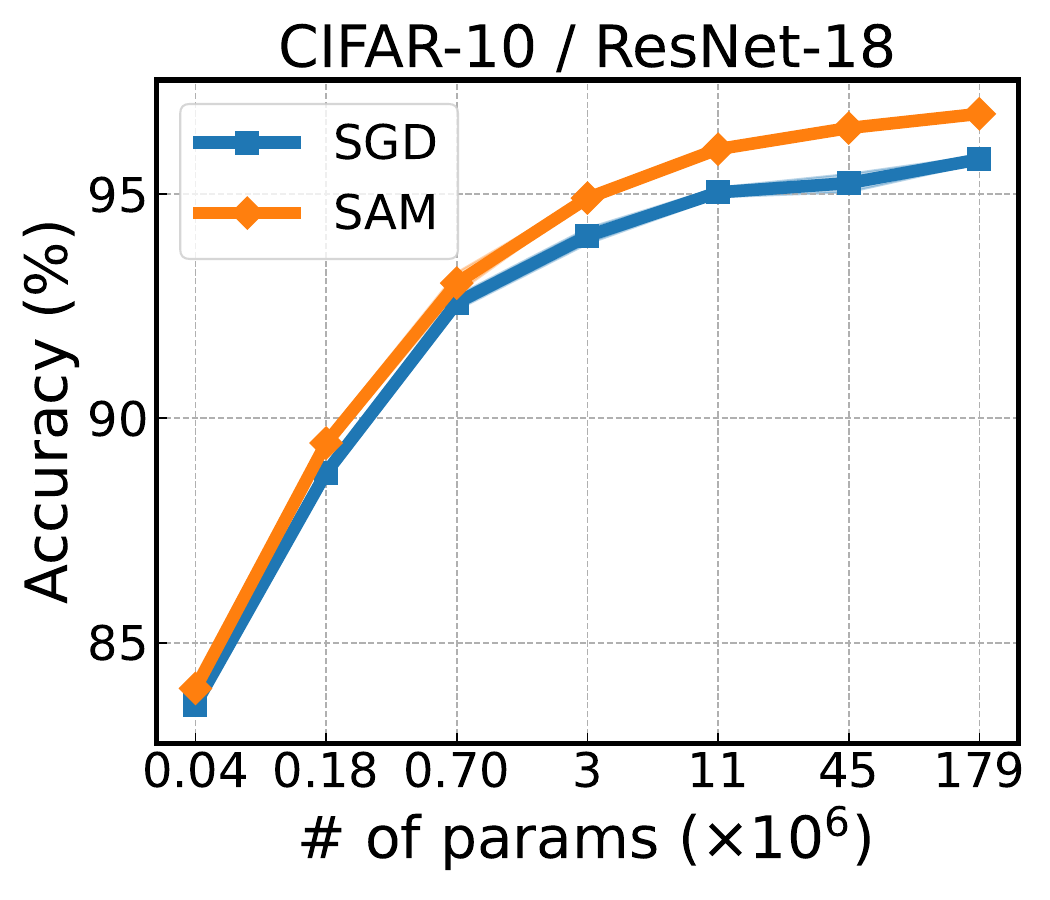}
  \includegraphics[width=0.24\linewidth,trim={0.4cm 0.4cm 0.4cm 0.4cm},clip]{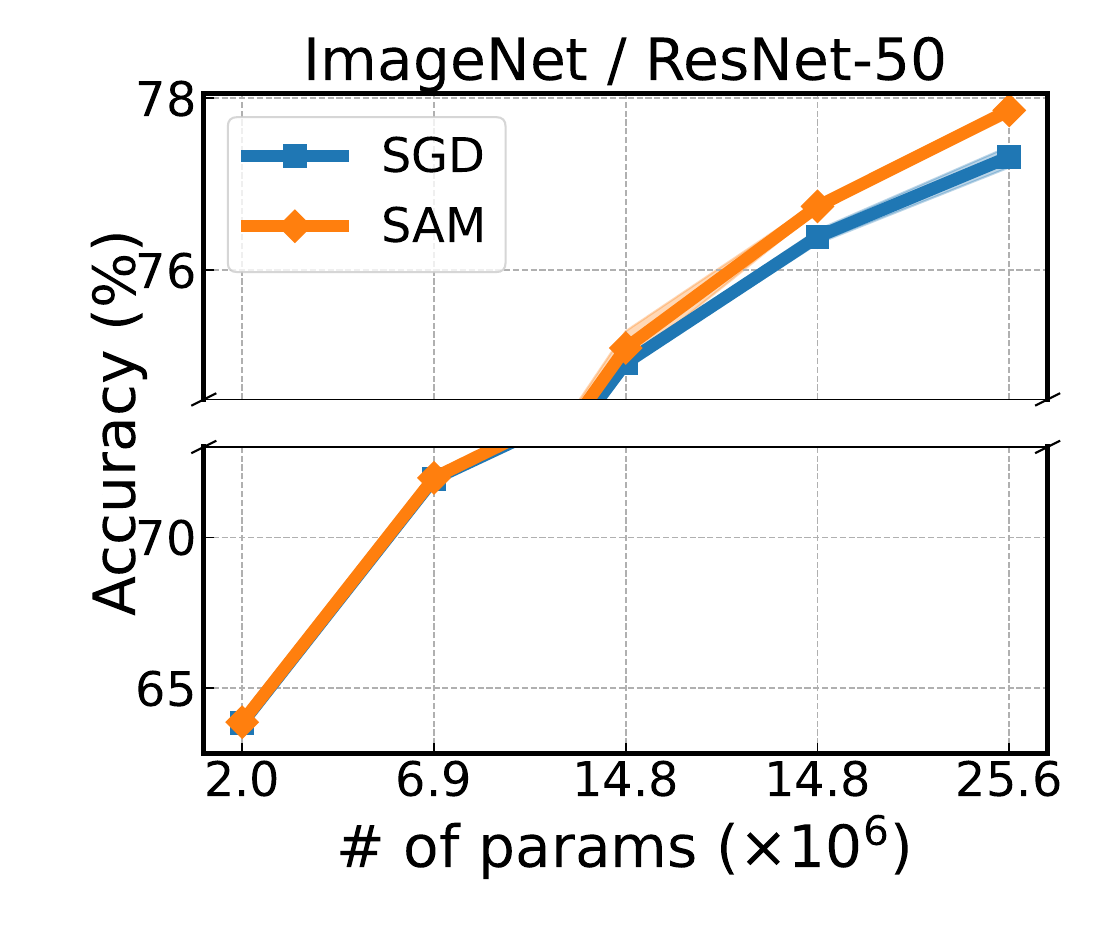}\\
  \includegraphics[width=0.24\linewidth]{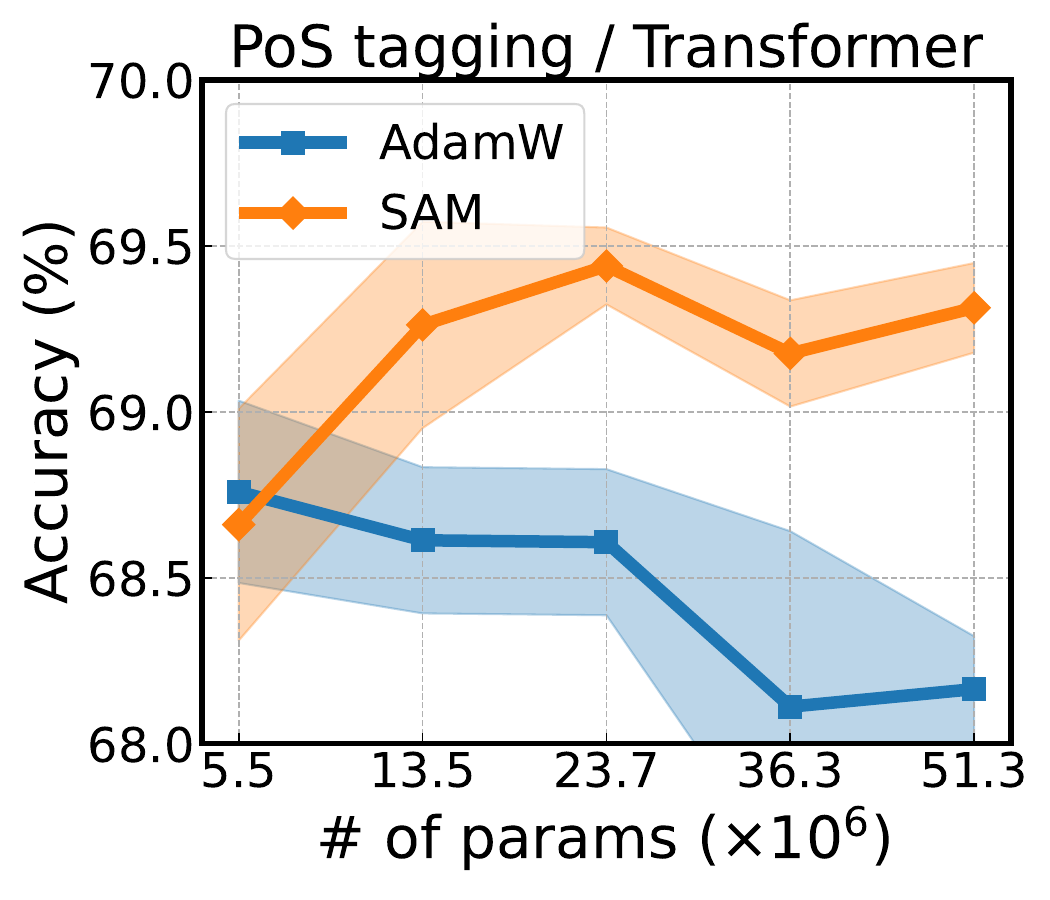}
  \includegraphics[width=0.24\linewidth]{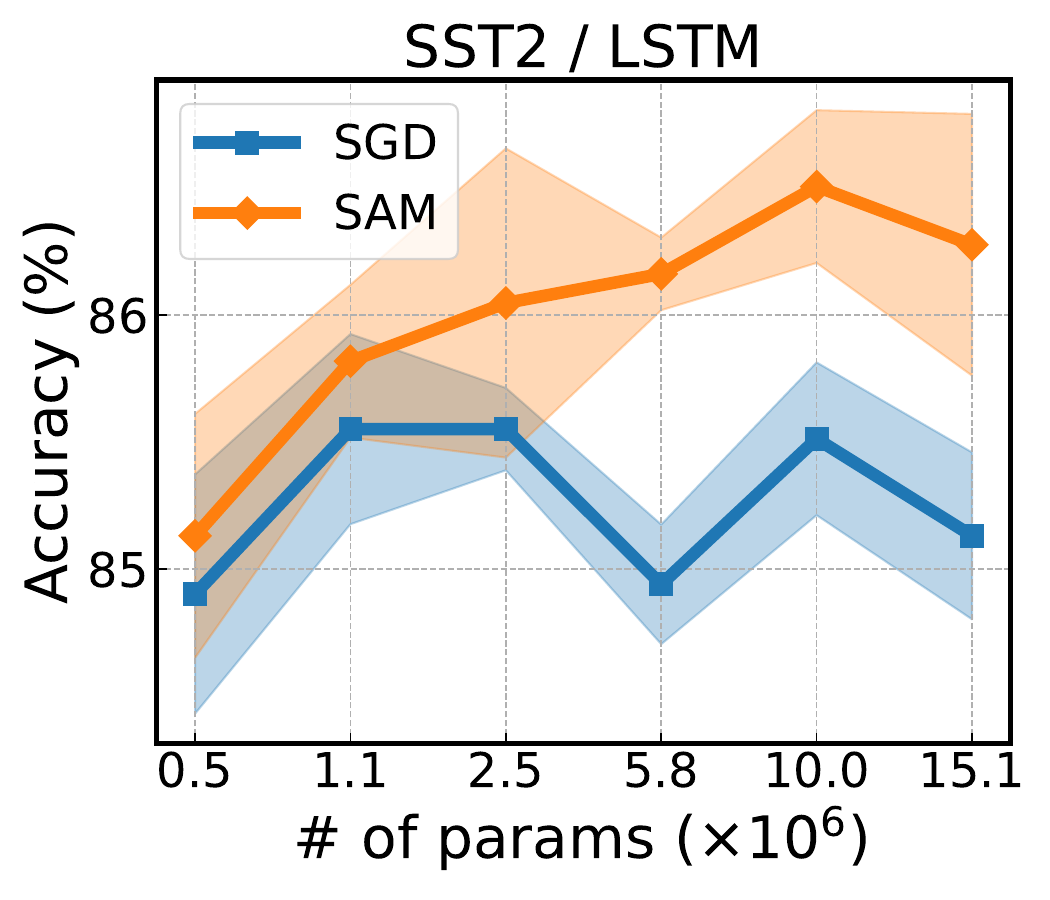}
  \includegraphics[width=0.24\linewidth]{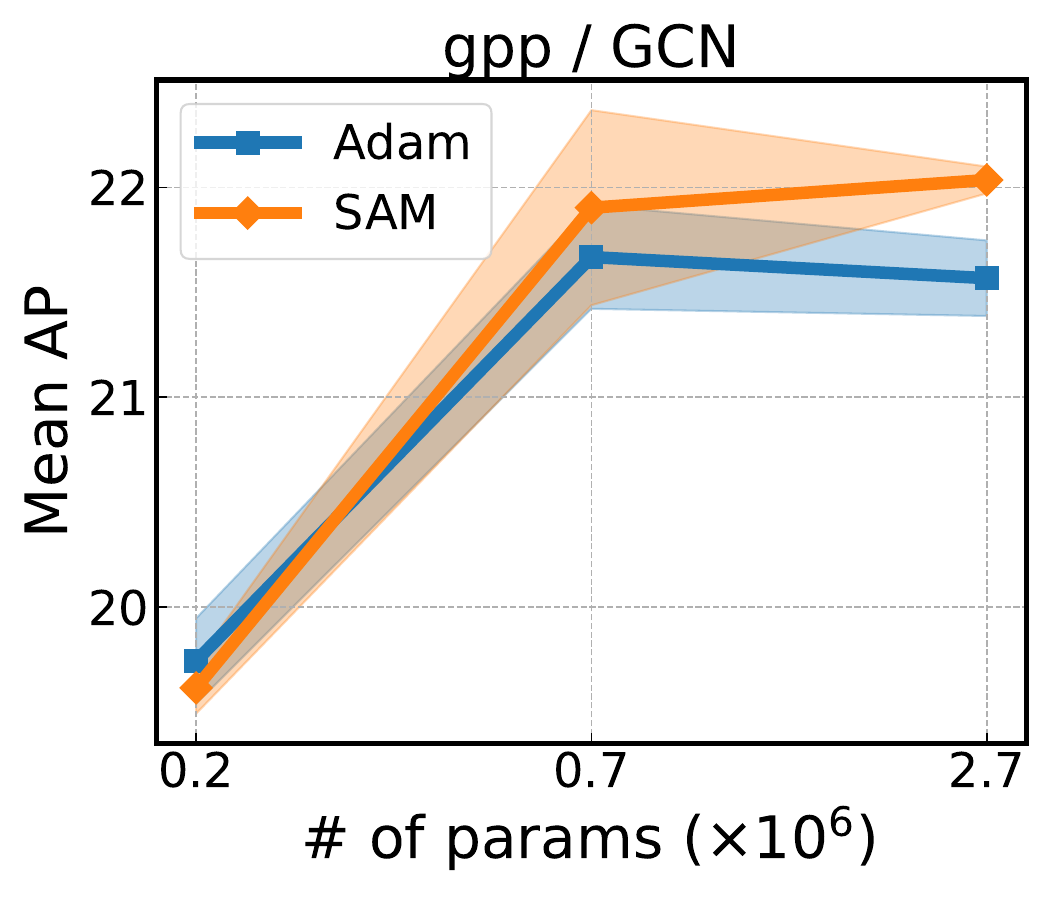}
  \includegraphics[width=0.24\linewidth]{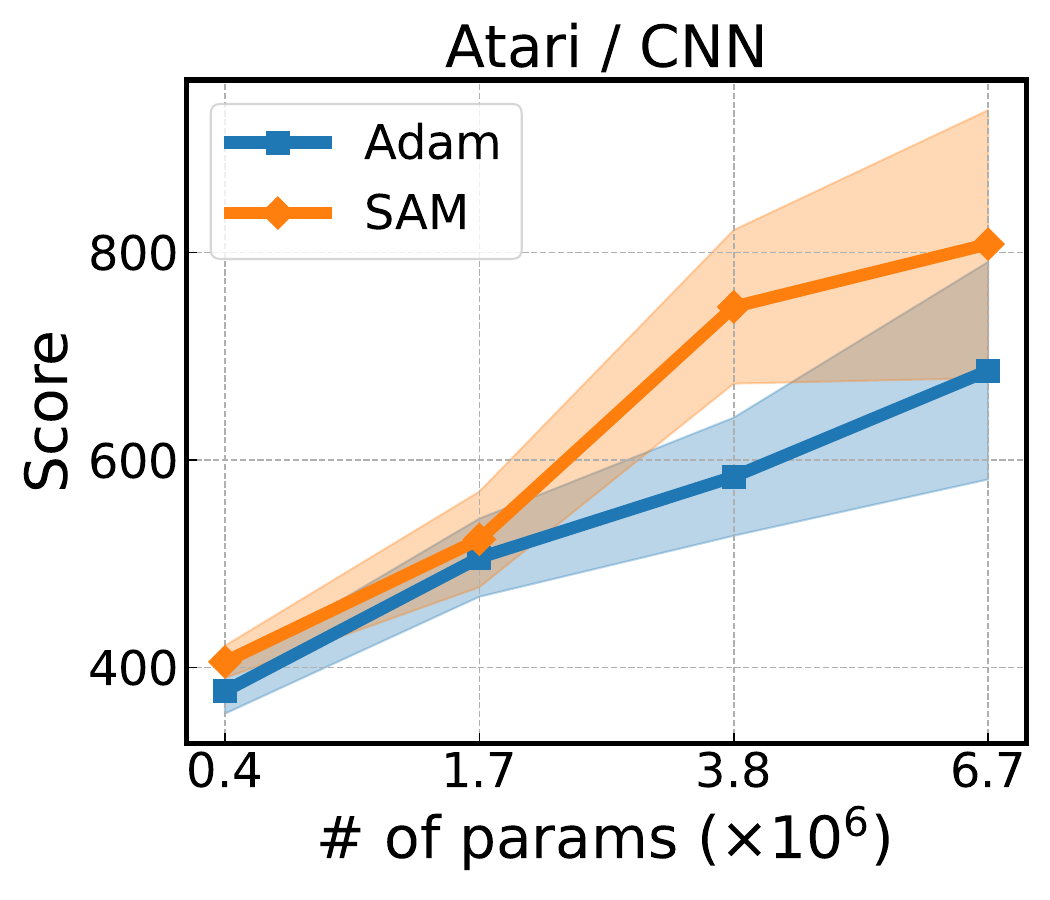}
  \caption{
    The absolute metrics for SAM and baseline optimizers.
    The generalization benefit of SAM tends to increase as the model becomes more overparameterized.
  }
  \label{fig:result_overparam_acc_abs_app}
\end{figure*}

\newpage

\section{Full results on optimal perturbation bound} \label{app:add-rho}
\vspace{-1em}

\begin{figure}[!ht]
    \centering
    \includegraphics[width=0.16\linewidth]{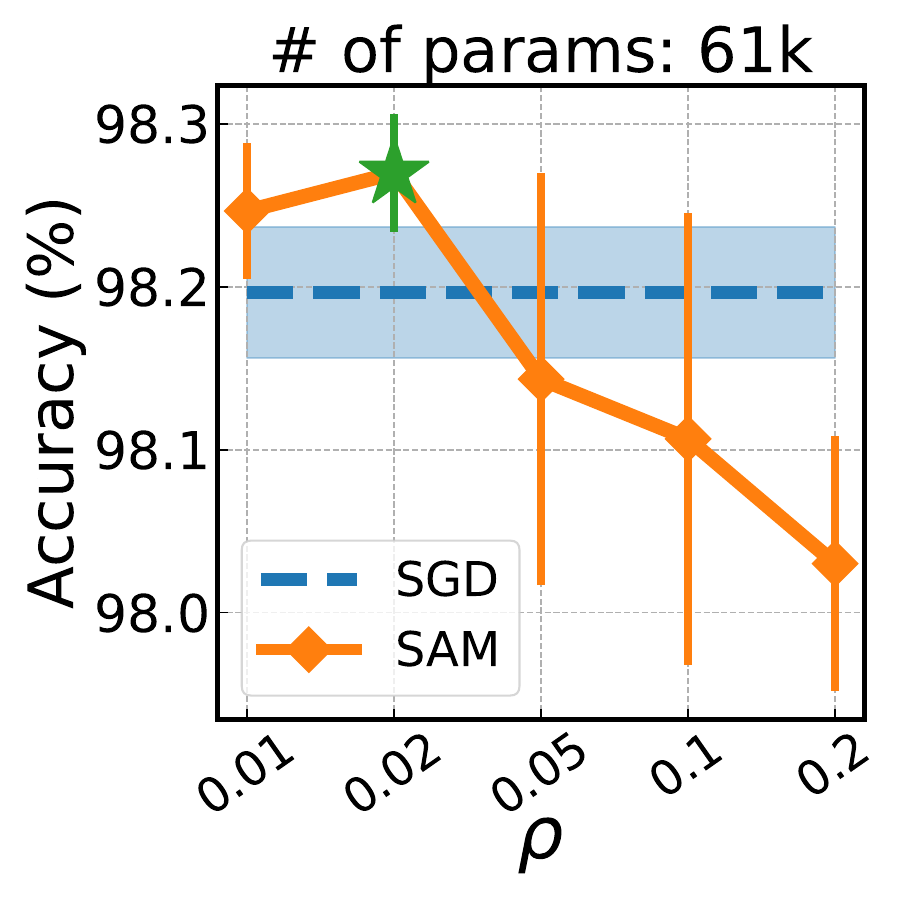}
    \includegraphics[width=0.16\linewidth]{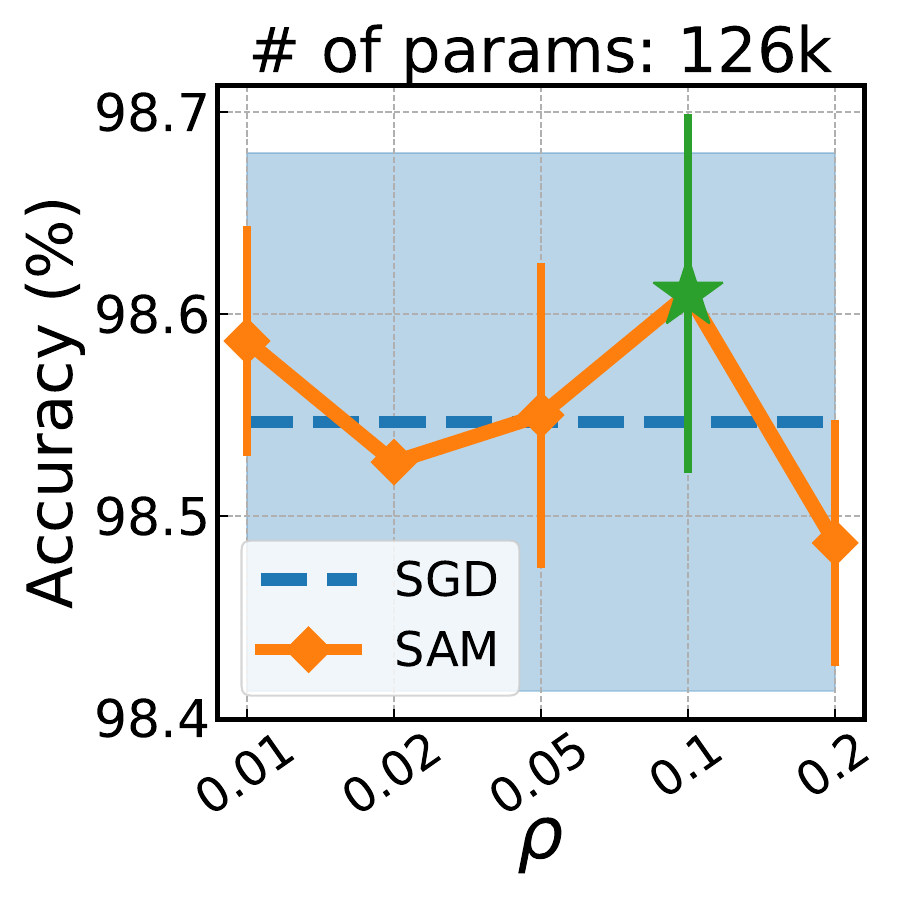}
    \includegraphics[width=0.16\linewidth]{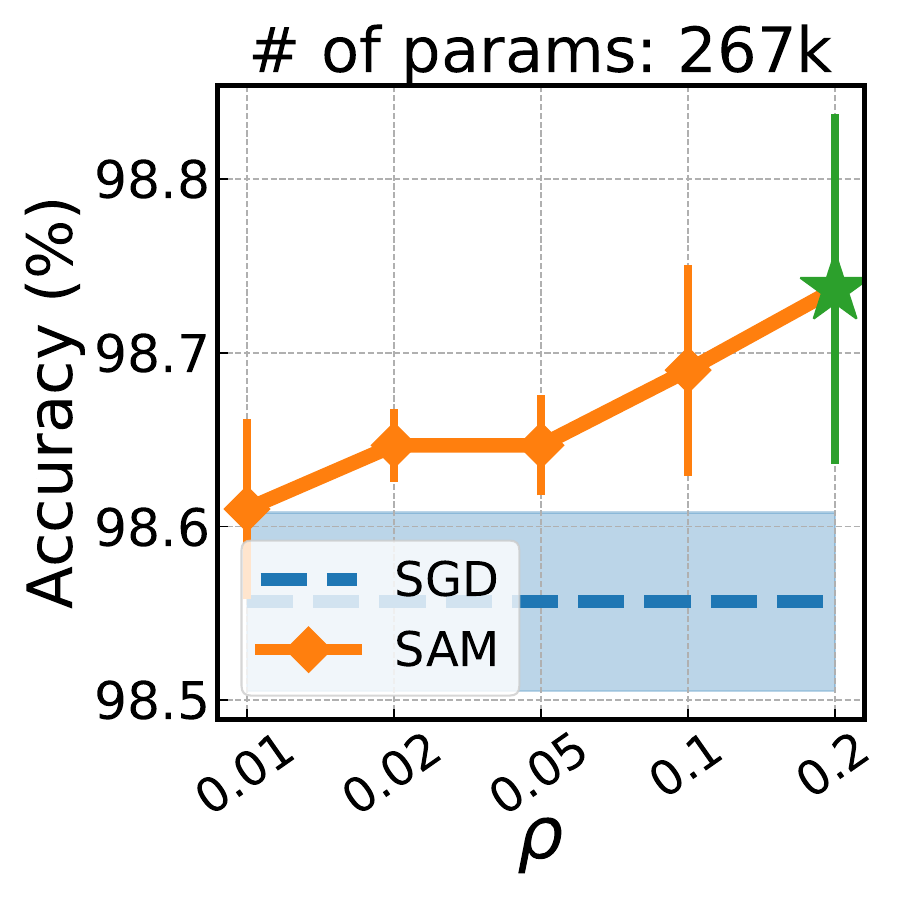}
    \includegraphics[width=0.16\linewidth]{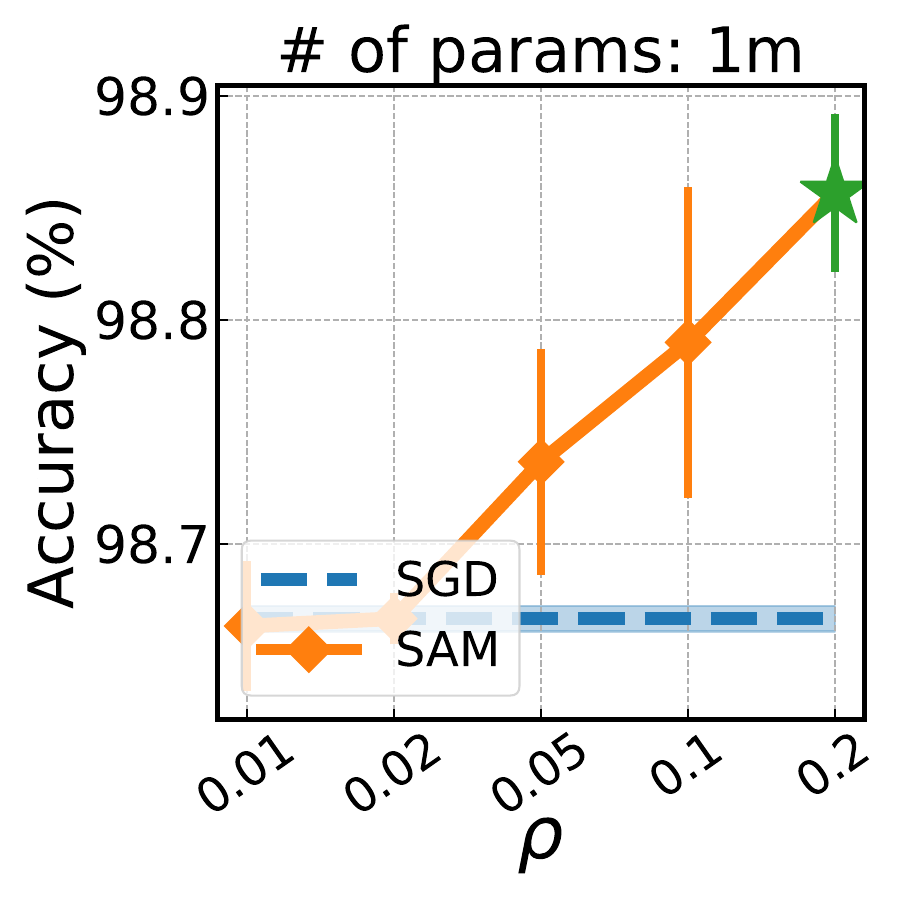}
    \includegraphics[width=0.16\linewidth]{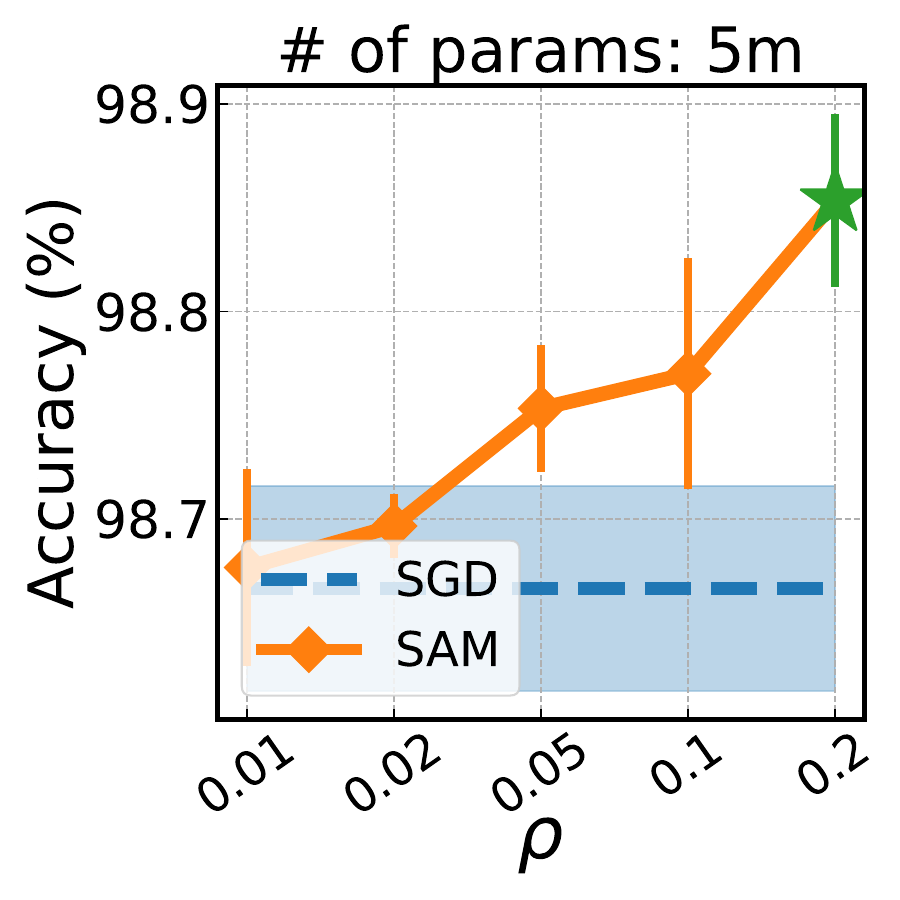}
    \includegraphics[width=0.16\linewidth]{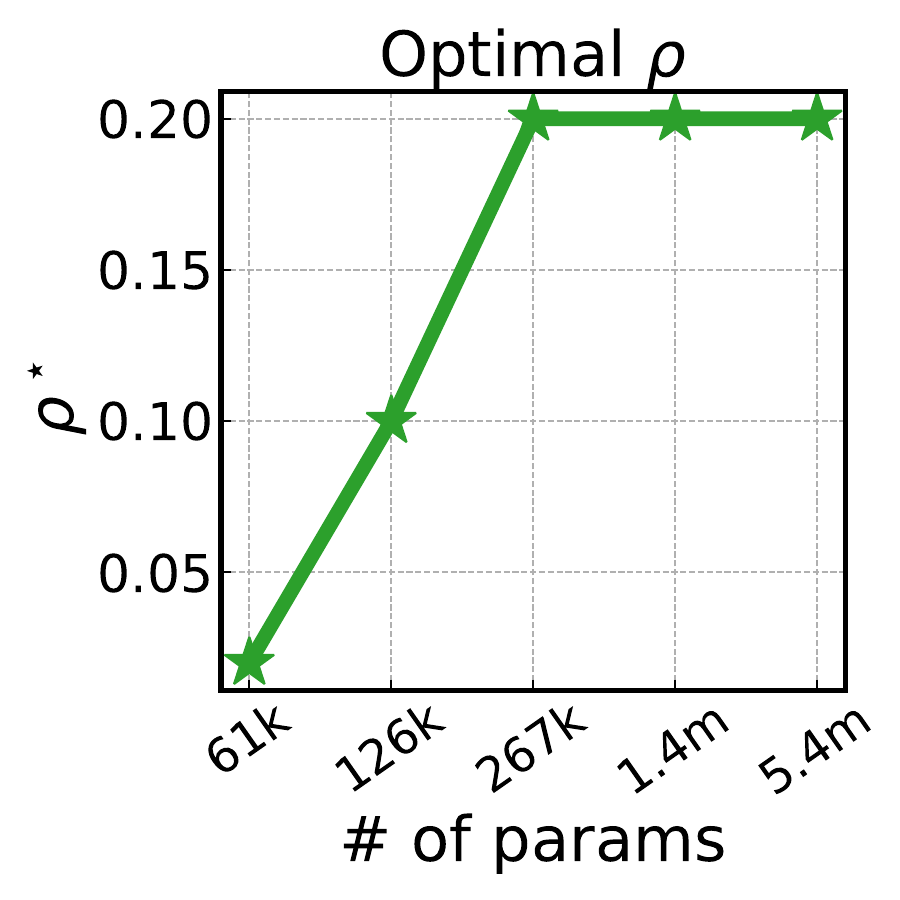}\\
    \vspace{-0.6em}
  \caption{
    Validation accuracy versus $\rho$ for MNIST and $3$-layer MLP.
  }
  \vspace{-1em}
  \label{fig:result_overparam_each_rho_mnist}
\end{figure}

\begin{figure}[!ht]
    \centering
    \includegraphics[width=0.16\linewidth]{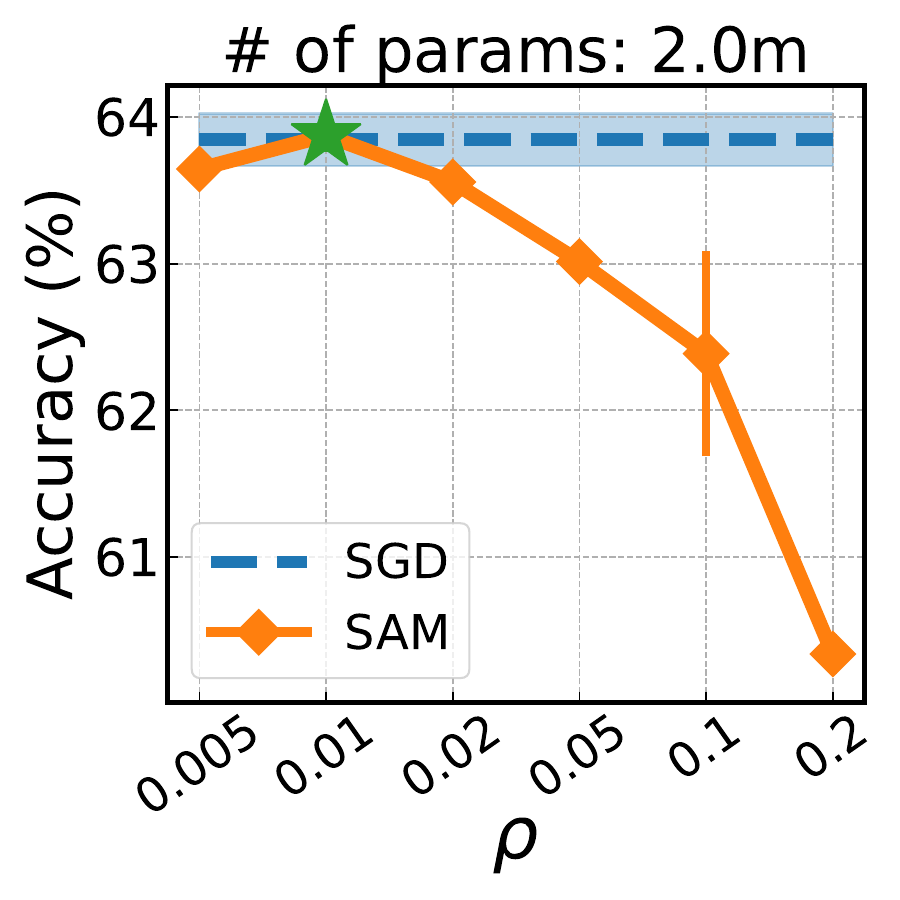}
    \includegraphics[width=0.16\linewidth]{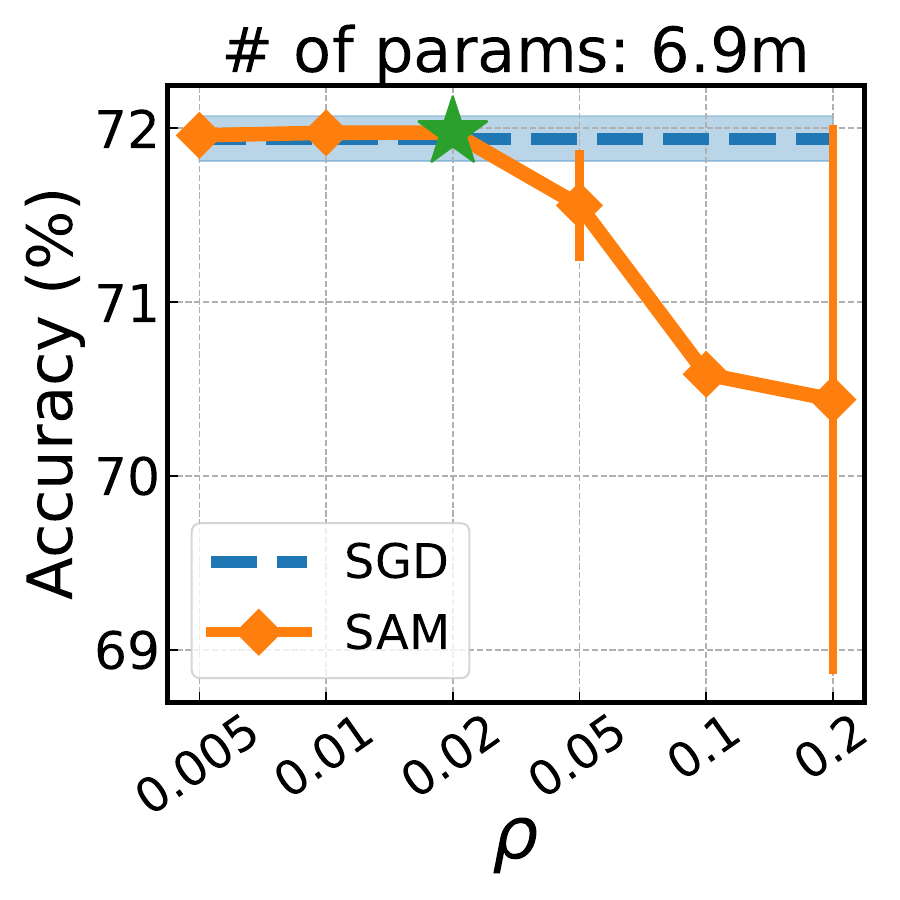}
    \includegraphics[width=0.16\linewidth]{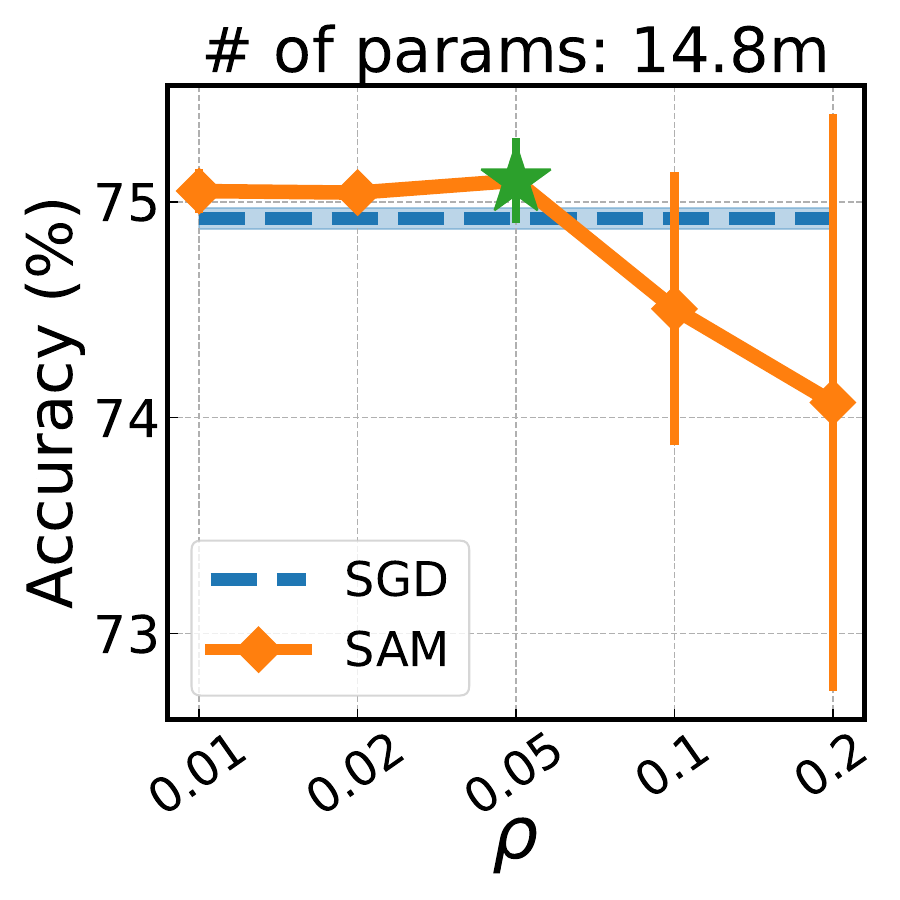}
    \includegraphics[width=0.16\linewidth]{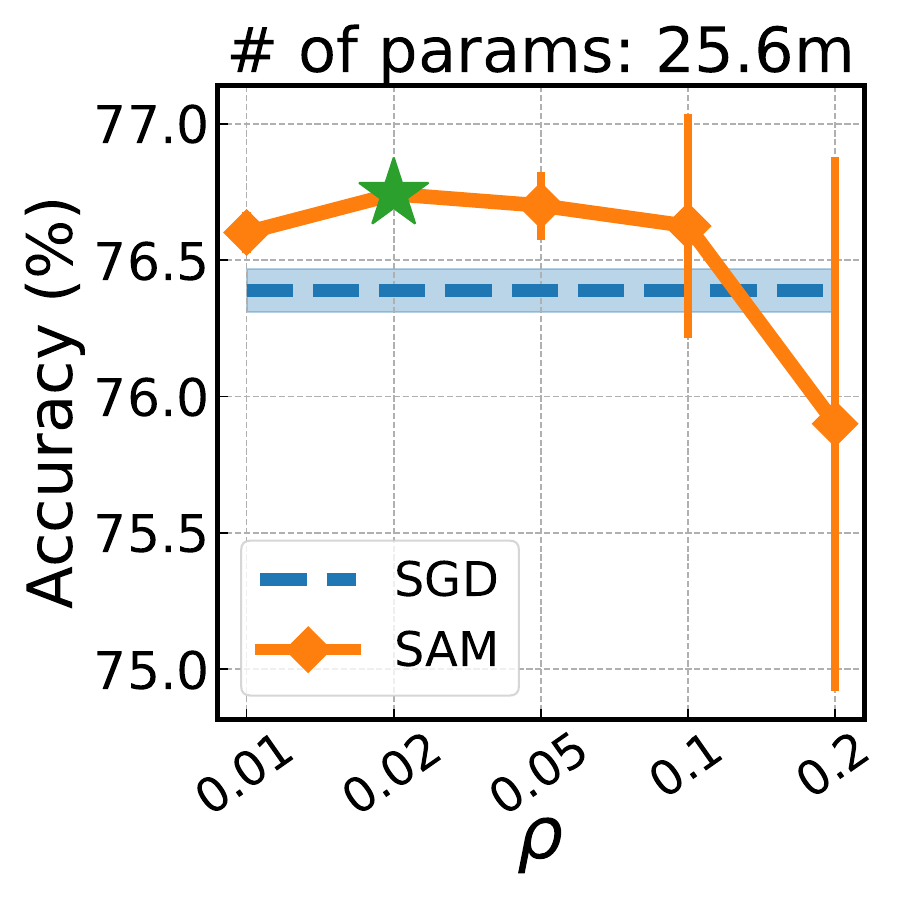}
    \includegraphics[width=0.16\linewidth]{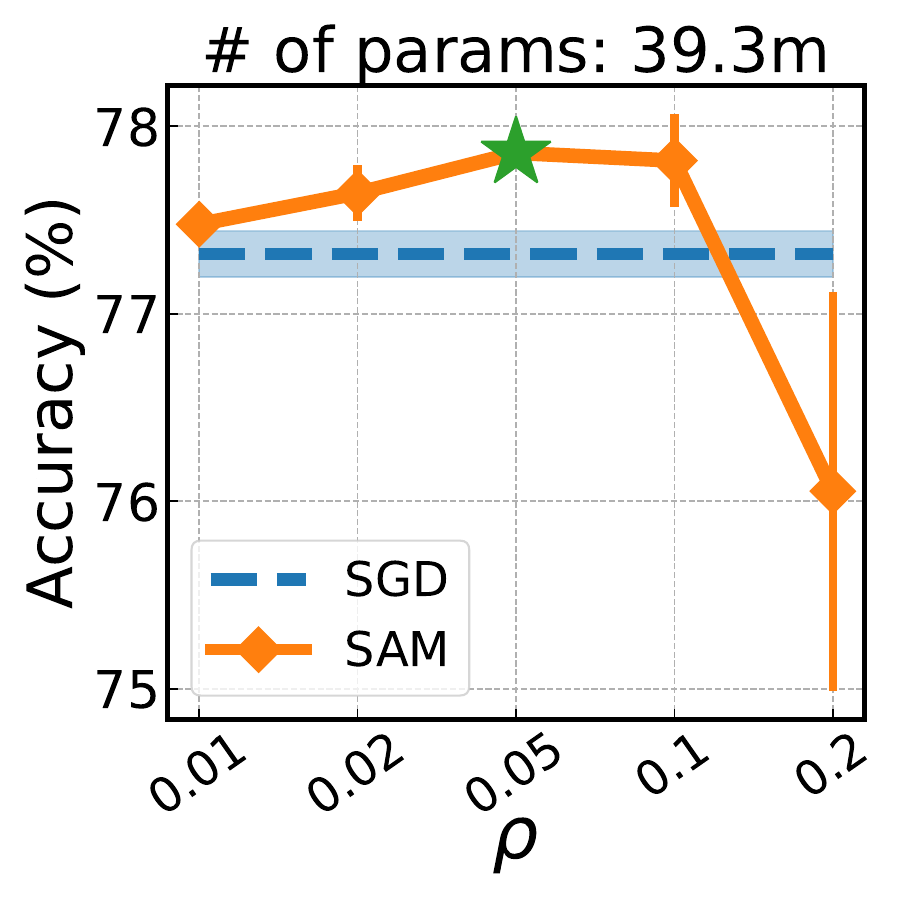}
    \includegraphics[width=0.16\linewidth]{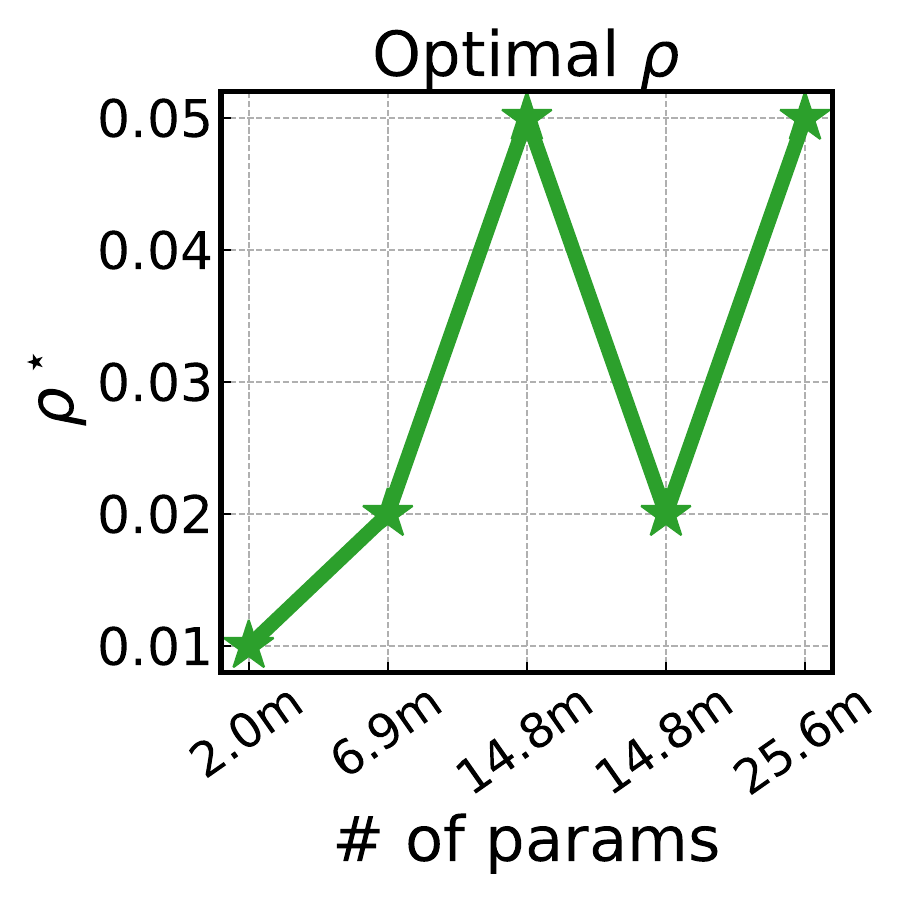}\\
    \vspace{-0.6em}
  \caption{
    Validation accuracy versus $\rho$ for ResNet-50 and ImageNet.
  }
  \vspace{-1em}
  \label{fig:result_overparam_each_rho_imagenet}
\end{figure}

\begin{figure}[!ht]
    \centering
    \includegraphics[width=0.16\linewidth]{figures/cifar/ResNet18/sam_rho/filter4_val_acc.pdf}
    \includegraphics[width=0.16\linewidth]{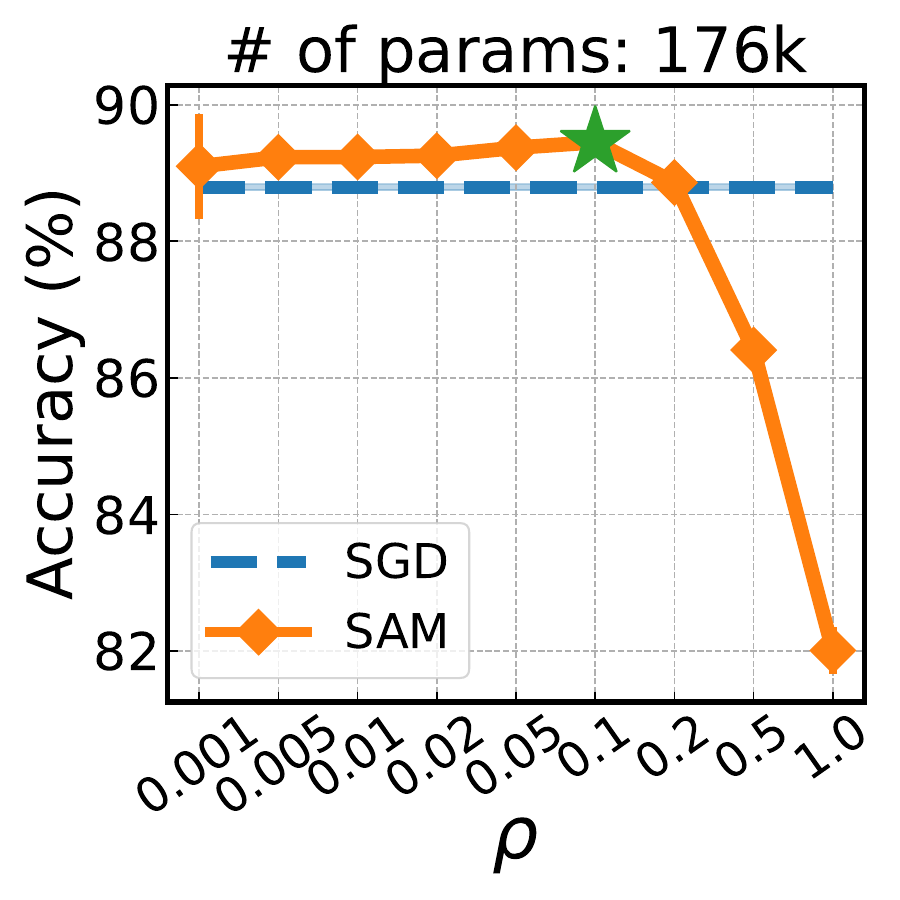}
    \includegraphics[width=0.16\linewidth]{figures/cifar/ResNet18/sam_rho/filter16_val_acc.pdf}
    \includegraphics[width=0.16\linewidth]{figures/cifar/ResNet18/sam_rho/filter32_val_acc.pdf} \\
    \vspace{-1em}
    \includegraphics[width=0.16\linewidth]{figures/cifar/ResNet18/sam_rho/filter64_val_acc.pdf}
    \includegraphics[width=0.16\linewidth]{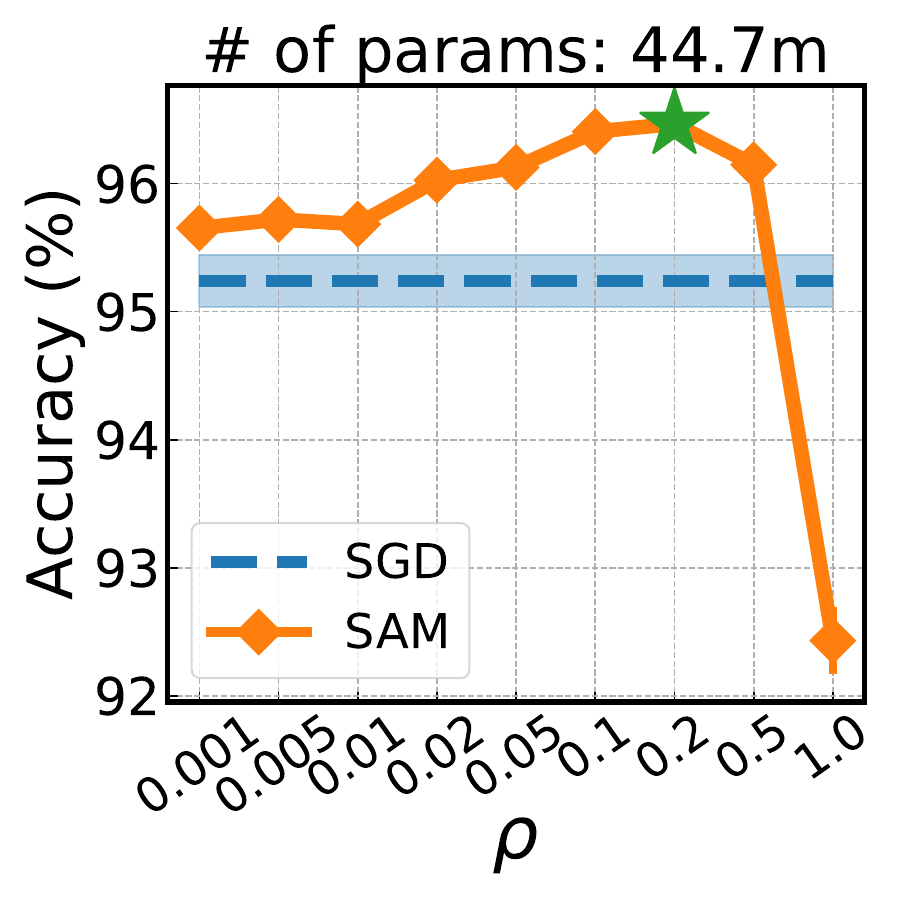}
    \includegraphics[width=0.16\linewidth]{figures/cifar/ResNet18/sam_rho/filter256_val_acc.pdf}
    \includegraphics[width=0.16\linewidth]{figures/cifar/ResNet18/sam_rho/best_rho.pdf}
  \caption{
    Validation accuracy versus $\rho$ for ResNet-18 and CIFAR-10.
  }
  \vspace{-1em}
  \label{fig:result_overparam_each_rho_cifar}
\end{figure}

\begin{figure}[!ht]
    \centering
    \includegraphics[width=0.16\linewidth]{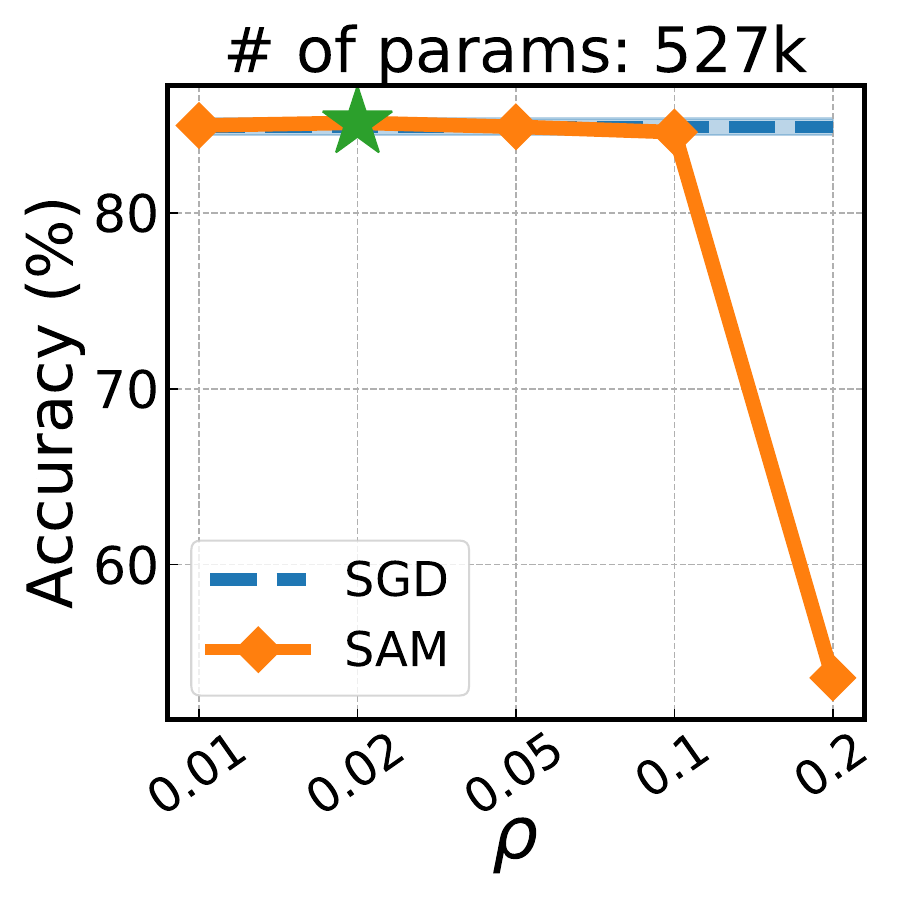}
    \includegraphics[width=0.16\linewidth]{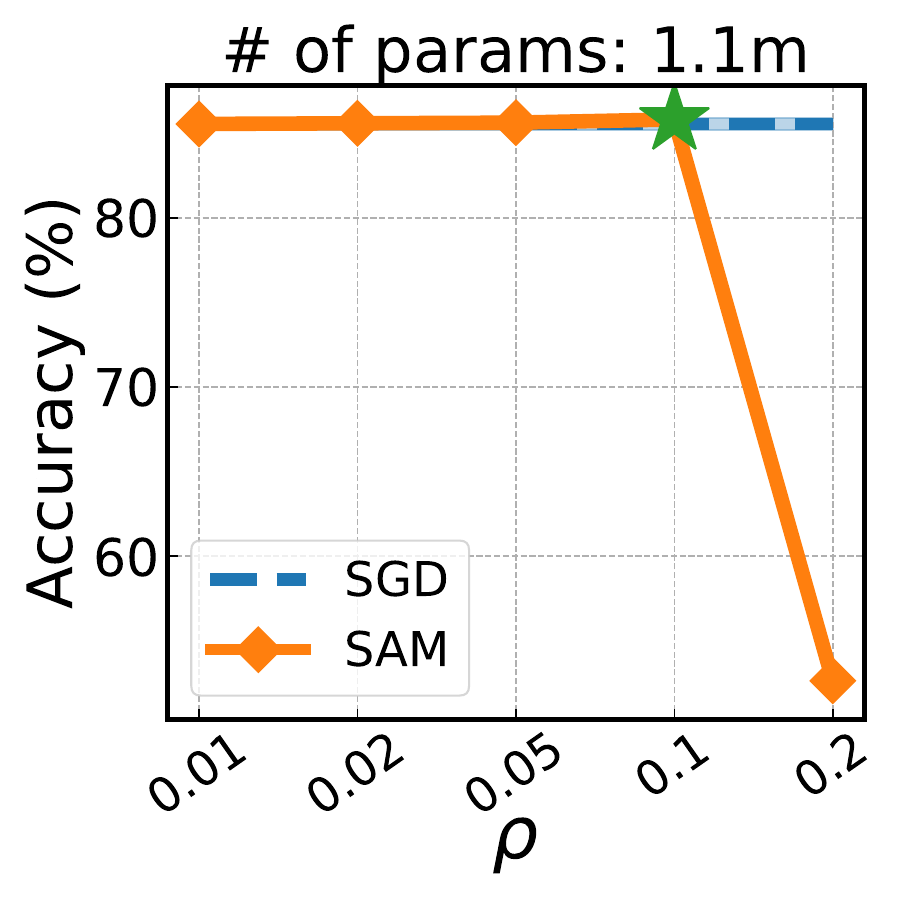}
    \includegraphics[width=0.16\linewidth]{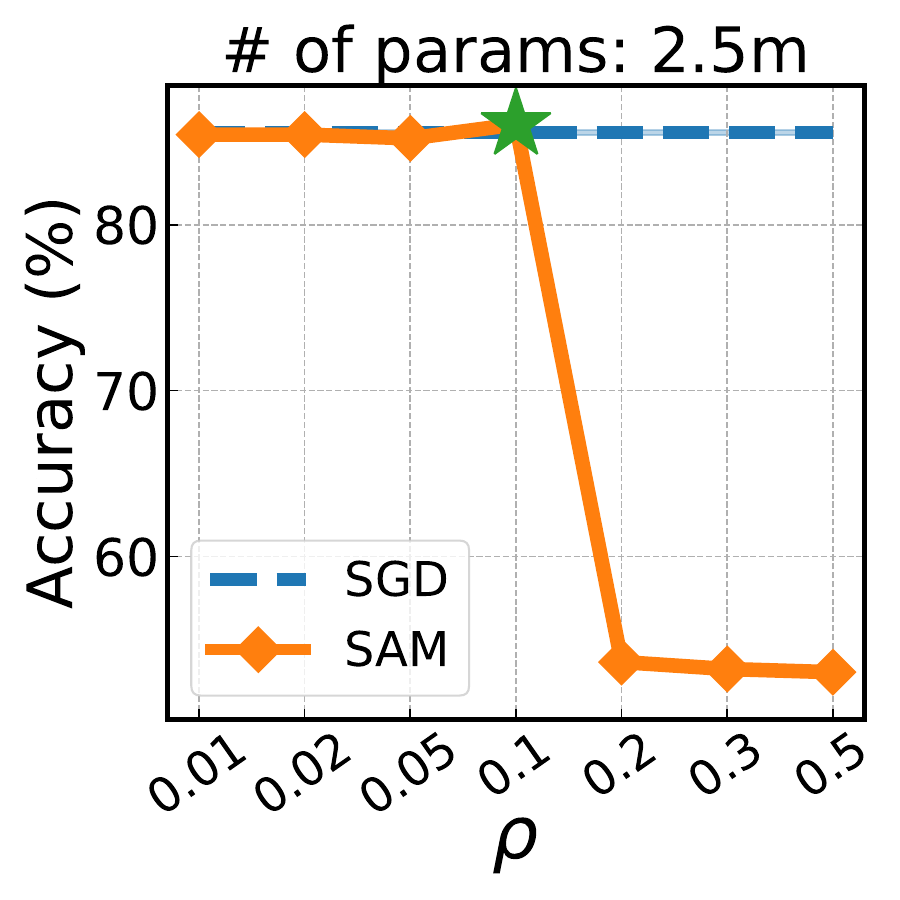} 
    \includegraphics[width=0.16\linewidth]{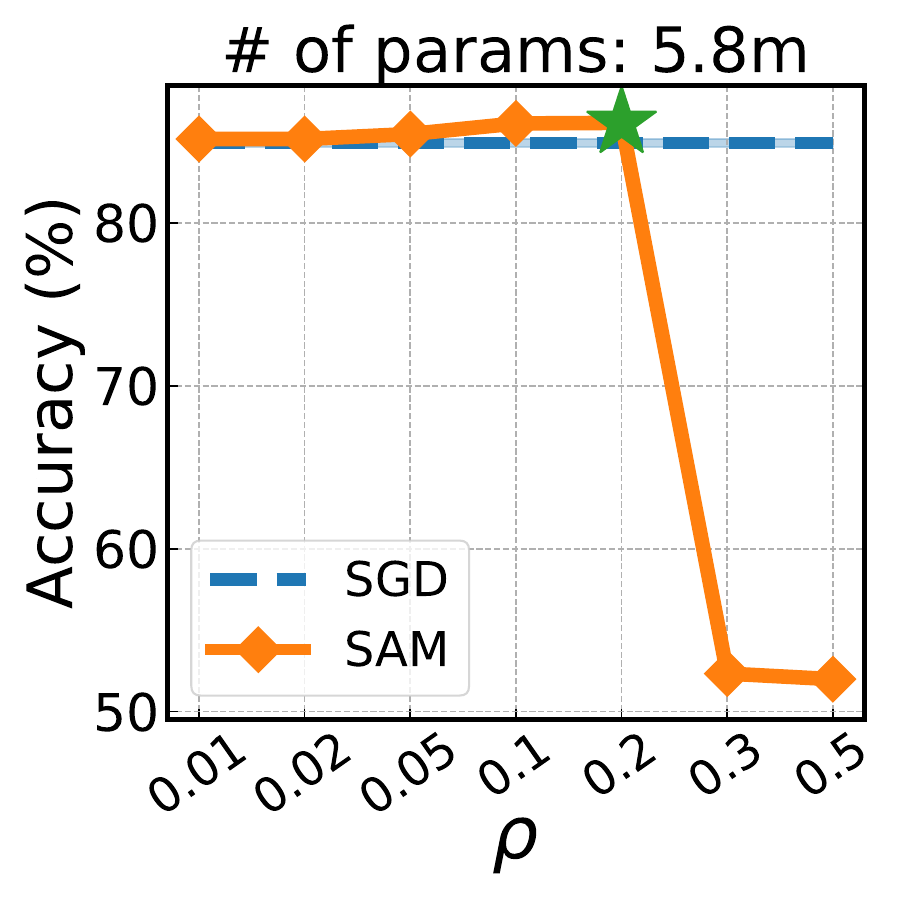} \\
    \vspace{-1em}
    \includegraphics[width=0.16\linewidth]{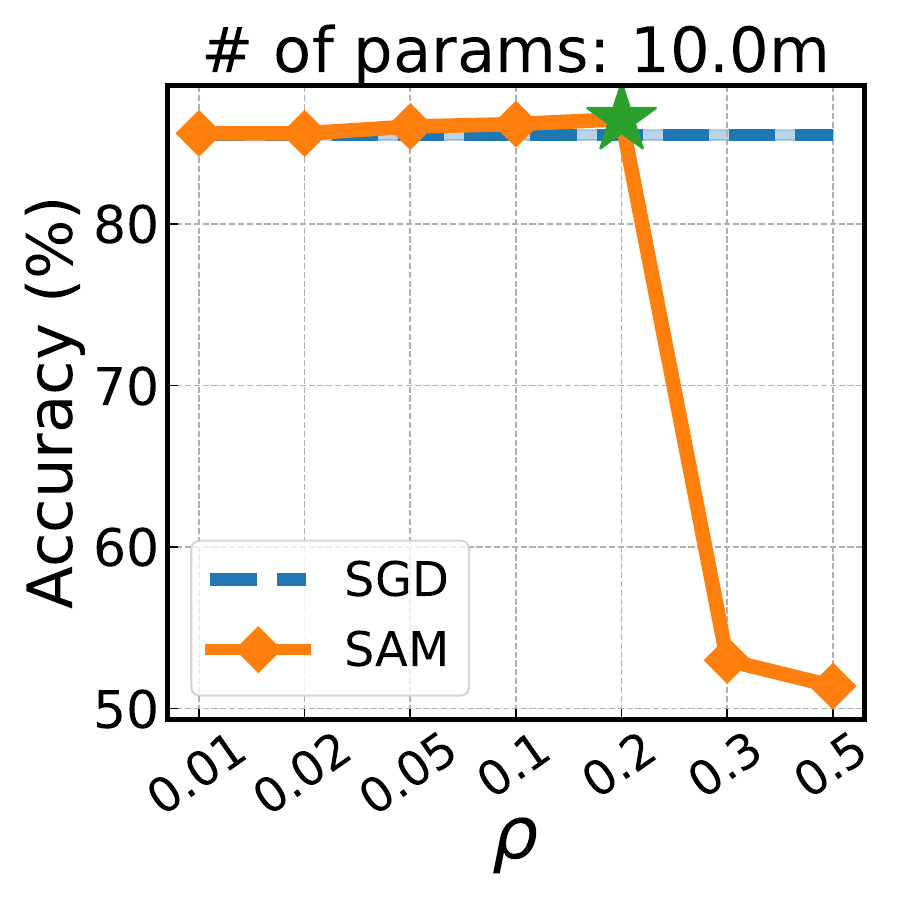}
    \includegraphics[width=0.16\linewidth]{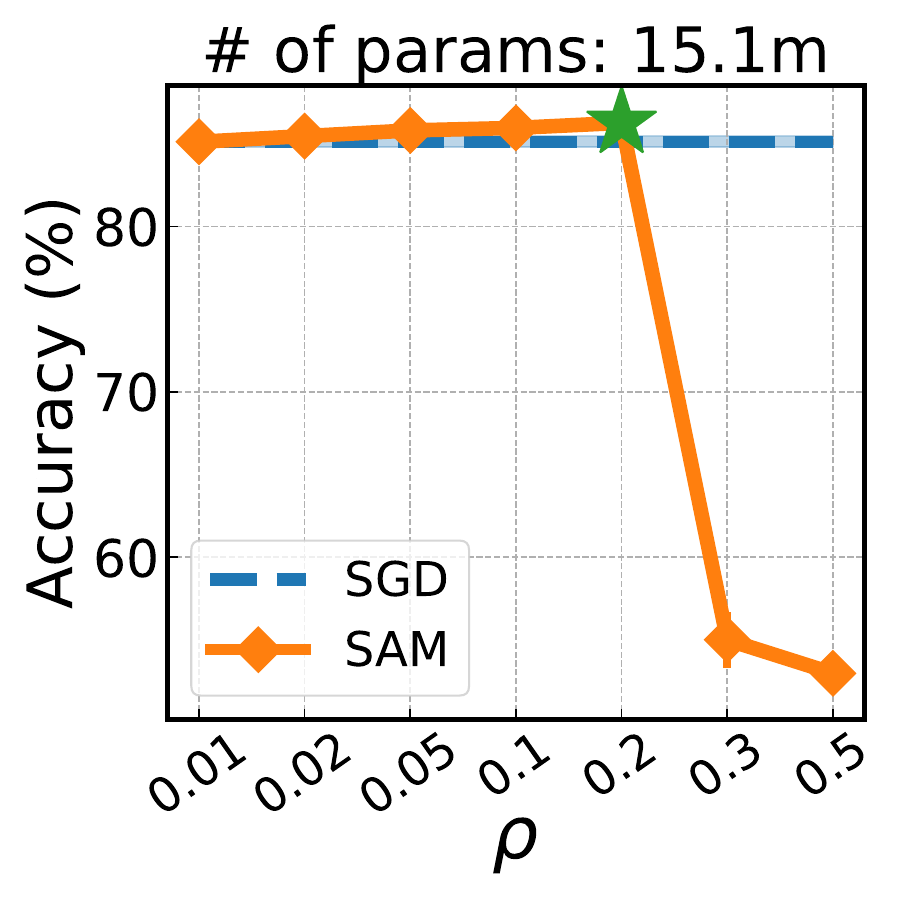}
    \includegraphics[width=0.16\linewidth]{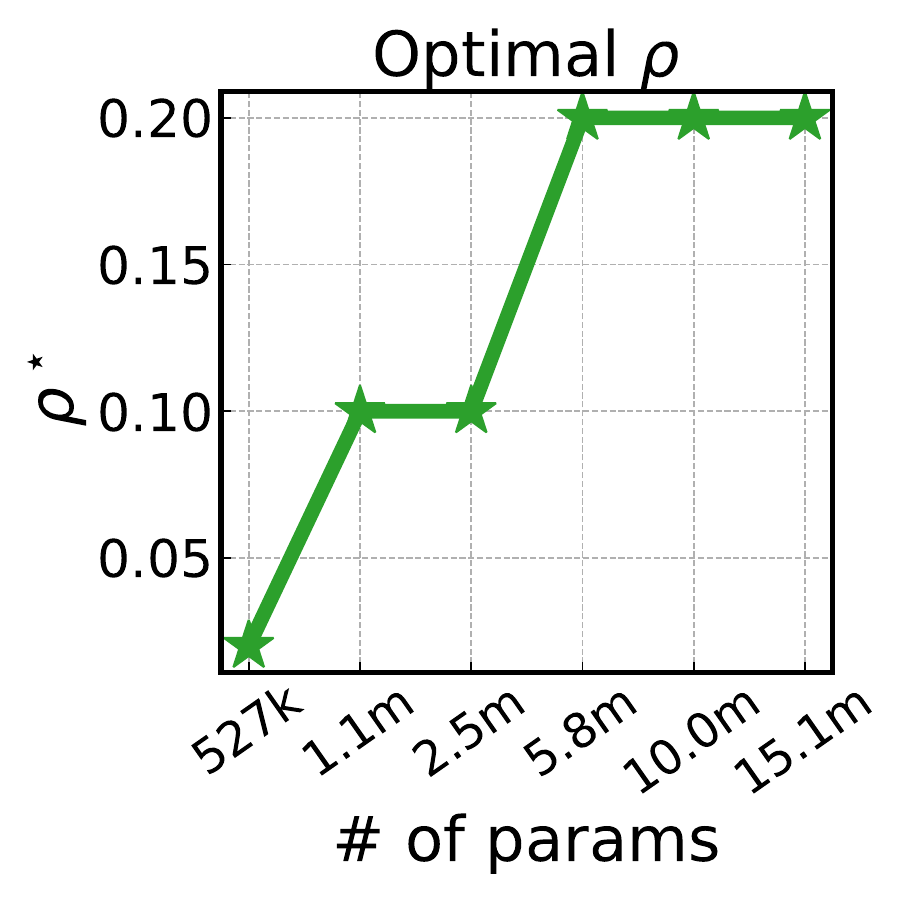}
  \caption{
    Validation accuracy versus $\rho$ for LSTM and SST2.
  }
  \label{fig:result_overparam_each_rho_sst}
\end{figure}

Extending from \cref{sec:understanding}, we plot the validation accuracy of SAM versus different values of $\rho$, along with their optimal value of $\rho$ for 3-layer-MLP/MNIST, ResNet-50/ImageNet, ResNet-18/CIFAR-10, and LSTM/SST2 in \cref{fig:result_overparam_each_rho_mnist,fig:result_overparam_each_rho_cifar,fig:result_overparam_each_rho_imagenet,fig:result_overparam_each_rho_sst}, respectively.
It is observed that $\rho^\star$ tends to increase as the model becomes more overparameterized;
on CIFAR-10 with ResNet18, the smallest model has $\rho^\star = 0.01$ while the largest three have $\rho^\star = 0.2$.

\newpage

\section{Additional results for Section \ref{sec:practical}} \label{app:add_practical}

\subsection{Label noise}

\begin{figure}[!th]
    \centering
  \begin{subfigure}{0.24\linewidth}
      \centering
      \includegraphics[width=\linewidth]{figures/labelnoise/Noise_rate.pdf}
      \vspace{-1.5em}
      \caption{Effect of label noise}
      \label{fig:resut_labelnoise_noiserate_app}
  \end{subfigure}
  \begin{subfigure}{0.24\linewidth}
      \centering
      \includegraphics[width=\linewidth]{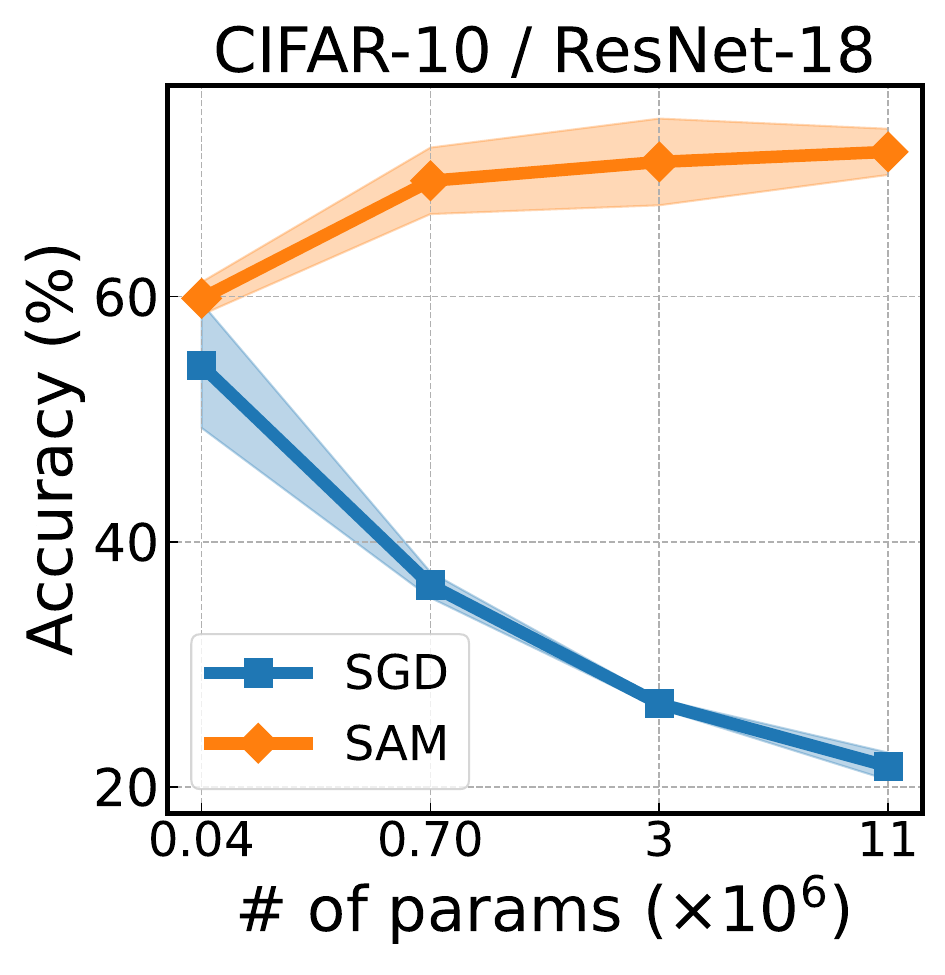}
      \vspace{-1.5em}
      \caption{Noise rate = $0.75$}
      \label{fig:resut_labelnoise_fixednoise_0.75}
  \end{subfigure}
  \begin{subfigure}{0.24\linewidth}
      \centering
      \includegraphics[width=\linewidth]{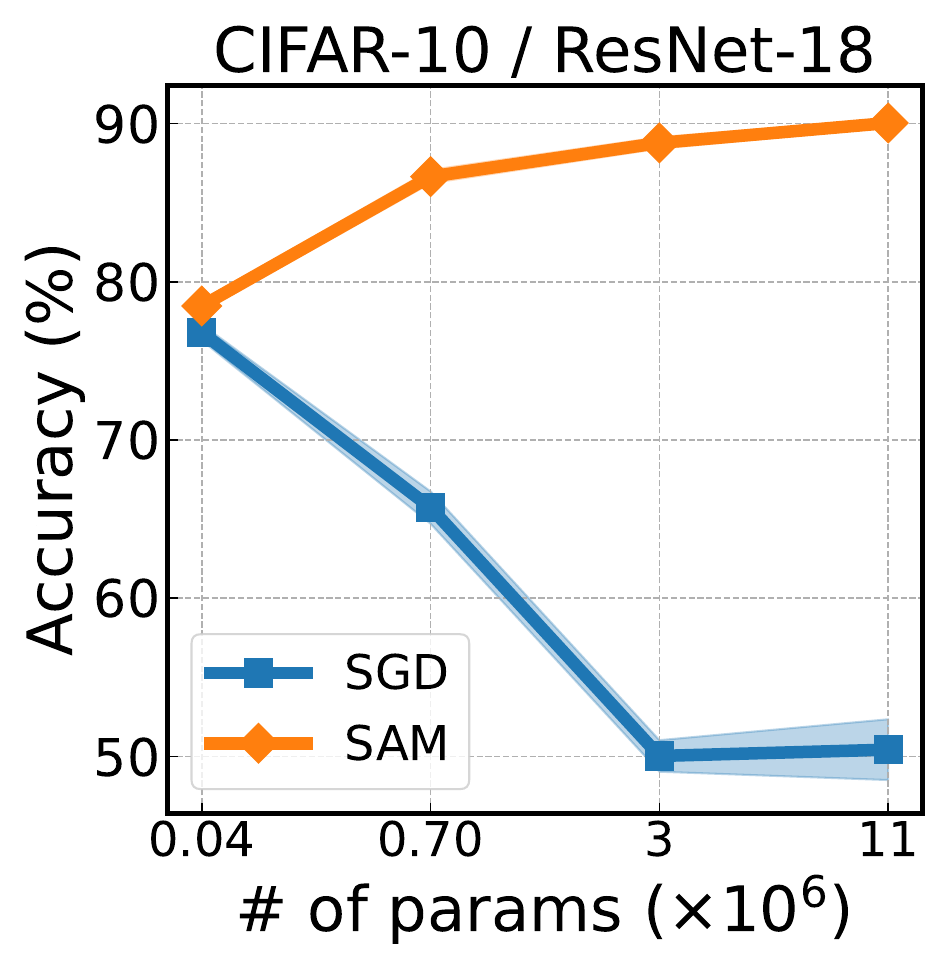}
      \vspace{-1.5em}
      \caption{Noise rate = $0.5$}
      \label{fig:resut_labelnoise_fixednoise_0.5}
  \end{subfigure}
  \begin{subfigure}{0.24\linewidth}
      \centering
      \includegraphics[width=\linewidth]{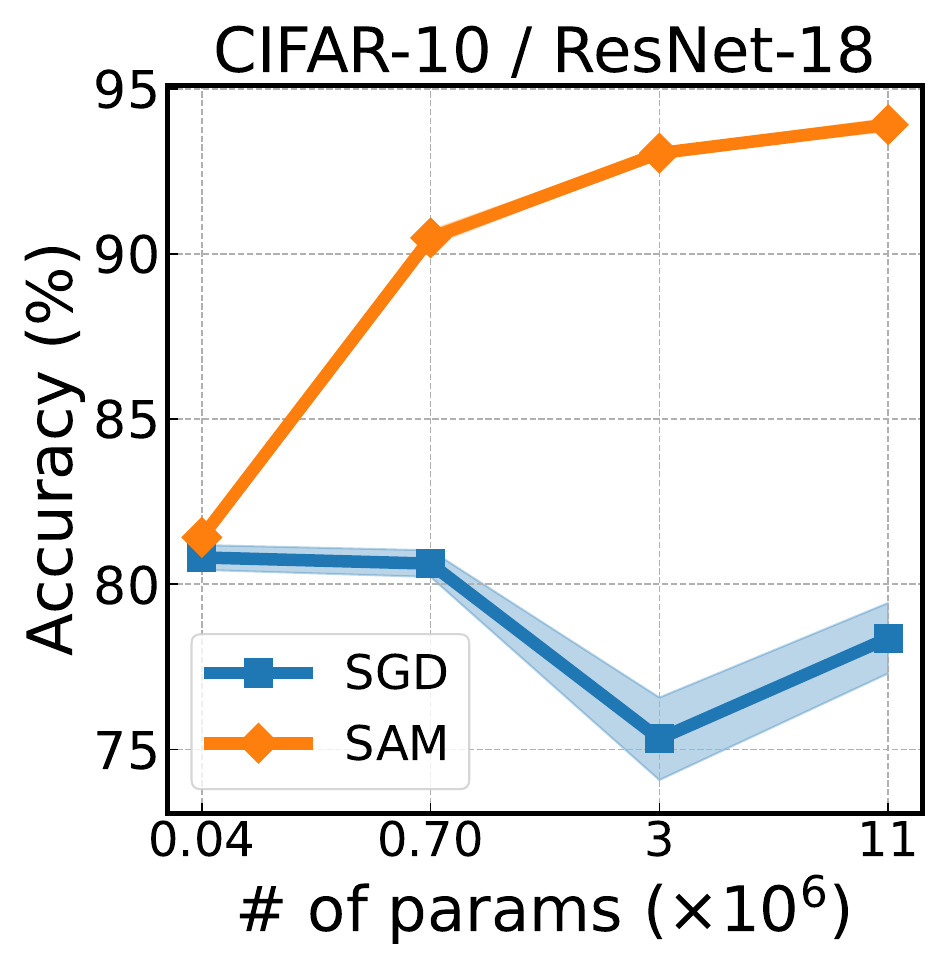}
      \vspace{-1.5em}
      \caption{Noise rate = $0.25$}
      \label{fig:resut_labelnoise_fixednoise_0.25}
  \end{subfigure}
  \caption{
    Effect of overparameterization on SAM under label noise for CIFAR-10 and ResNet-18.
    (a) SAM benefits a lot more from overparameterization than SGD; it is more pronounced with high noise level.
    (b-d) Under label noise, SGD tends to overfit as with more parameters unlike SAM.
  }
  \label{fig:result_labelnoise}
\end{figure}

More results on the effect of overparameterization on SAM under label noise are presented in \cref{fig:result_labelnoise}.
Overall, we find SAM benefits from overparameterization significantly more than SGD in the presence of label noise.
Precisely, the accuracy improvement made by SAM keeps on increasing as the model becomes more overparameterized, and this trend is more pronounced with higher noise levels;
e.g., it rises from $5\%$ to nearly $50\%$ at the highest noise rate.

\subsection{Sparse overparameterization}

\begin{figure}[!ht]
    \begin{subfigure}{0.48\linewidth}
      \centering
      \includegraphics[width=0.45\linewidth]{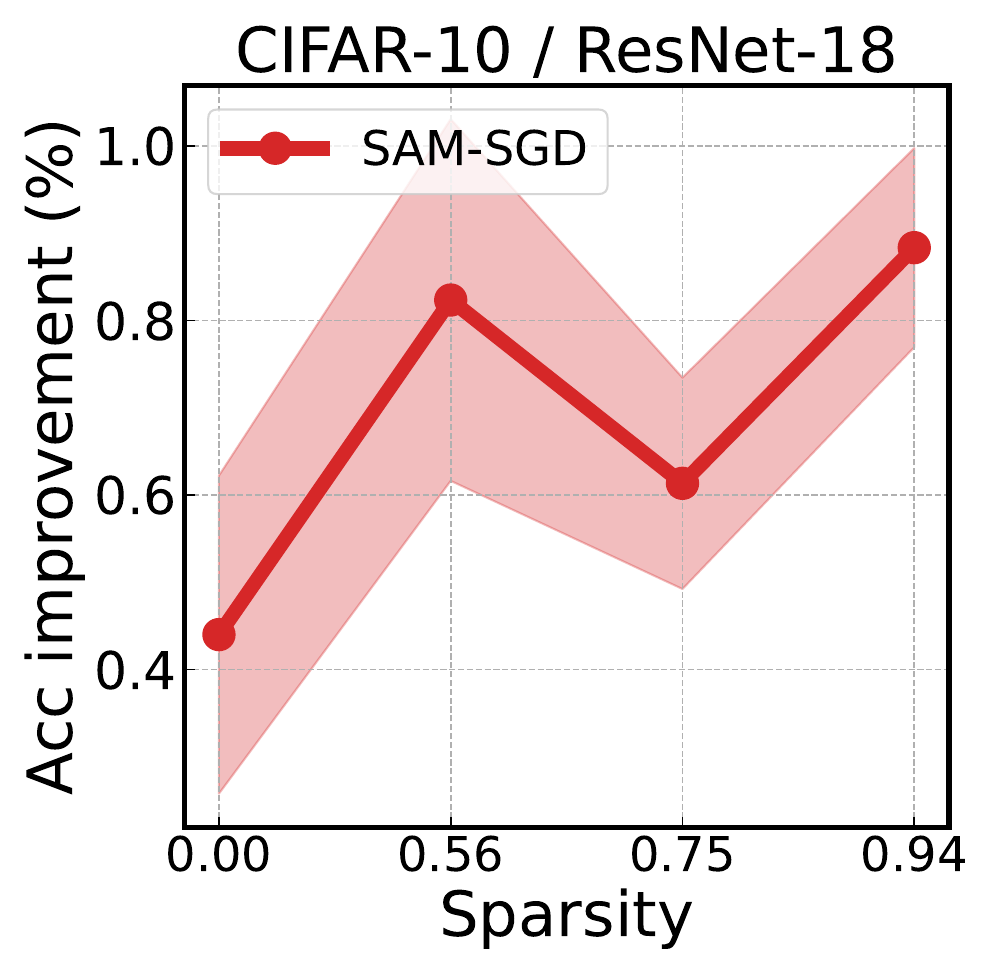}
      \includegraphics[width=0.45\linewidth]{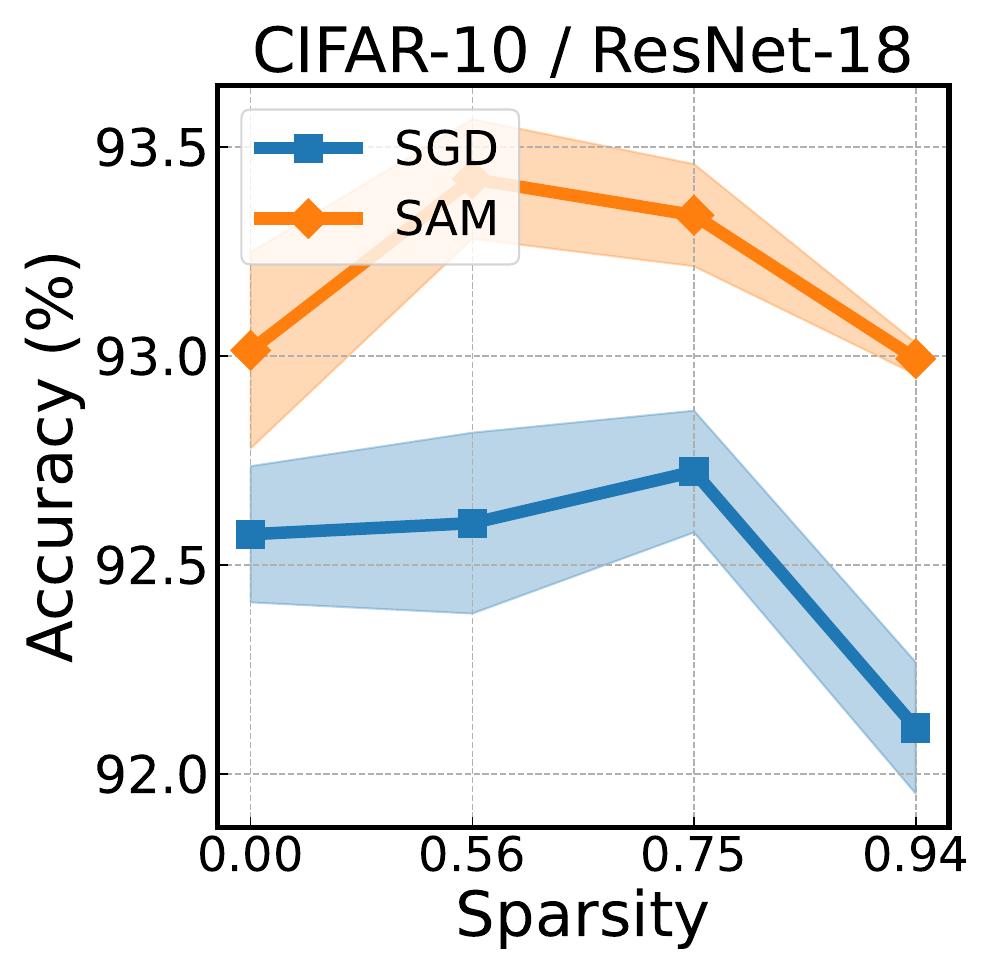}
      \centering
      \includegraphics[width=0.45\linewidth]{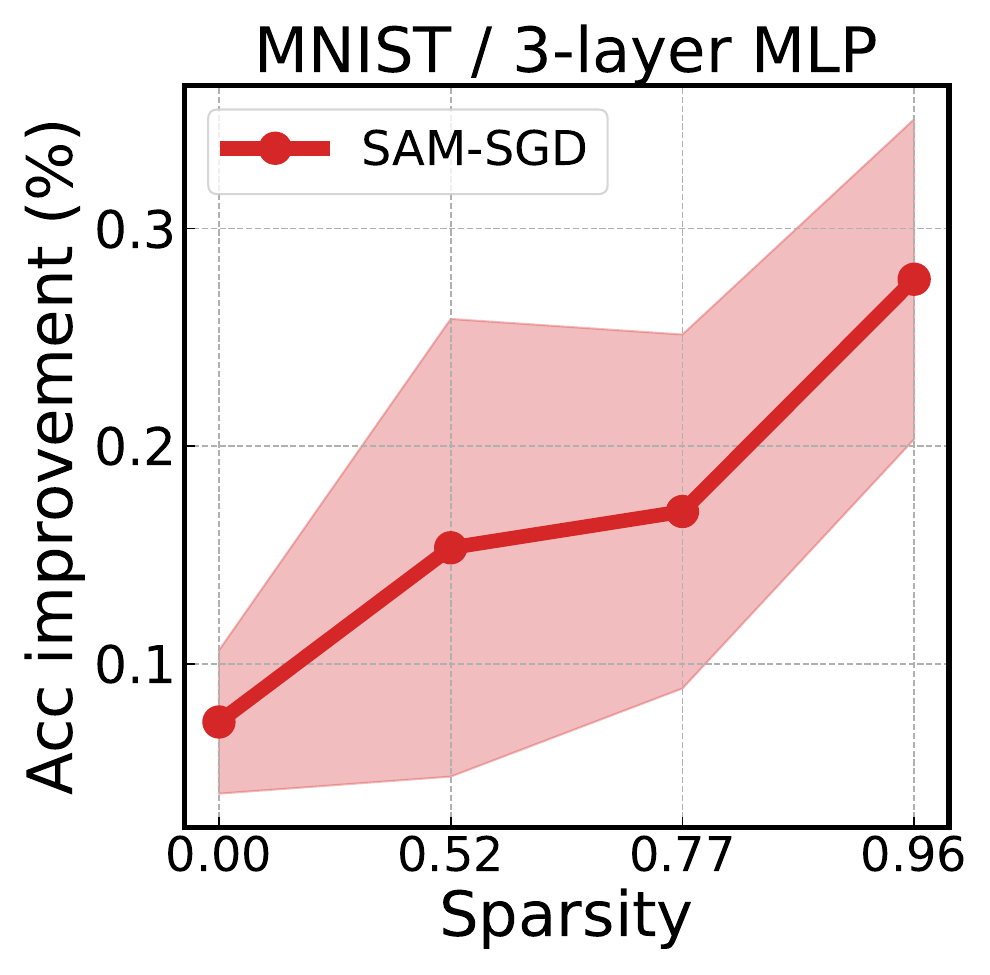}
      \includegraphics[width=0.45\linewidth]{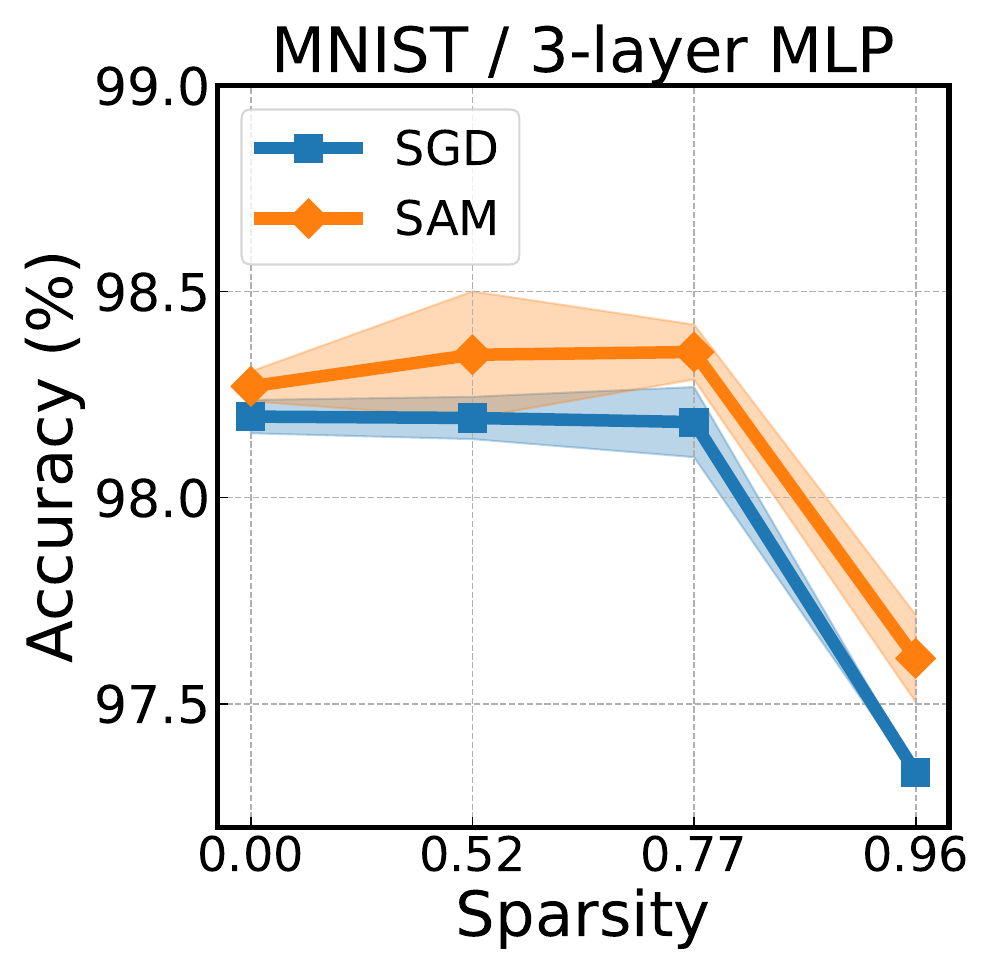}
      \caption{Random pruning}
      \label{fig:largesparse_random}
  \end{subfigure}
  \hspace*{\fill}
  \begin{subfigure}{0.48\linewidth}
      \centering
      \includegraphics[width=0.45\linewidth]{figures/cifar/ResNet18/largesparse_diff/largesparse_snip_dense16.pdf}
      \includegraphics[width=0.45\linewidth]{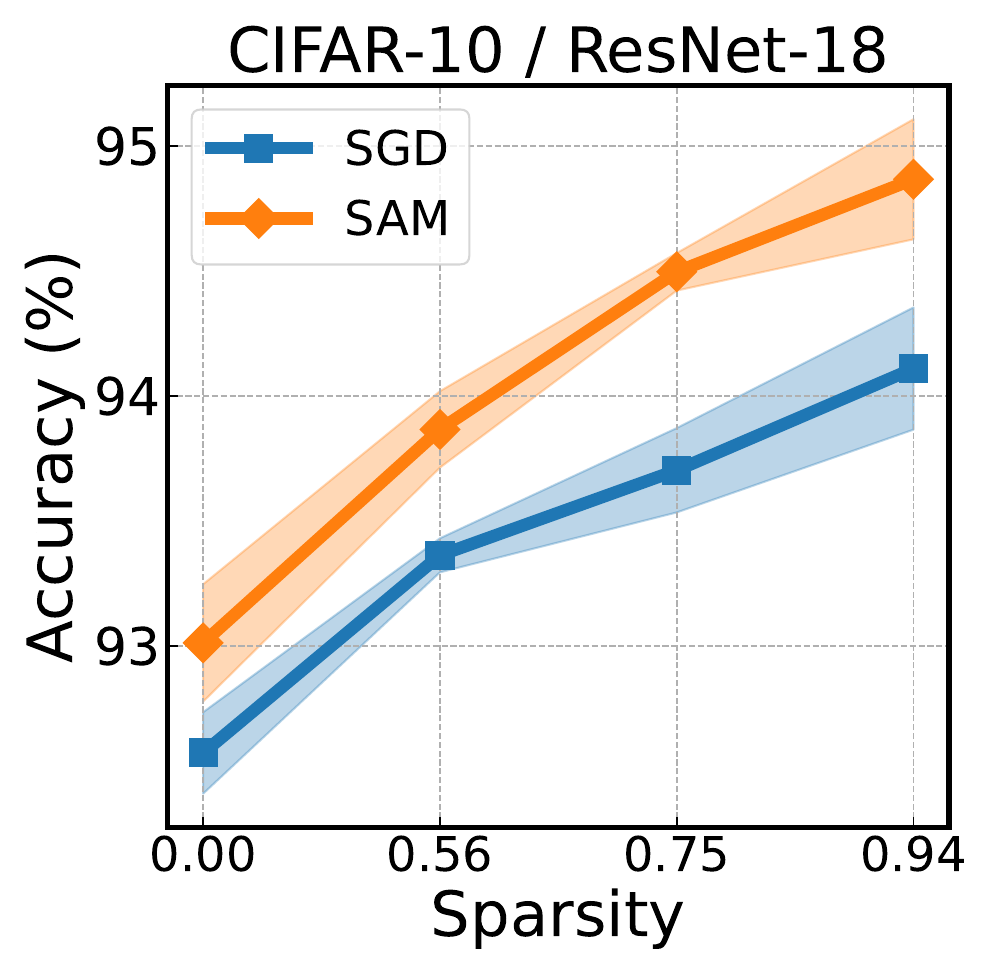}
      \includegraphics[width=0.45\linewidth]{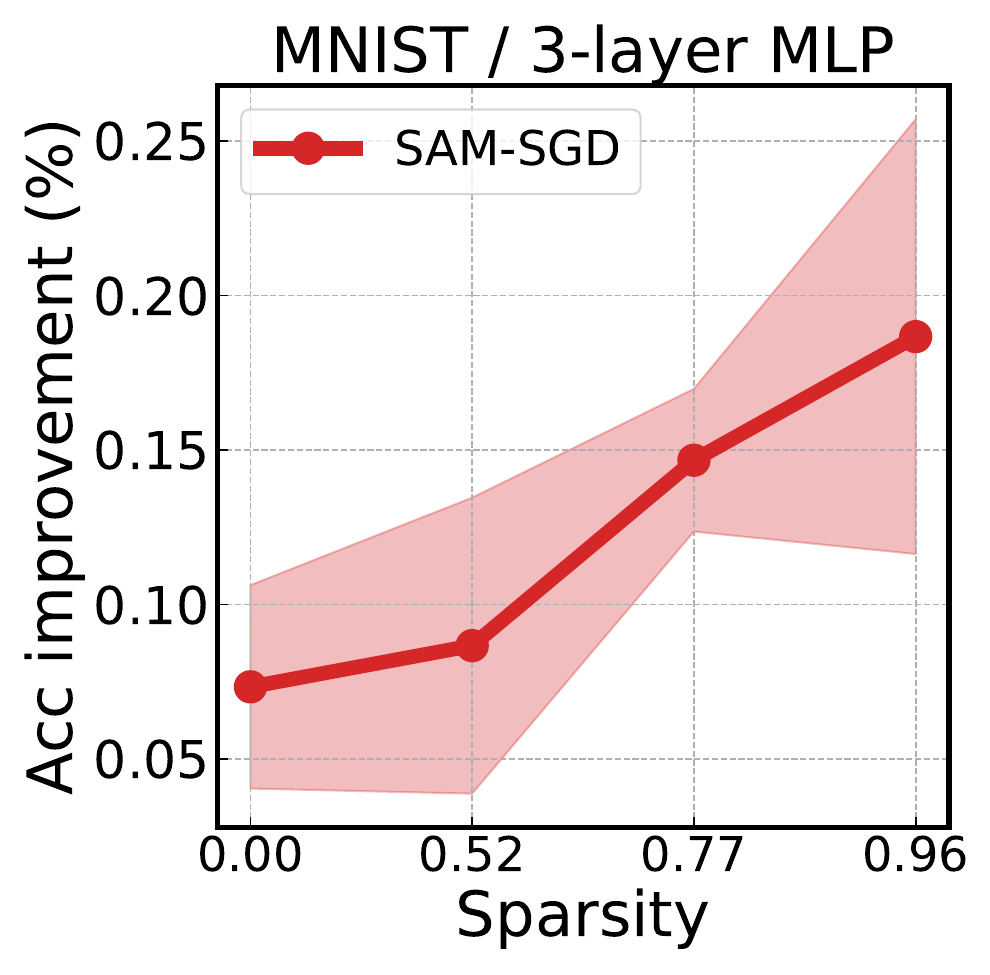}
      \includegraphics[width=0.45\linewidth]{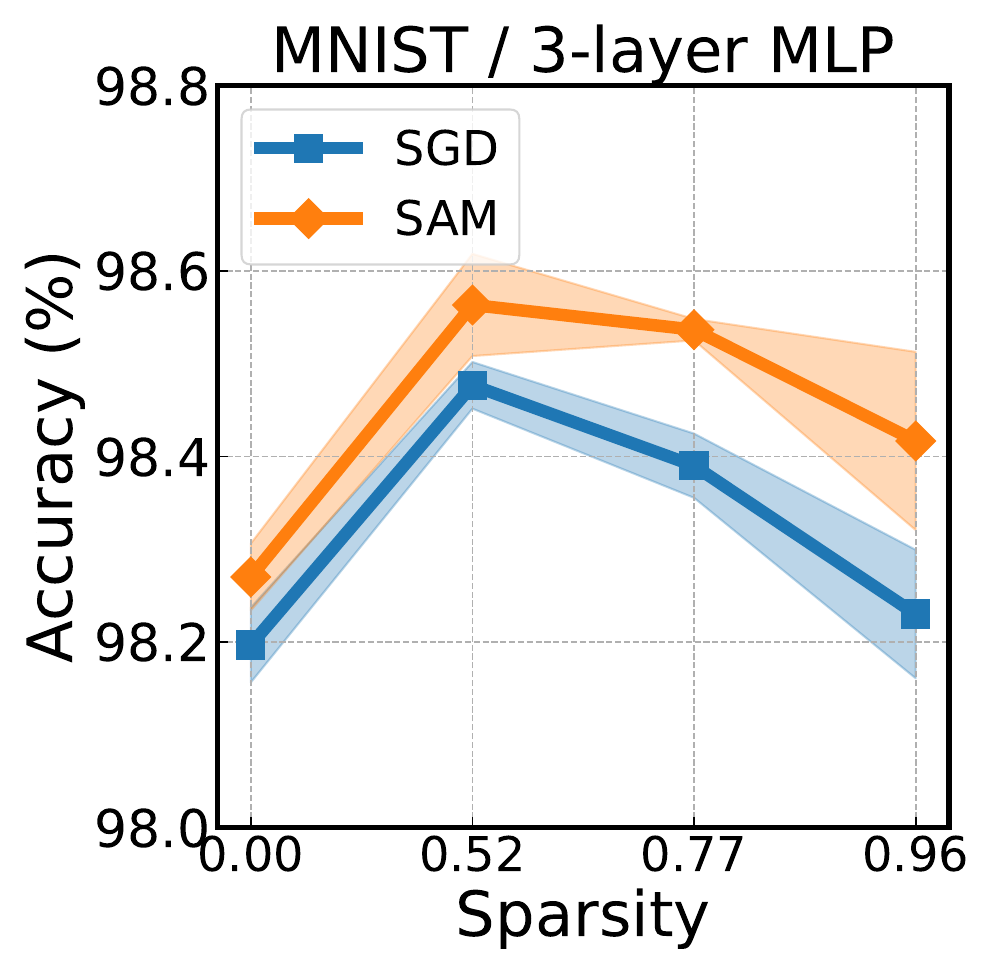}
      \caption{SNIP}
      \label{fig:result_largesparse_snip}
  \end{subfigure}
    \caption{
    Effect of sparsification on SAM for CIFAR-10/ResNet-18 and MNIST/3-layer MLP.
    Here, we set ResNet-18 and 3-layer MLP to have $701$k and $61$k parameters, respectively.
    The improvement tends to increase in large sparse models compared to their small dense counterparts.
  }
  \label{fig:result_largesparse}
\end{figure}

\begin{figure}[!th]
      \centering
      \begin{subfigure}{0.49\linewidth}
          \includegraphics[width=0.45\linewidth]{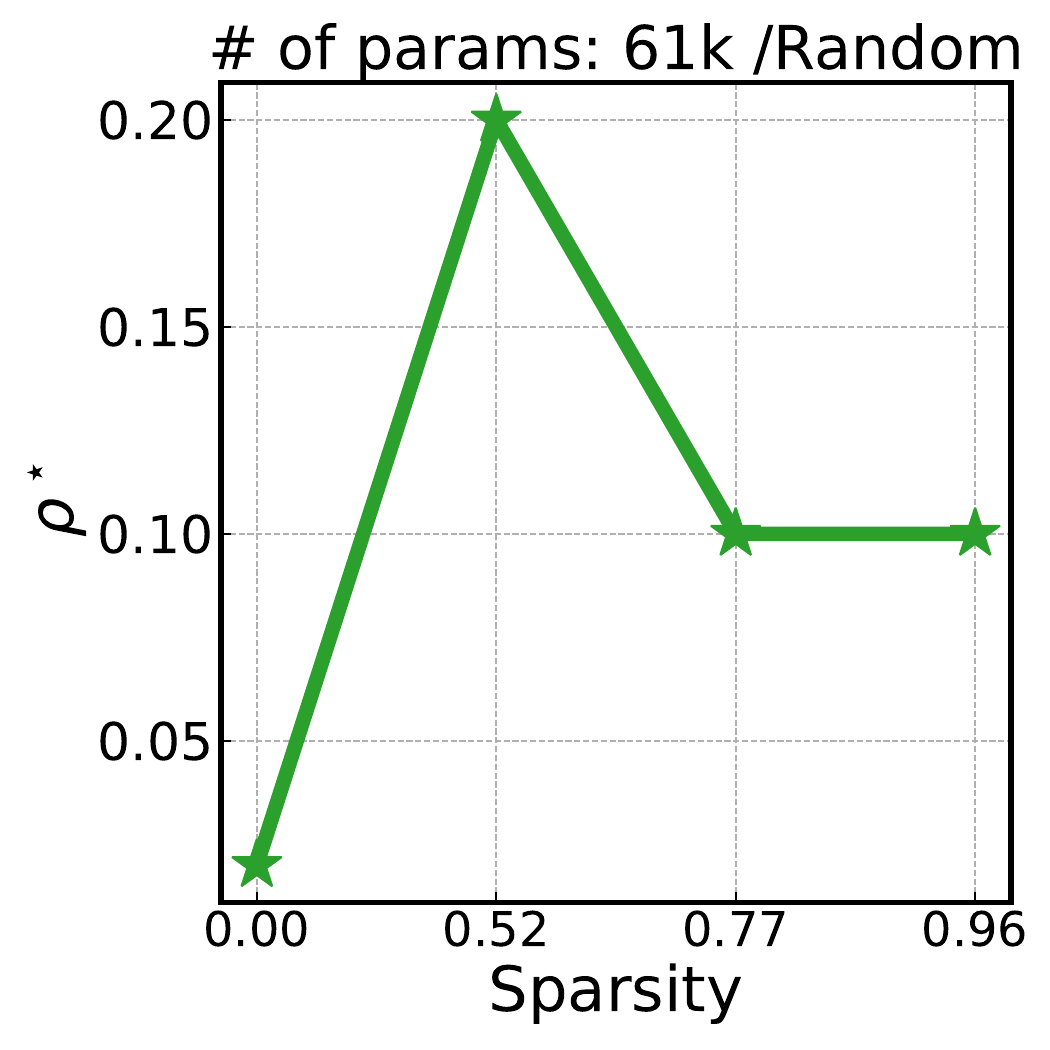}
          \includegraphics[width=0.45\linewidth]{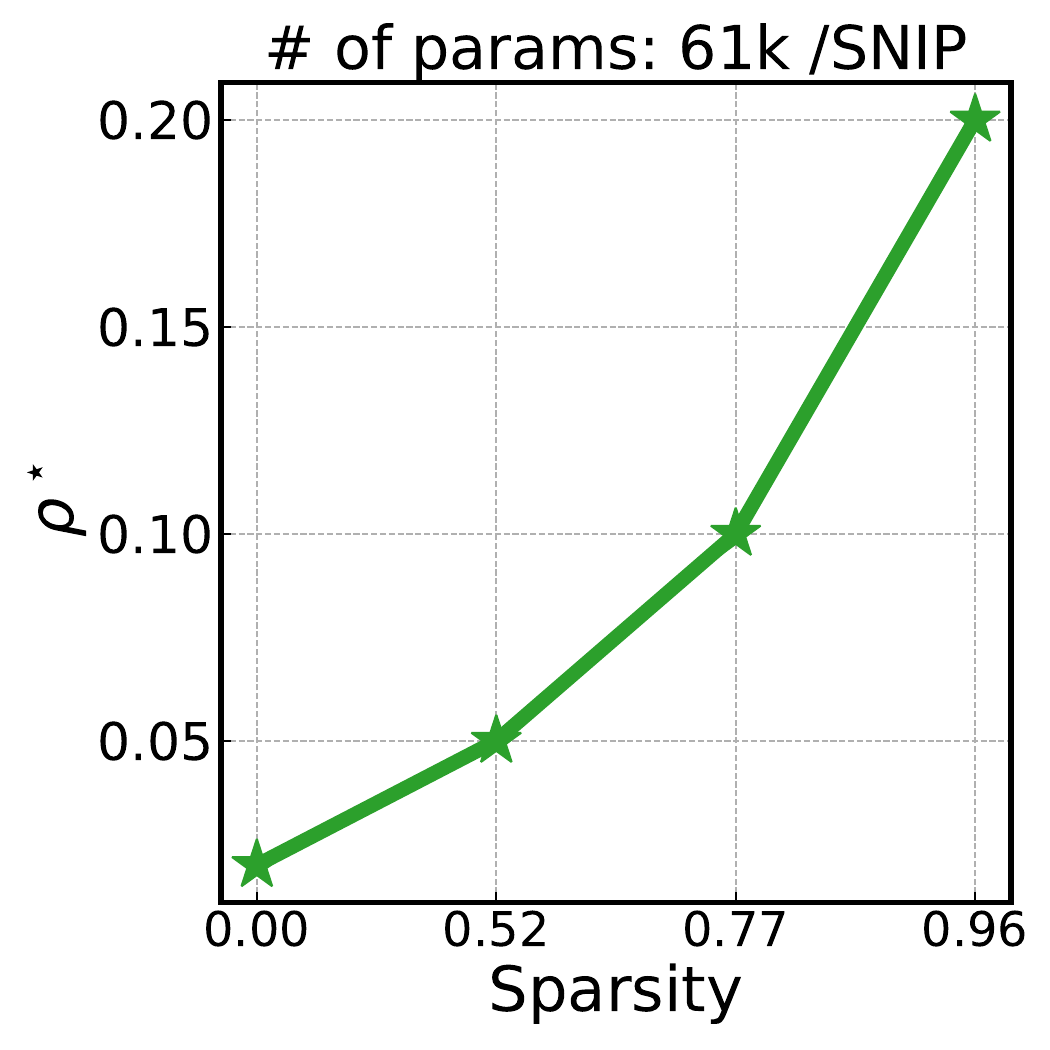}
          \caption{MNIST/MLP}
      \end{subfigure}    
      \begin{subfigure}{0.49\linewidth}
          \includegraphics[width=0.45\linewidth]{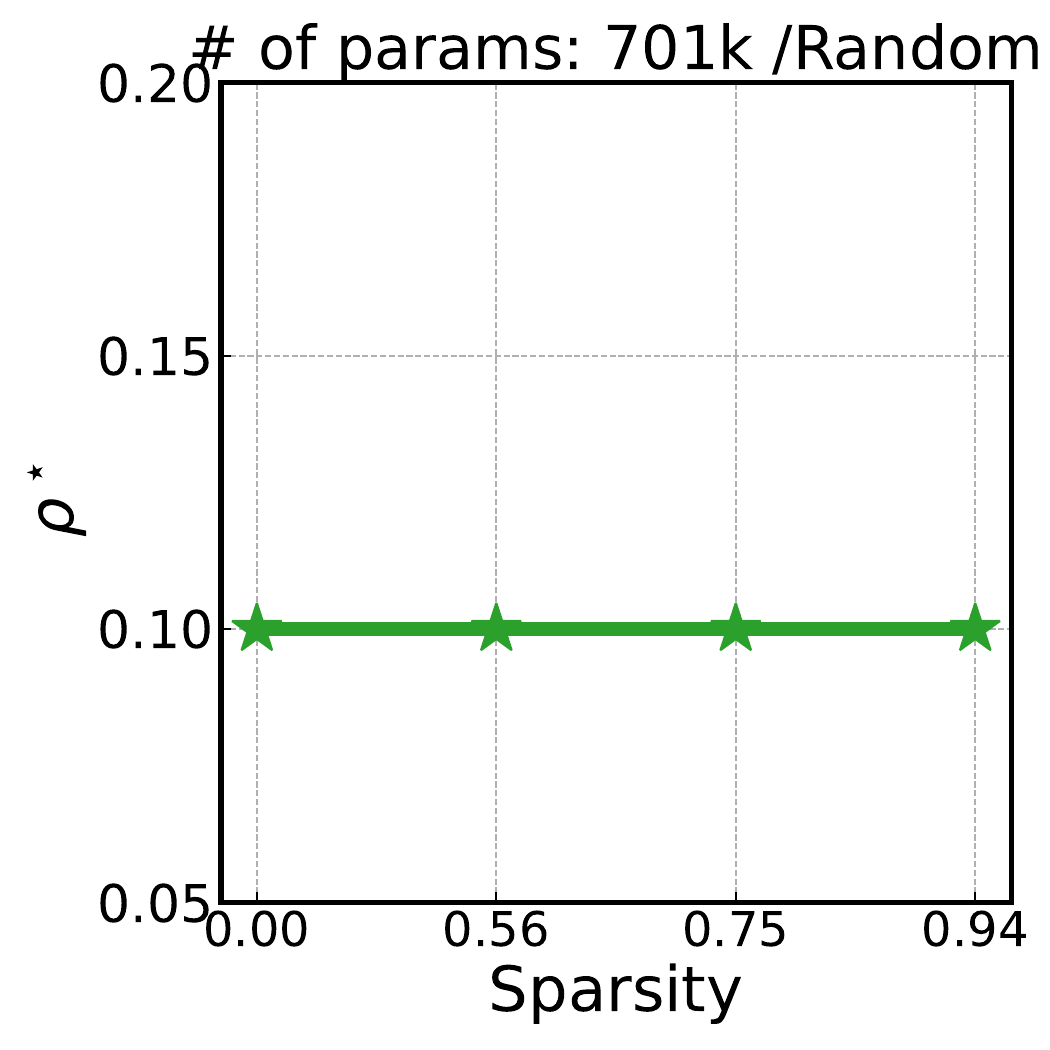}
          \includegraphics[width=0.45\linewidth]{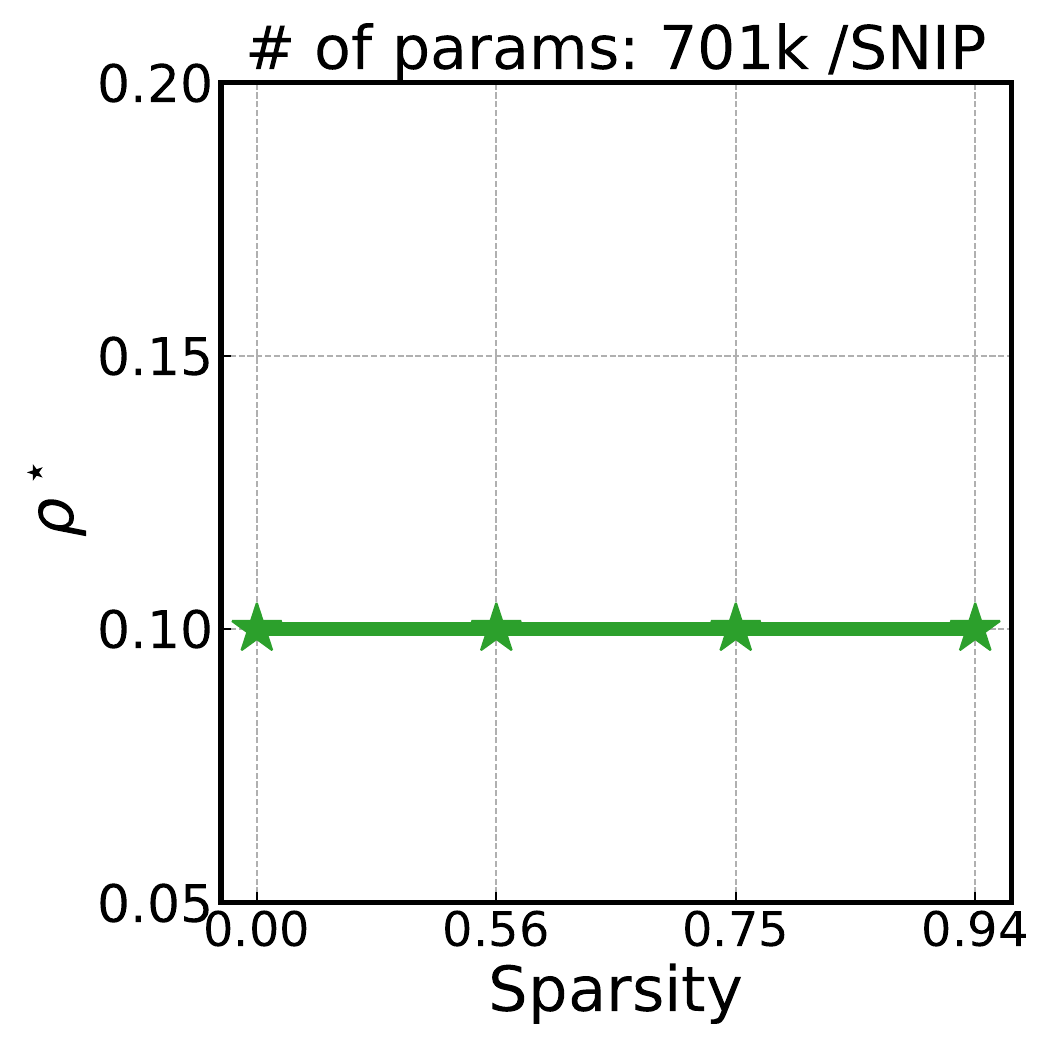}
          \caption{CIFAR-10/ResNet-18}
      \end{subfigure}      
  \caption{
    Effect of sparsification on $\rho^\star$.
    $\rho^\star$ can sometimes be different across different sparsity patterns despite having a similar number of parameters.
  }
  \label{fig:result_largesparse_rho}
\end{figure}

Additional results on the effect of sparsification on the generalization benefit of SAM are plotted in \cref{fig:result_largesparse}.
Here, we try two sparsification methods that do not require pretaining, random pruning, and SNIP \citep{leesnip}.
For both methods, we note that the generalization improvement by SAM tends to increase as the model becomes more sparsely overparameterized.

We also plot the effect of sparsification on $\rho^\star$ in \cref{fig:result_largesparse_rho}.
We find that $\rho^\star$ is sometimes different between small dense and large sparse models despite having a similar number of parameters;
for the MLP of $61$k parameters on MNIST, $\rho^\star$ changes over different sparsity levels and sparsification methods, but this does not generalize to the CIFAR-10 and ResNet-18.
This indicates that it is not just the parameter count that affects the behavior of SAM, but some other factors such as the pattern of parameterization also have an influence on how SAM shapes training.

\subsection{Regularization}

\begin{figure}[!th]
  \centering
  \begin{subfigure}{0.33\linewidth}
      \centering
      \includegraphics[width=0.48\linewidth]{figures/cifar/ResNet18/overparam_diff/overparamwd_0.0_.pdf}
      \includegraphics[width=0.48\linewidth]{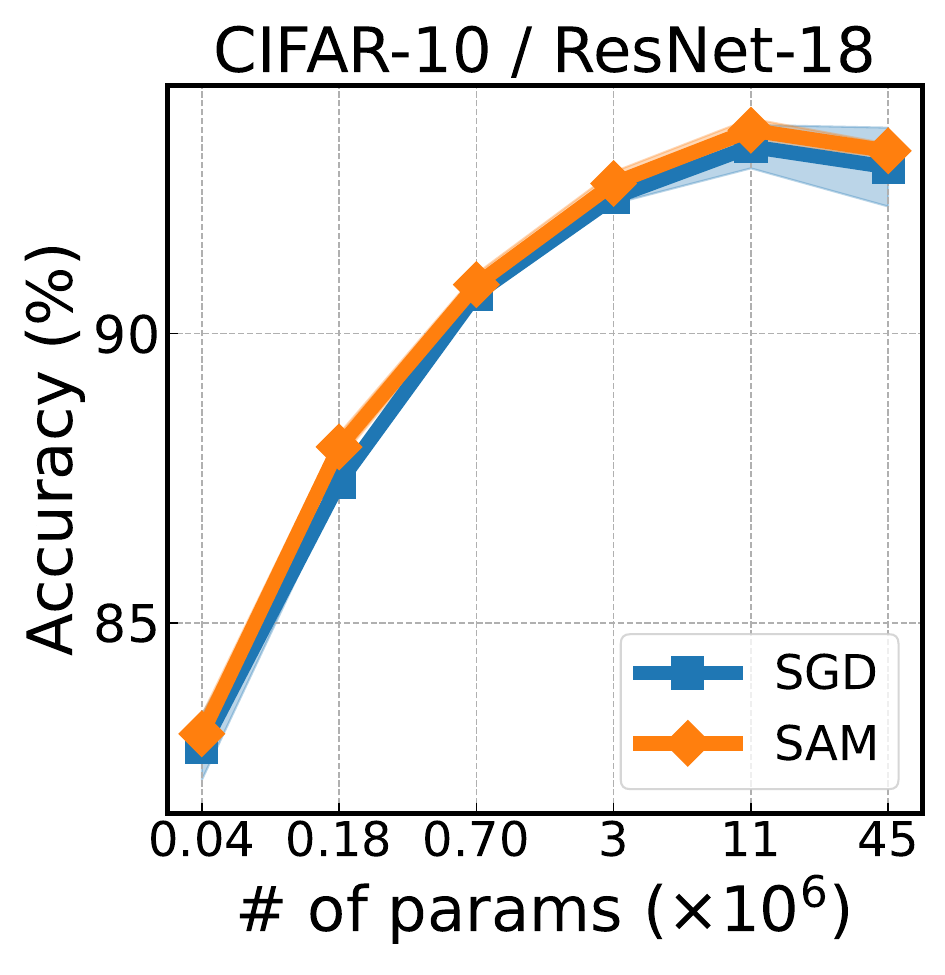}
      \caption{w/o weight decay}
      \label{fig:result_overparam_resnet_wo_wd}
  \end{subfigure}
  \begin{subfigure}{0.32\linewidth}
      \centering
      \includegraphics[width=0.45\linewidth]{figures/pos/overparam_diff/overparam.pdf}
      \includegraphics[width=0.49\linewidth]{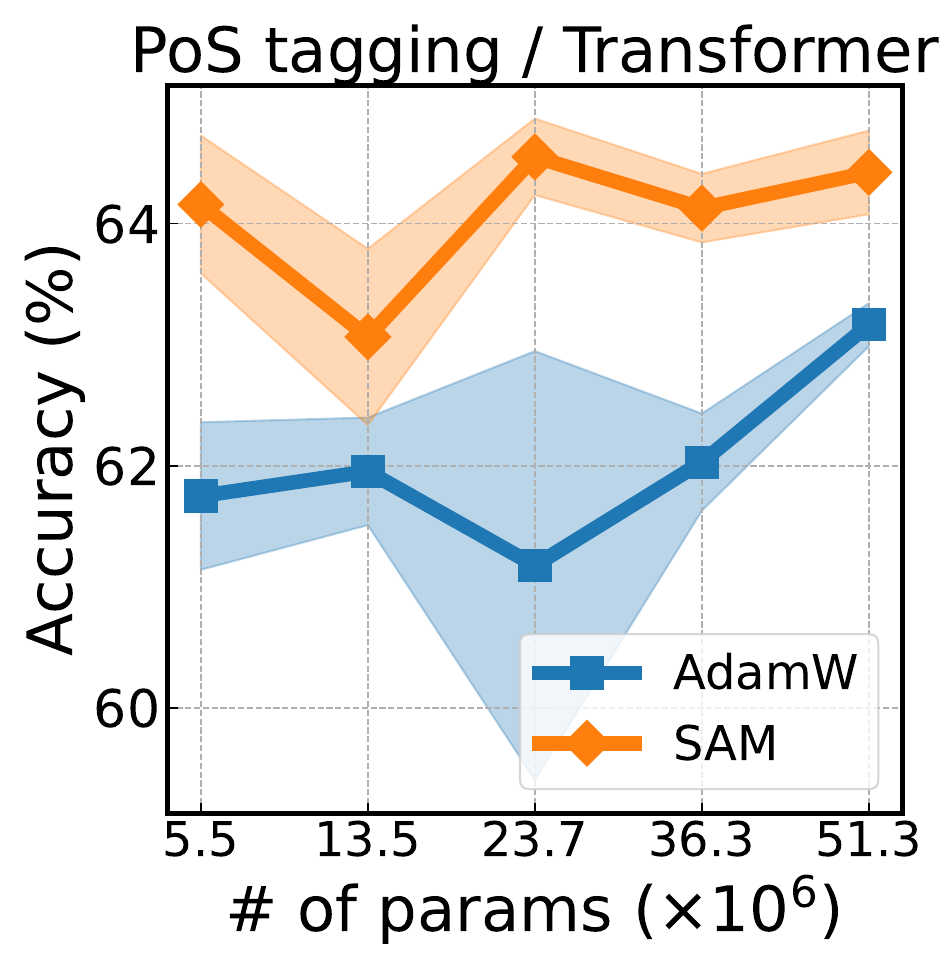}
      \caption{w/o early stop.}
      \label{fig:result_overparam_pos_wo_es}
  \end{subfigure}
  \begin{subfigure}{0.32\linewidth}
      \centering
      \includegraphics[width=0.45\linewidth]{figures/cifar/ViT/overparam_diff/overparam.pdf}
      \includegraphics[width=0.49\linewidth]{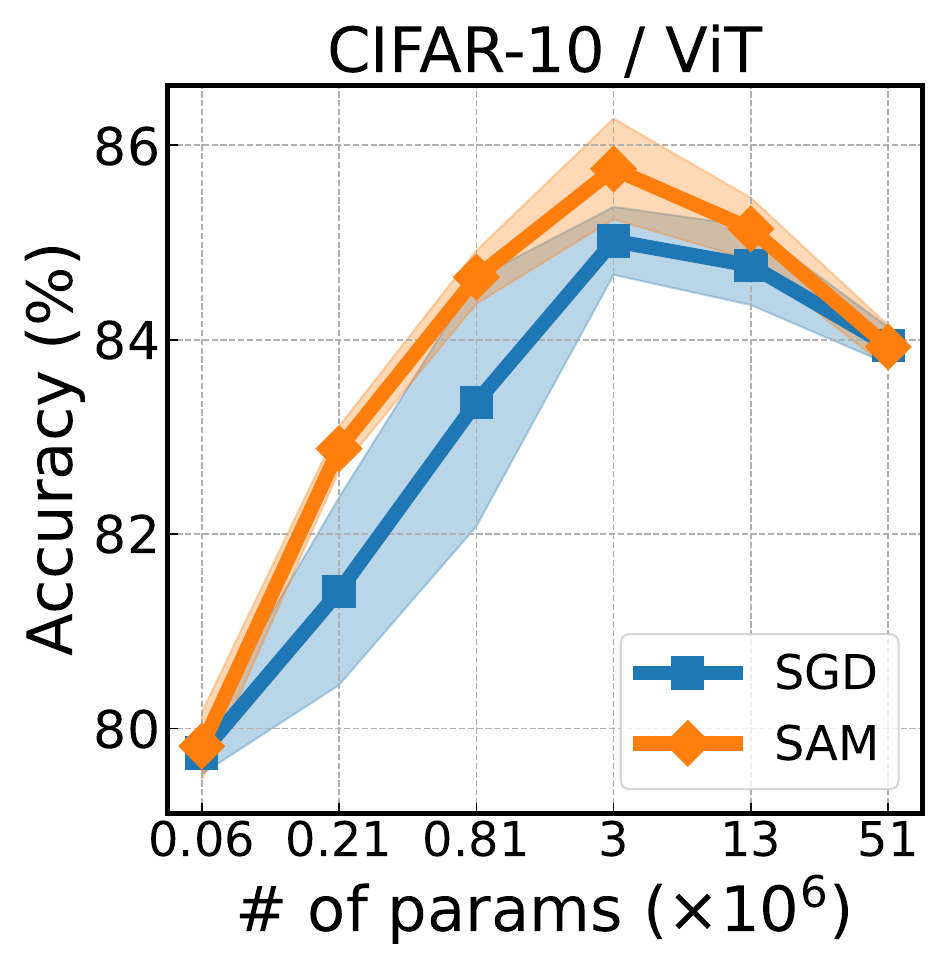}
      \caption{ViT}
      \label{fig:result_overparam_cifar_vit}
  \end{subfigure}
  \caption{
    Effect of overparameterization on SAM without regularization:
    (a) CIFAR-10/ResNet-18 without weight decay, (b) Transformer/PoS tagging without early stopping, and (c) ViT/CIFAR-10.
    SAM does not always benefit from overparameterization in these cases.
  }
  \label{fig:result_overparam_overfit}
\end{figure}

More results on the effect of regularization on SAM are presented in \cref{fig:result_overparam_overfit}.
We find that overparameterization does not increase the generalization benefit of SAM.
We suspect this is because the models are prone to overfitting in these cases and overparameterizing models may decrease the overall performance both for SGD and SAM;
for example in \cref{fig:result_overparam_cifar_vit}, the validation accuracy drops after $11.2$m parameters.

\newpage

\section{Ablation} \label{app:ablation}

\subsection{Effect of depth} \label{app:depth}

\begin{wrapfigure}{r}{0.45\linewidth}
  \vspace{-2em}
  \centering
  \includegraphics[width=0.49\linewidth]{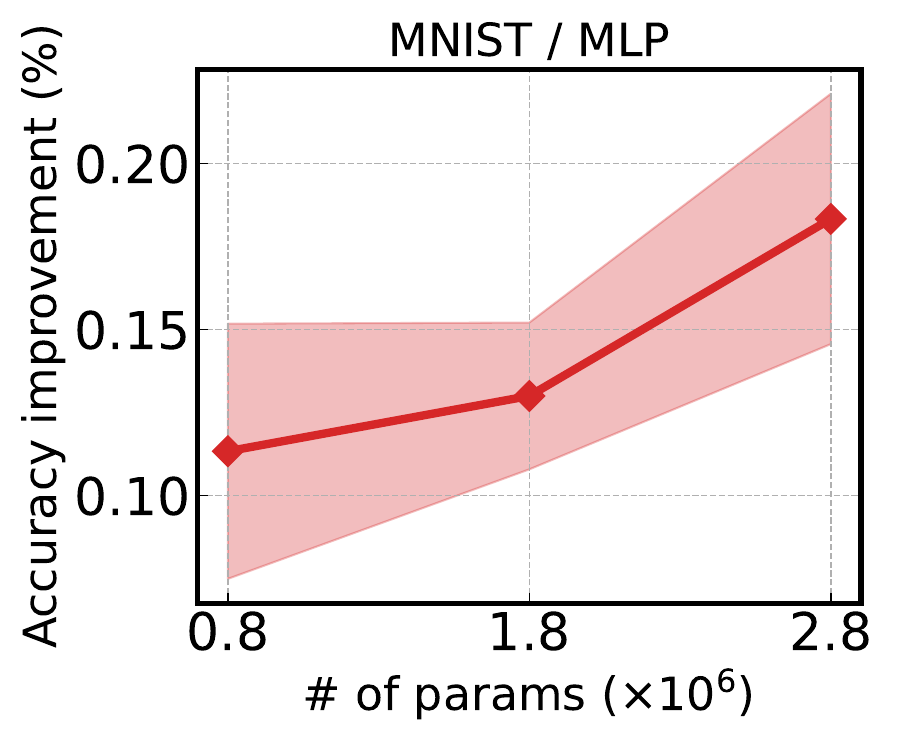} %
  \includegraphics[width=0.49\linewidth]{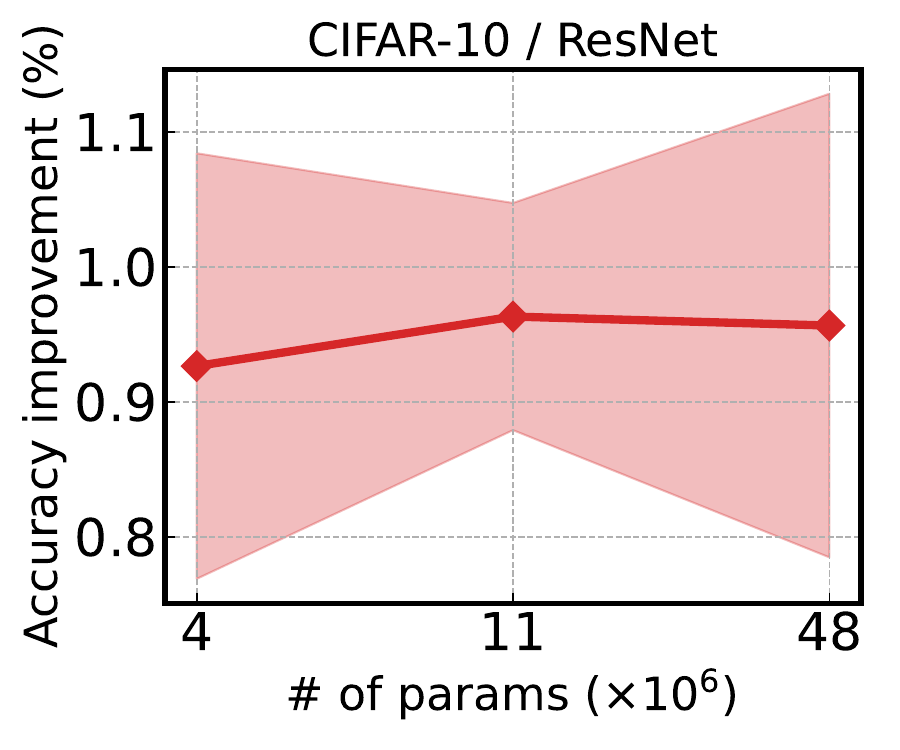} 
  \caption{
  Improvement in validation metrics by SAM over different model depths.
  Deeper models tend to yield higher validation accuracy improvements.
  Here we change the number of width-1000 hidden layers in MLP and resblock in each stage of ResNet-18 for MNIST and CIFAR-10, respectively. Benefits of SAM also improves with overparameterization in terms of depth, although the increase is not significant for ResNet.
  }
  \label{fig:depth}
\end{wrapfigure}

We experiment with changing the number of layers for MNIST/MLP and Cifar-10/ResNet-18.
Precisely, we change the number of width-1000 hidden layers in MLP and resblock in each stage of ResNet-18 for MNIST and CIFAR-10 respectively. 
The results are provided in \cref{fig:depth}. 
We find that SAM also improves with overparameterization for MLPs, while the increase is not significant for ResNets.
We suspect that this may result from the complex interplay of various intricate factors and decisions involved in increasing depth in modern architectures such as ResNets (\eg, deciding whether to increase the number of resblocks, layers within the resblock, width stages, or some combination of them), each affecting the training dynamics in distinct ways.
Further study into these factors would be an interesting direction to understand these influences more comprehensively.

\subsection{SAM vs. weight decay} \label{app:wd}

\begin{figure}[!th]
    \centering
    \includegraphics[width=0.19\linewidth]{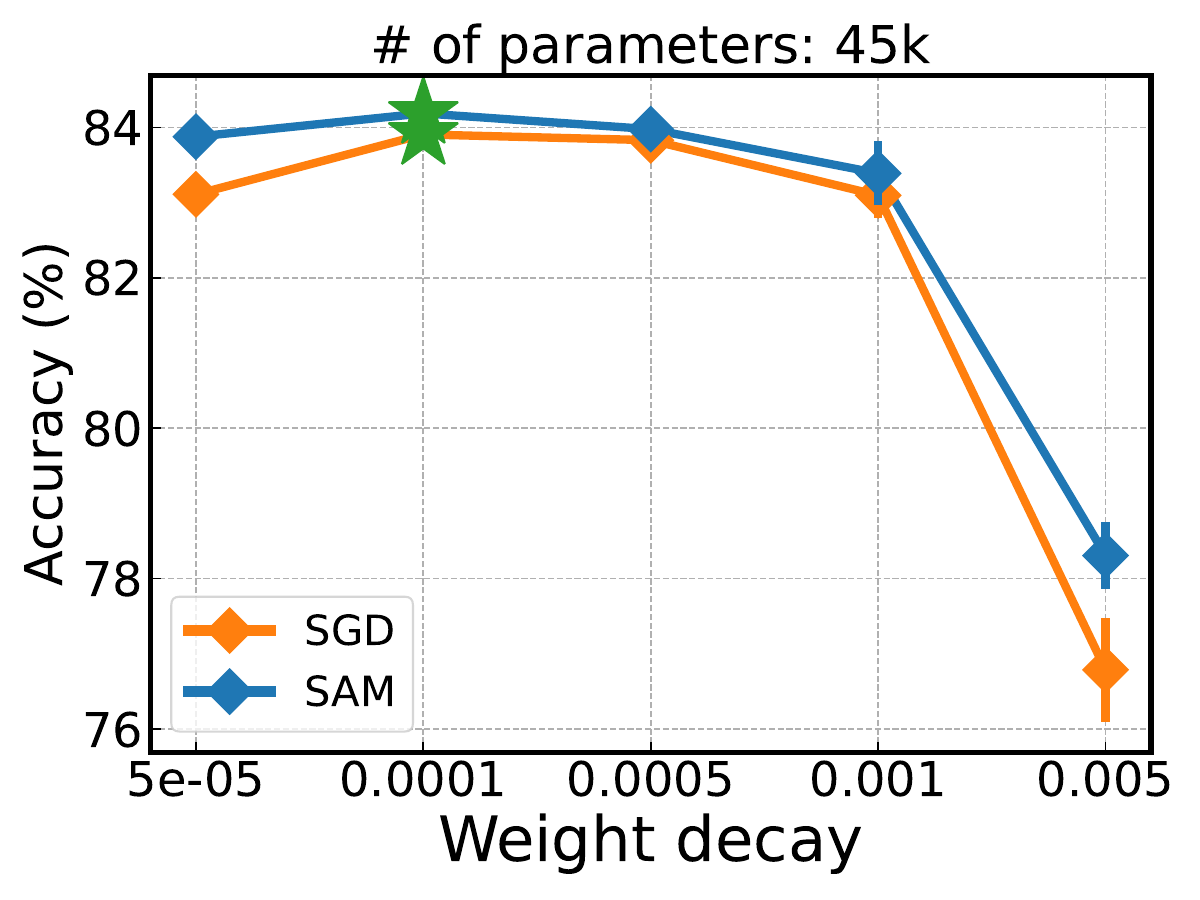}
    \includegraphics[width=0.19\linewidth]{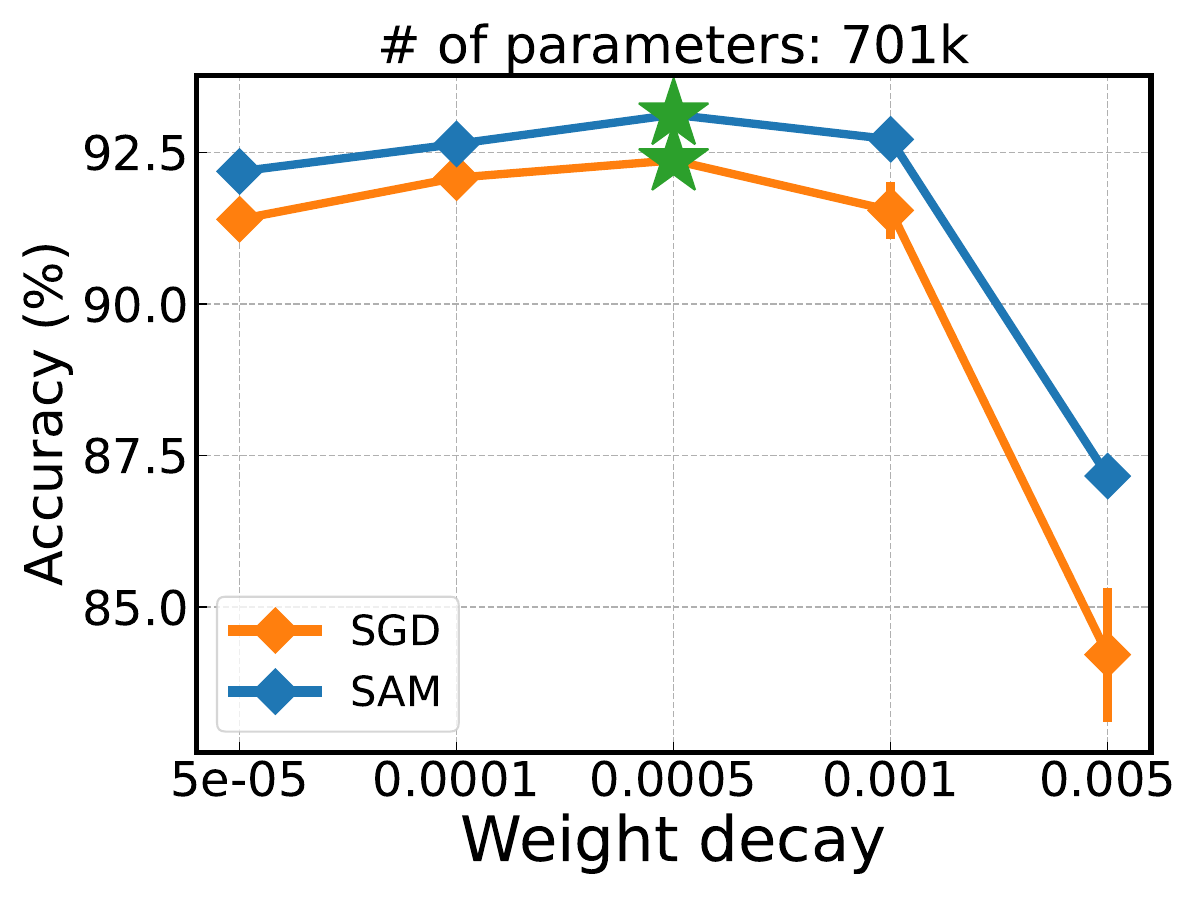}
    \includegraphics[width=0.19\linewidth]{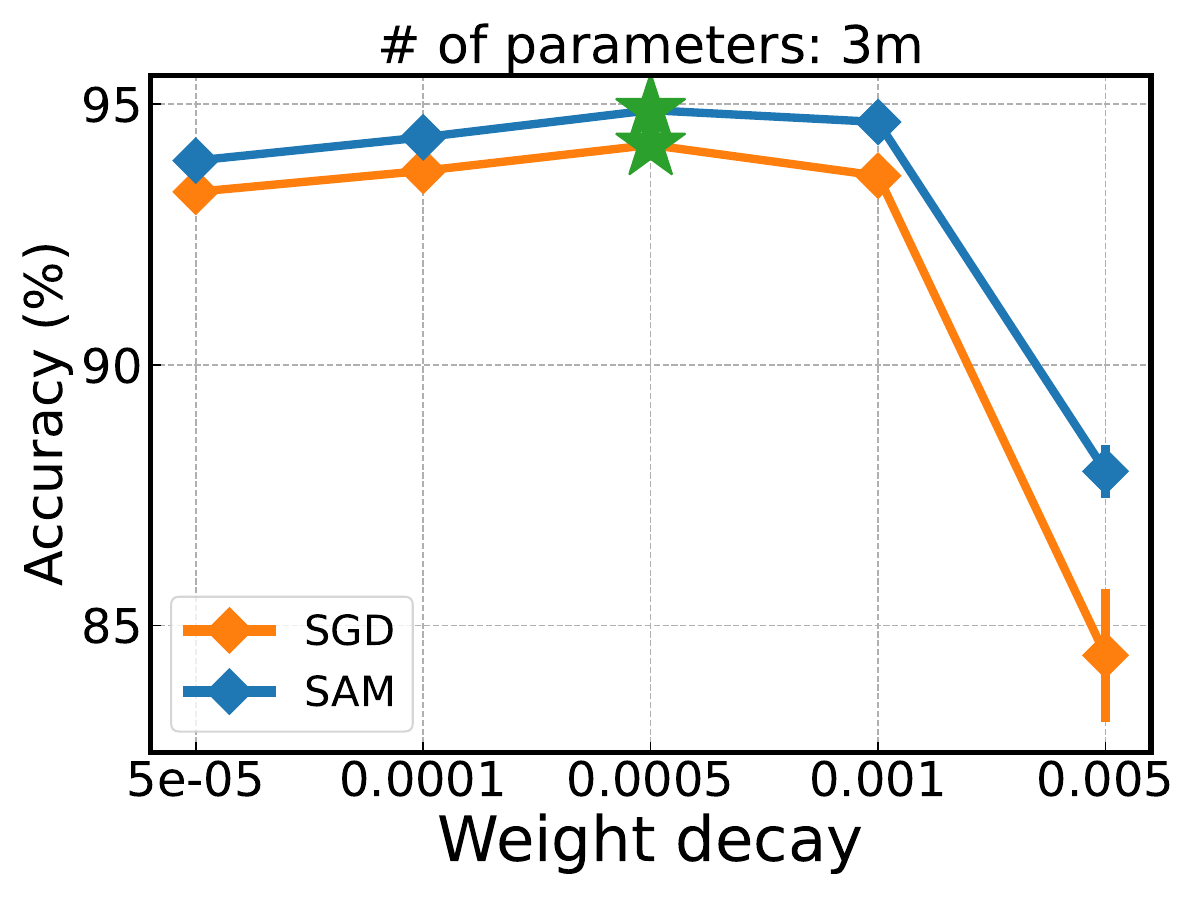}
    \includegraphics[width=0.19\linewidth]{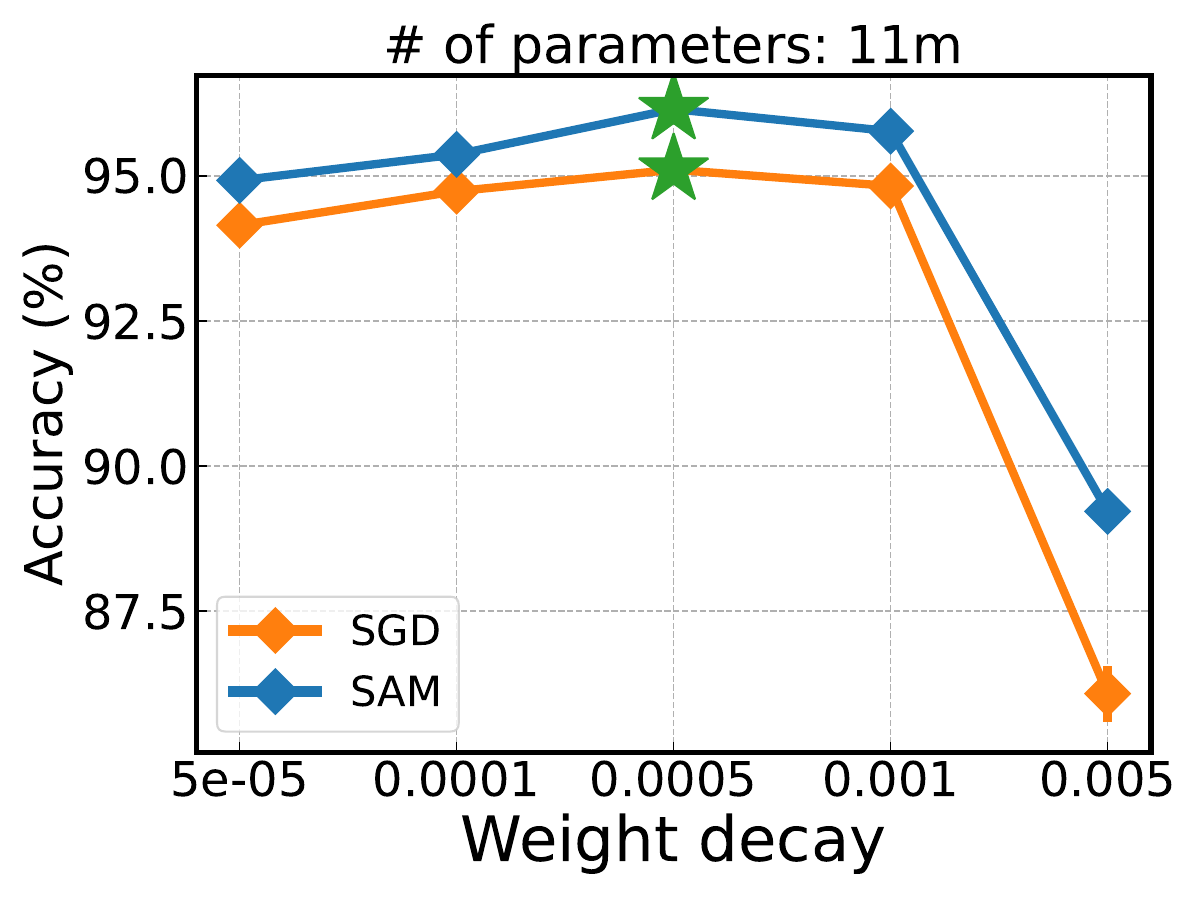}
    \hspace{0.15em}
    \includegraphics[width=0.19\linewidth]{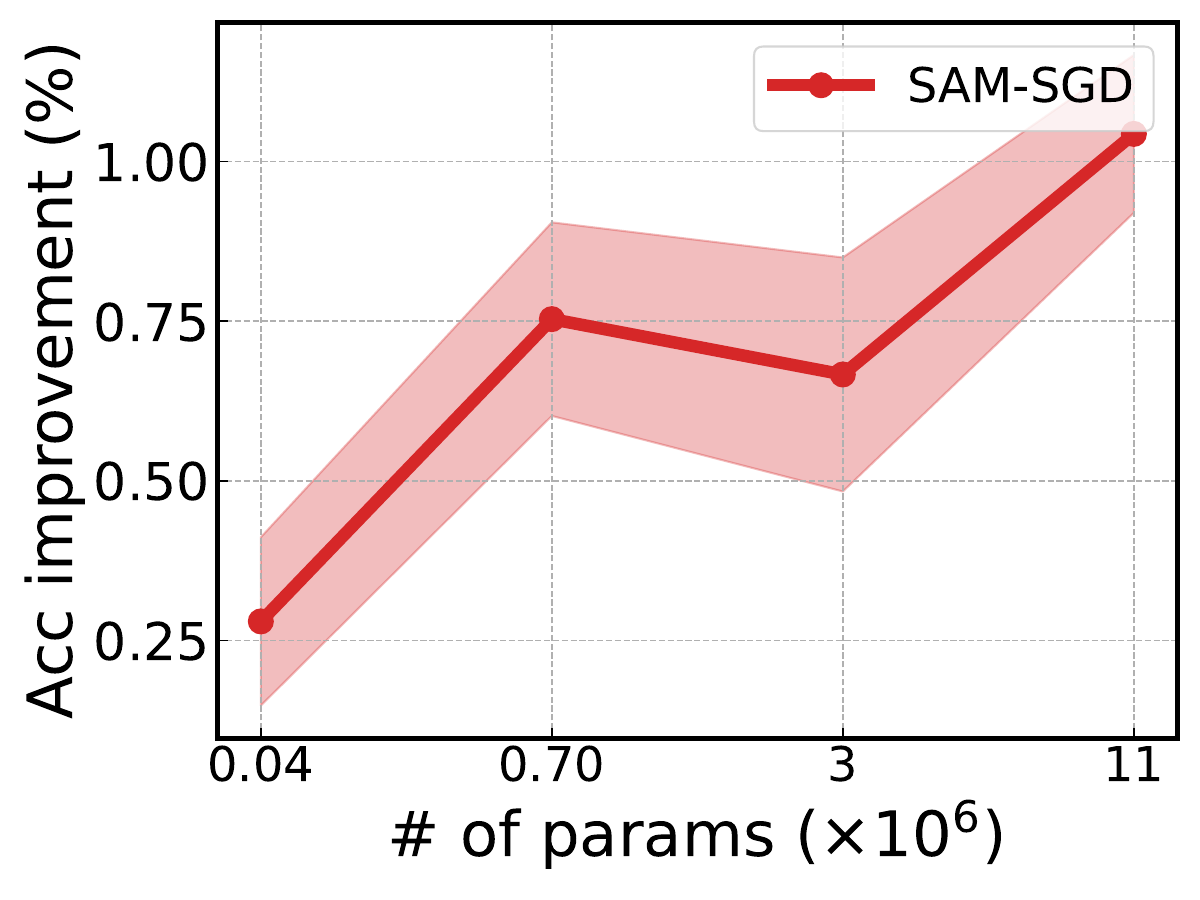}
  \caption{
    Effect of weight decay on validation accuracy of ResNet-18 trained on CIFAR-10 with SAM and SGD over various model scales and the improvement in validation metrics by SAM when considering weight decay. 
    Even after being given much larger values of weight decay, SGD isn't able to outperform SAM on any model size.
  }
  \label{fig:sam_vs_wd}
\end{figure}

We conduct experiments on Cifar-10/ResNet-18 for four different model sizes and five values of weight decay. 
The results are provided in \cref{fig:sam_vs_wd}. 
We find that SGD with stronger weight decay does not compete to replace SAM for overparameterized models; for overparameterized models, using larger weight
decay rather degrades the performance for SGD. This potentially indicates that a generic regularization strategy may not suffice for overparameterized models relatively compared to SAM.

\subsection{Results on SAM under Linearized Regime}

\label{app:exp-linear}

Recent studies suggest that highly overparameterized models can behave like linearized networks \citep{jacot2018neural}, while such implicit linearization phenomenon can coincide independently of overparameterization \citep{chizat2019lazy}.
One might wonder if the increased effectiveness of SAM directly comes from the overparameterization itself or is rather due to linearization.
\begin{wrapfigure}{r}{0.45\linewidth}
    \centering
    \begin{subfigure}{\linewidth}
        \centering
        \includegraphics[width=0.49\linewidth]{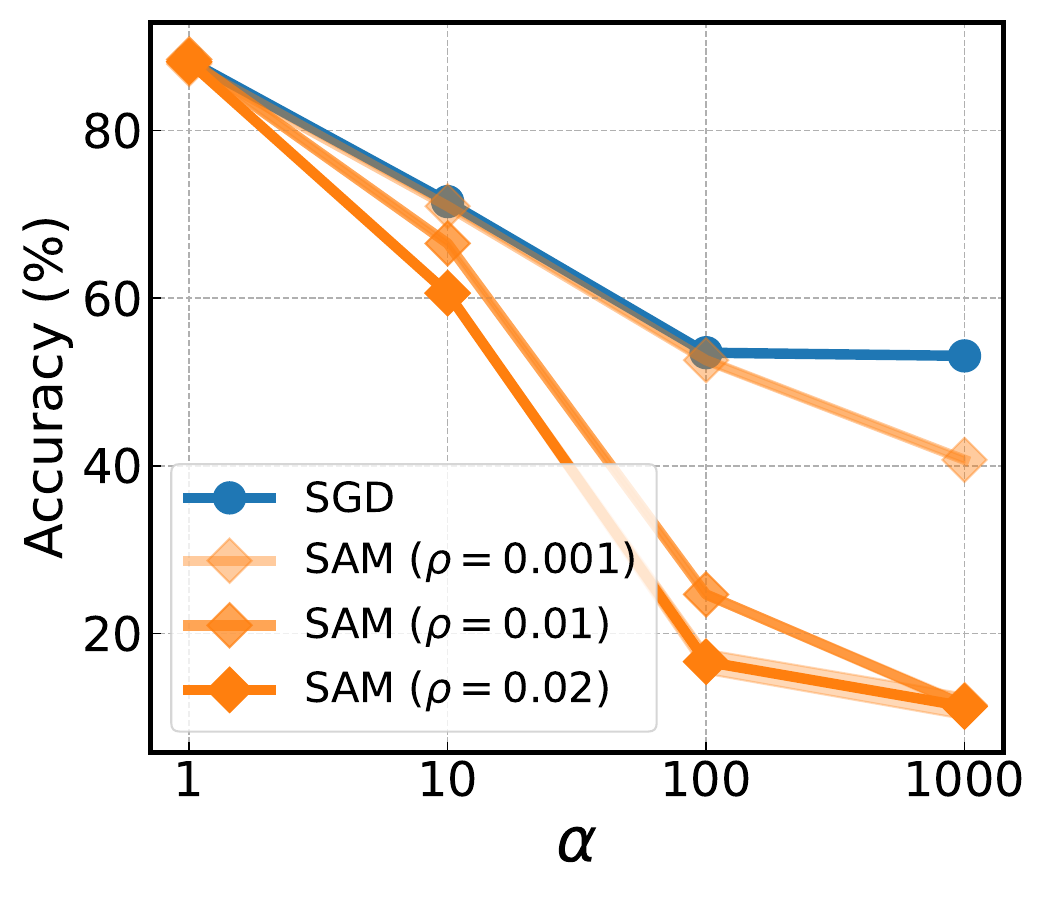}
        \includegraphics[width=0.49\linewidth]{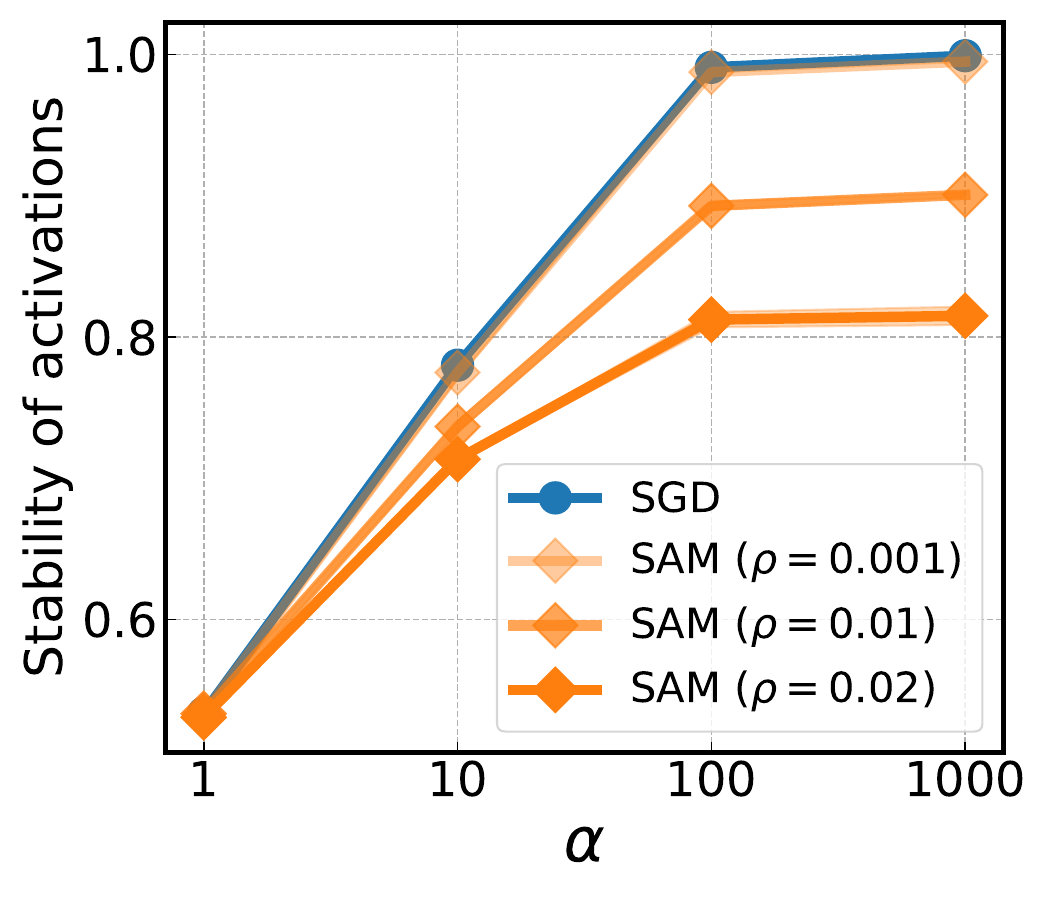}
        \vspace{-2em}
    \end{subfigure}
  \caption{
    Effect of linearization on SAM.
    Here, $\alpha$ controls the degree of linearization.
    High linearization does not yield improvement, and in fact, SAM ($\rho=0.001$) underperforms SGD in the linearized regime (left), although both achieve effective linearization at $\alpha = 1000$ with stability close to $1$ (right).
  }
  \label{fig:result_overparam_linearization}
  \vspace{-1em}
\end{wrapfigure}
To verify, we reproduce experiments in \citet{chizat2019lazy} and see how SAM performs in the linearized regimes while fixing the number of parameters.
Specifically, we train VGG-11 \citep{simonyan2014very} on the Cifar-10 with the $\alpha$-scaled squared loss $L(x, y) = \| f(x) - y/\alpha \|^2$ and use the centered model whose initial output is set to $0$.
Here, a large value of $\alpha$ leads to a higher degree of linearization of the models.

The results are reported in \cref{fig:result_overparam_linearization}.
We observe that SAM underperforms SGD in the linearized regimes;
while SAM ($\rho=0.001$) and SGD both achieve effective linearization at $\alpha=1000$, SAM underperforms SGD by more than $10\%$.
This indicates that linearization is not the main factor, and overparameterization itself is what leads to the improvement of SAM in previous experiments.

\subsection{SGD with twice the epochs}
\label{app:sgd-twice}

\begin{wrapfigure}{r}{0.45\linewidth}
  \vspace{-3em}
    \centering
    \begin{subfigure}{\linewidth}
        \centering
        \includegraphics[width=0.49\linewidth]{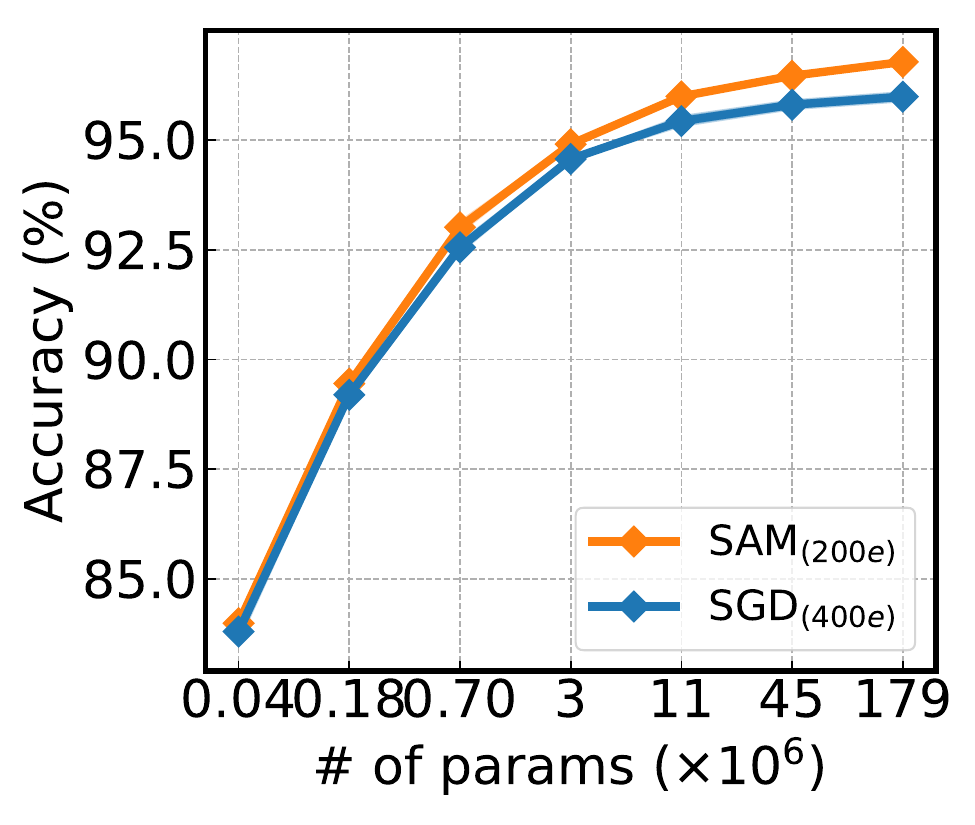}
        \includegraphics[width=0.49\linewidth]{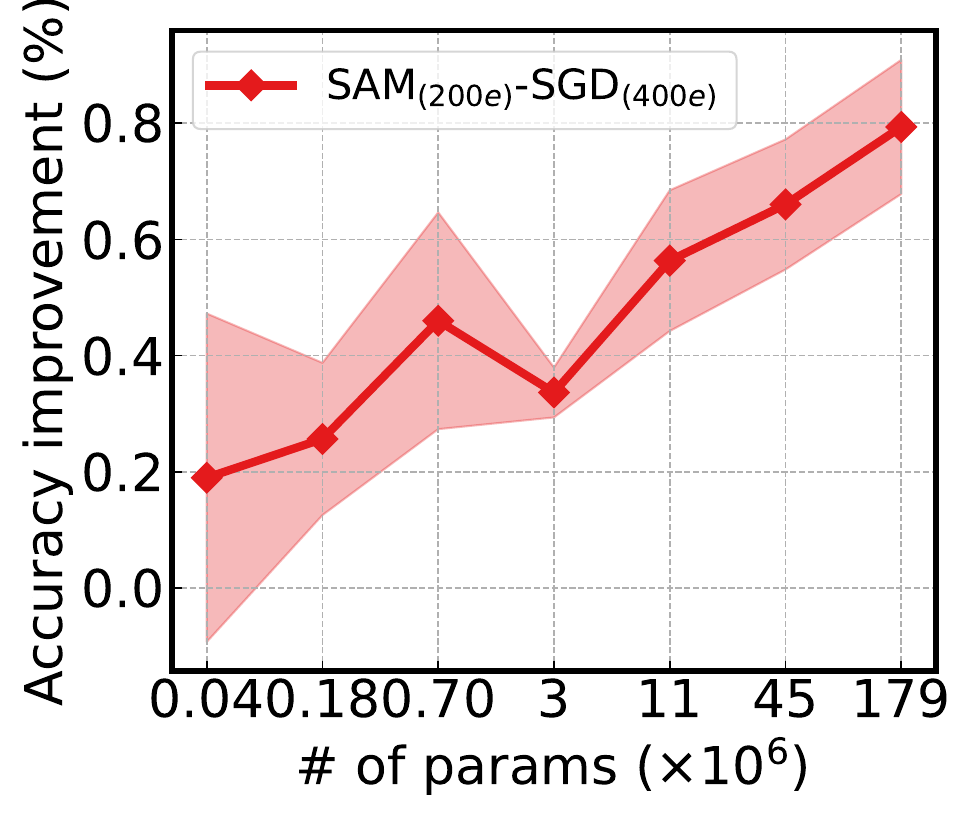}
        \vspace{-1.5em}
    \end{subfigure}
  \caption{
    SGD with twice the epochs. SAM still increasingly outperforms SGD with further overparameterization.
  }
  \label{fig:sgd_twice}
  \vspace{-1em}
\end{wrapfigure}

We compare SGD, trained for twice the number of epochs as SAM, to account for the additional gradient computation of SAM. 
Precisely, we report the validation accuracy of SGD trained for 400 epochs and compare it to SAM trained for the original 200 epochs on ResNet-18/CIFAR-10 in \cref{fig:sgd_twice}.
Similar to our observation in \cref{sec:experiments-main}, we find that the generalization benefit of SAM improves with overparameterization.
This outcome is expected: granting SGD additional training iterations can lead to overfitting and degrade generalization, potentially placing it at a disadvantage.
Thus, while providing a larger training budget to SGD may appear fair, it is not necessarily a more equitable comparison, and may in fact be unfair.

\section{Empirical Measurement of Lipschitz Smoothness and PL Constants} \label{app:emp-measure}

\begin{wrapfigure}{r}{0.45\linewidth}
  \centering
  \vspace{-1.3em}
  \begin{subfigure}{0.46\linewidth}
      \centering
      \includegraphics[width=\linewidth]{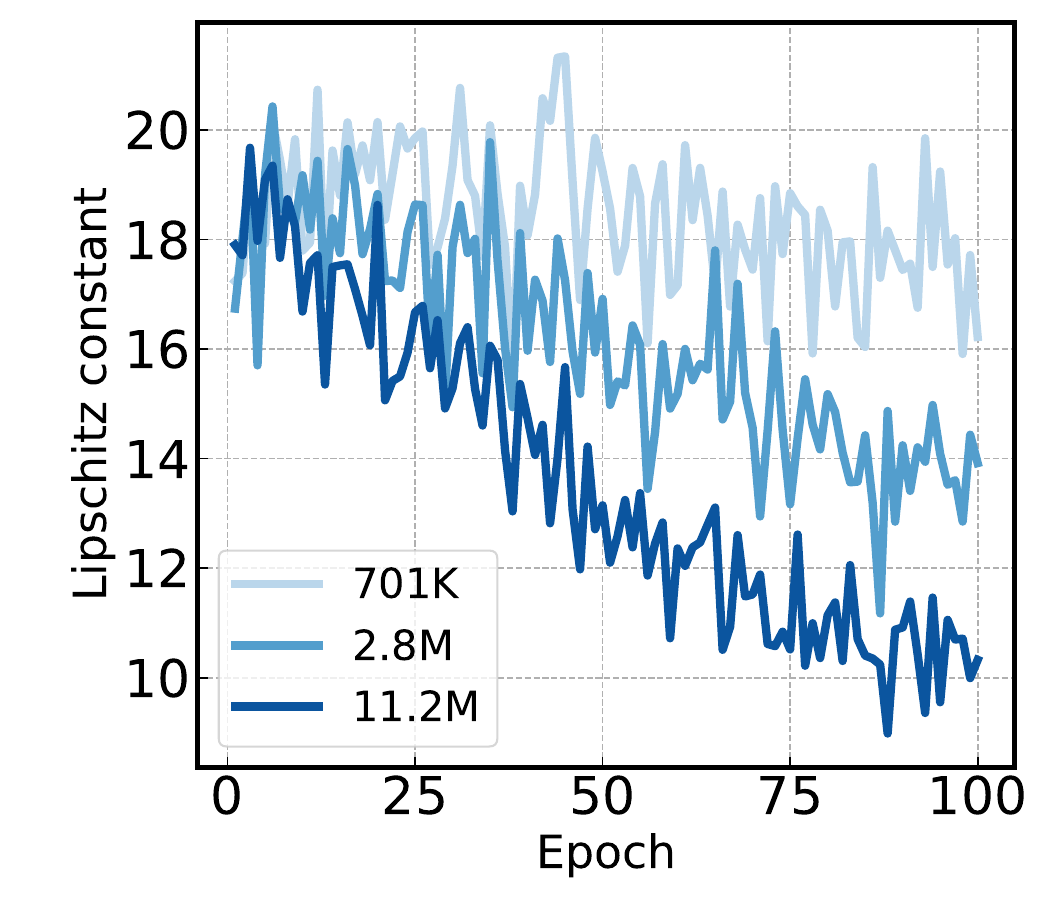}
      \caption{Lipschitz smoothness}
      \label{fig:result_empirical_lipschitz}
  \end{subfigure}
  \begin{subfigure}{0.46\linewidth}
      \centering
      \includegraphics[width=\linewidth]{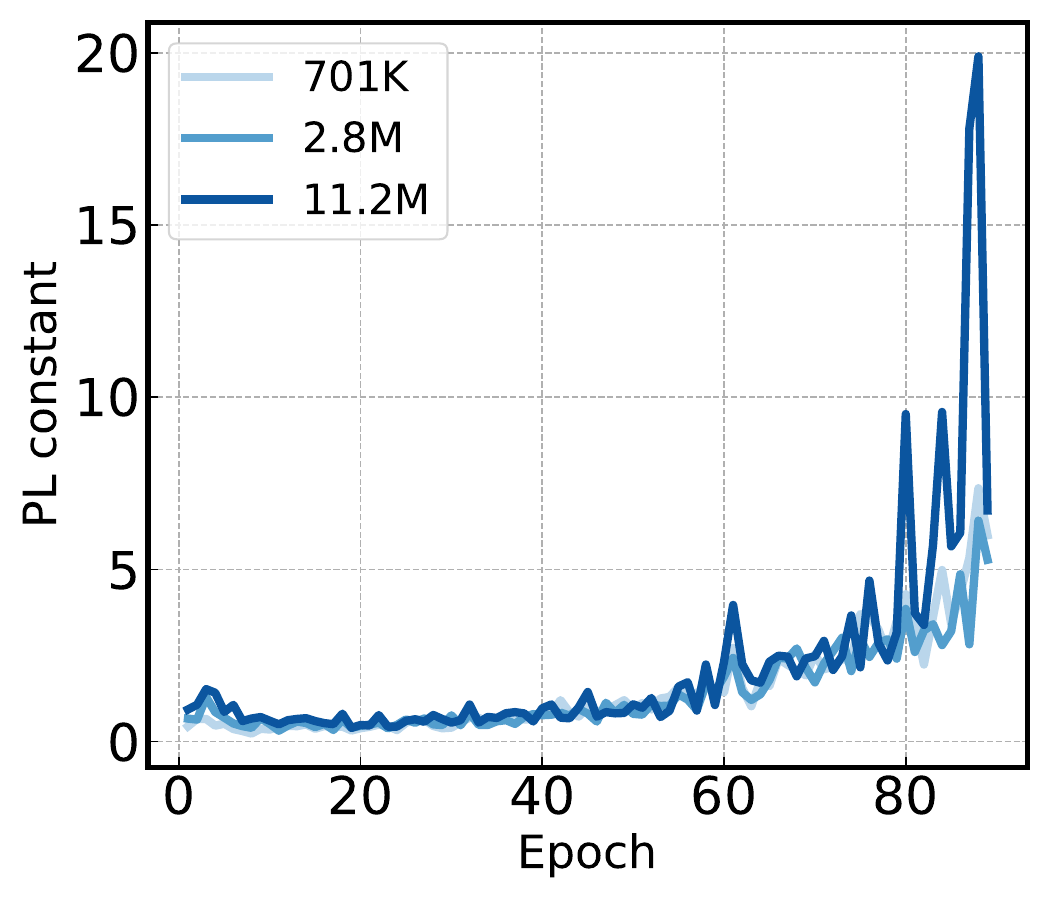}
      \caption{PL-ness}
      \label{fig:resut_empirical_pl}
  \end{subfigure}
  \vspace{-0.5em}
  \caption{
    The empirical measurement of Lipschitz smoothness (a) and PL-ness (b) for CIFAR-10 and ResNet-18.
    The Lipschitz smoothness becomes smaller and PL constant becomes larger as the model size increases.
  }
  \label{fig:result_overparam_empirical_measure}
\end{wrapfigure}

We measure the empirical Lipschitz smoothness constant and PL constant based on \citet{zhang2019which}.
Specifically, the empirical smoothness $\hat\beta(x_k)$ and empirical PL constant $\hat\alpha(x_k)$ at iteration $k$ is computed as follows:
\begin{align}
    \hat{\beta}(x_k) &= \max_{\gamma \in \{\delta, 2\delta, \dots, 1\}} \frac{\|\nabla f(x_k+\gamma \mathbf{d})-\nabla f(x_k)\|_2}{\|\gamma \mathbf{d}\|_2},\\
    \hat{\alpha}(x_k) &= \min_{\gamma \in \{\delta, 2\delta, \dots, 1\}} \frac{\|\nabla f(x_k+\gamma \mathbf{d})\|_2^2}{f(x_k+\gamma \mathbf{d})-f(x^\star)}.
\end{align}
where $\mathbf{d}=x_{k+1}-x_{k}$ and $\delta \in (0, 1)$ where we choose $\delta = 0.1$.
We measure these quantities at the end of every epoch throughout training.
The results are shown in \cref{fig:result_overparam_empirical_measure}.

\input{text/SAM_stability_unbatched}
\input{text/Linear_convergence_Stoch_SAM}
\input{text/Generalization_bound}

%% file: text/SAM_stability_unbatched.tex
\section{Proof of Theorem \ref{thm:sam-stability}}

\label{app:sam-stability}

Here, we provide the detailed proof of \cref{thm:sam-stability}.

We first define a linearized stochastic SAM, which is derived from applying first-order Taylor approximation on a stochastic SAM update given as follows:
\begin{definition} (Linearized stochastic SAM)
    We define a linearized stochastic SAM as
    \begin{equation} \label{eq:quad-linearized-SAM}
        x_{t+1} = x_t - \eta H_{\xi_t} (x_{t+1/2} - x^\star),
    \end{equation}
    where $x_{t+1/2} = x_t + \rho H_{\xi_t} (x_t - x^\star)$ is the linearized ascent step and $H_{\xi_t}$ is the Hessian estimation at step $t$.
\end{definition}
This actually corresponds to using SAM for the quadratic approximation of $f$ near $x^\star$, and we use this fact in the experiment setup.
We assume without loss of generality that the fixed point $x^\star$ satisfies $x^\star = 0$.

Then, we are ready to present the proof of \cref{thm:sam-stability}. Our goal is to derive a bound of the form $\E\| x_{t}\|^2\leq C \|x_0\|^2$.
We first apply (\ref{eq:quad-linearized-SAM}) to $\expec{}{\|x_{t+1}\|^2 \,\lvert\, x_t}$ and continue expanding the terms as follows:

\begin{align*}
    \expec{}{\|x_{t+1}^2\| \,\lvert\, x_t}
    &= \E \|x_t - \eta H_{\xi_t} (x_t + \rho H_{\xi_t}x_t)\|^2\\
    &= x_t^\top \expec{}{\left(I - \eta H_{\xi_t} - \eta\rho H_{\xi_t}^2\right)^2 \,\Big\lvert\, x_t}x_t\\
    &= x_t^\top \expec{}{ I - 2\eta(H_{\xi_t} + \rho H_{\xi_t}^2) + \eta^2 \left( H_{\xi_t} + \rho H_{\xi_t}^2\right)^2 \,\Big\lvert\, x_t}x_t\\
    &= x_t^\top \expec{}{ I - 2\eta(H_{\xi_t} + \rho H_{\xi_t}^2) + \eta^2 \left( H_{\xi_t}^2 + 2\rho H_{\xi_t}^3 + \rho^2 H_{\xi_t}^4\right)  \,\Big\lvert\, x_t}x_t\\
    &= x_t^\top \expec{}{I - 2\eta H_{\xi_t} + \eta(\eta-2\rho)H_{\xi_t}^2 + 2\eta^2\rho H_{\xi_t}^3 + \eta^2\rho^2 H_{\xi_t}^4 \,\Big\lvert\, x_t}x_t\\
    &= x_t^\top \Big( I - 2\eta H + \eta(\eta-2\rho) \E H_{\xi_t}^2 + 2\eta^2\rho \E H_{\xi_t}^3 + \eta^2\rho^2 \E H_{\xi_t}^4 \Big) x_t \\
    &= x_t^\top \Big( I - 2\eta H  + \eta(\eta-2\rho) H^2 + 2\eta^2\rho H^3 + \eta^2\rho^2 H^4  \\
    & \hspace{3em} + \eta(\eta-2\rho) (\E H_{\xi_t}^2-H^2) + 2\eta^2\rho (\E H_{\xi_t}^3-H^3) + \eta^2\rho^2 (\E H_{\xi_t}^4 -H^4) \Big) x_t \\
    &= x_t^\top \Big( \left(I - \eta H - \eta\rho H^2\right)^2  \\
    & \hspace{4em} + \eta(\eta-2\rho) (\E H_{\xi_t}^2-H^2) + 2\eta^2\rho (\E H_{\xi_t}^3-H^3) + \eta^2\rho^2 (\E H_{\xi_t}^4 -H^4) \Big) x_t
\end{align*}

Since $x^\top A x \leq \lambda_{\text{max}}(A)\|x\|^2$ always holds for any $x$ and any matrix $A$ with the maximum eigenvalue $\lambda_{\text{max}}(A)$, applying this inequality and taking the total expectation gives the following;

\begin{align*}
    \expec{}{\|x_{t+1}\|^2 }
    &\leq \lambda_{\text{max}}\bigg( \left(I - \eta H - \eta\rho H^2\right)^2 + \eta(\eta-2\rho) (\E H_{\xi}^2-H^2)  \\
    & \hspace{5em} + 2\eta^2\rho (\E H_{\xi}^3-H^3) + \eta^2\rho^2 (\E H_{\xi}^4 -H^4) \bigg) \expec{}{\|x_t\|^2}.
\end{align*}

Recursively applying this bound gives

\begin{align*}
    \E\| x_{t}\|^2
    &\leq \lambda_{\text{max}}\bigg( \left(I - \eta H - \eta\rho H^2\right)^2 + \eta(\eta-2\rho) (\E H_{\xi}^2-H^2) \\
    & \hspace{5em} + 2\eta^2\rho (\E H_{\xi}^3-H^3) + \eta^2\rho^2 (\E H_{\xi}^4 -H^4) \bigg)^t \|x_0\|^2.
\end{align*}

Here, we can see that $x^\star$ is linearly stable if

\begin{align*}
    \lambda_{\text{max}}\bigg((I - \eta H - \eta \rho H^2)^2  + \eta(\eta-2\rho) (\E H_{\xi}^2-H^2) + 2\eta^2\rho (\E H_{\xi}^3-H^3) + \eta^2\rho^2 (\E H_{\xi}^4 -H^4) \bigg) \leq 1. \\ && \qedsymbol{}
\end{align*}

%% file: text/Linear_convergence_Stoch_SAM.tex
\section{Proof of Theorem \ref{thm:main-PL-stochSAM}} \label{app:prooflinconv}
In this section, we show that a stochastic SAM converges linearly under an overparameterized regime.
To put into perspective, this is the rate of convergence of gradient descent for a family of functions satisfying the PL-condition and smoothness assumptions \citep{karimi2016linear}.
We first make several remarks on this result below.
\begin{itemize}[leftmargin=*]
    \item Crucially, this result shows that with overparameterization, a stochastic SAM can converge as fast as the deterministic gradient method at a linear convergence rate.
    It is much faster than the well-known sublinear rate of $\mathcal{O}(1/t)$ for SAM \citep{andriushchenko2022towards}.
    \item When $\rho=0$, we recover the well-known convergence rate for SGD in interpolated regime \citep{bassily2018exponential}.
    \item This result does not require the bounded variance assumption \citep{andriushchenko2022towards} since the interpolation provides necessary guarantees.
    This suggests that overparameterization can ease the convergence of SAM.
\end{itemize}

We prove the convergence for an unnormalized mini-batch SAM given as
$$
x_{t+1} = x_t - \eta g_{t}^{_B}(x_t+\rho g_{t}^{_B}(x_t)),
$$
where $g_{t}^{_B}(x)=\frac{1}{B}\sum_{i \in I_t^B} \nabla f_i(x)$ and $I_t^B \subseteq \{1,...,n\}$ is a set of indices for data points in the mini-batch of size $B$ sampled at step $t$.
This is a more general stochastic variant of SAM where a stochastic SAM in Section \ref{sec:convergence} is a particular case of a mini-batch SAM with mini-batch size $B=1$.

We first make the following assumptions:
\begin{description}
    \item[\textbf{(A1)}] \hypertarget{assump:betasmoo}{($\beta$-smothness of $f_i$)}. \textit{There exists  $\beta > 0$ s.t. $ 
    \| \nabla f_i (x) - \nabla f_i (y) \| \leq \beta \| x-y \|$ for all $x,y  \in \mathbb R ^d$,}
    \item[\textbf{(A2)}] \hypertarget{assump:lambdasmoo}{($\lambda$-smothness of $f$)}. \textit{There exists  $\lambda\!>\!0$ such that $ 
    \| \nabla f (x) -\nabla f (y) \| \leq \lambda \| x-y \|$ for all $x,y  \in \mathbb R ^d$,} 
    \item [\textbf{(A3)}] \hypertarget{assump:PL}{($\alpha$-PLness of $f$)}. \textit{There exists $\alpha > 0$ s.t. $\| \nabla f (x) \|^2 \geq \alpha (f(x)-f(x^\star))$ for all $w,v  \in \mathbb R ^d$,}
    \item[\textbf{(A4)}] \hypertarget{assump:interpol}{(Interpolation)}. \textit{If $f(x^\star)=0$ and $~\nabla f(x^\star)=0$, then $f_i(x^\star)=0$ and $\nabla f_i(x^\star)=0$ for $i=1,\hdots,n$, where $n$ is the number of training data points.}
\end{description}

Before we prove the main theorem, we first introduce two lemma important to the proof.

\begin{lemma}\label{lemma:mini-batch-grad-align}
Suppose that Assumption \hyperlink{assump:betasmoo}{\textbf{(A1)}} holds.
Then
\begin{align}
   \langle \nabla f_i( x_{t+1/2}), \nabla f (x_t) \rangle & \geq \langle  \nabla f_i(x_t) , \nabla f (x_t)\rangle - \frac{{ \beta} \rho}{2} \| \nabla f_i (x_t) \|^2  - \frac{\beta\rho}{2} \| \nabla f (x_t)\|^2,
\end{align}
where $x_{t+1/2}=x_t + \rho \nabla f_i(x_t)$.
\end{lemma}
This lemma shows how well a stochastic SAM gradient $\nabla f_i( x_{t+1/2})$ aligns with the true gradient $\nabla f (x_t)$. 
The two gradients become less aligned as $\beta$ and $\rho$ grow bigger, \ie for sharper landscape and larger perturbation size. 

\begin{proof}
We first add and subtract $\nabla f_i^{}(x_t)$ on the left side of the inner product

\begin{align} \label{eq:step1bounding1}
   \langle \nabla f_i( x_{t+1/2}), \nabla f (x_t) \rangle =   \underbrace{\langle \nabla f_i( x_{t+1/2})- \nabla f_i(x_t) , \nabla f (x_t) \rangle}_{\tau_1} + \langle  \nabla f_i(x_t) , \nabla f (x_t)\rangle.
\end{align}

We here bound the term $\tau_1$ so that it becomes an equality when $\rho=0$. 
To achieve this, we start with the following binomial square, which is trivially lower bounded by 0.
\begin{align*}
   0 &\leq \frac{1}{2}\| \nabla f_i( x_{t+1/2})- \nabla f_i(x_t) +  \beta\rho \nabla f (x_t)\|^2
\end{align*}
We then expand the above binomial square so that the term containing $\tau_1$ appears.

\begin{align*}
   0 &\leq \frac{1}{2}\| \nabla f_i( x_{t+1/2}) - \nabla f_i(x_t) \|^2 + \underbrace{\langle \nabla f_i( x_{t+1/2})- \nabla f_i(x_t) ~,~  \beta\rho \nabla f (x_t) \rangle}_{\beta\rho  \tau_1} \, + \,\frac{1}{2}\| \beta\rho \nabla f (x_t)\|^2
\end{align*}

We subtract the term $\beta\rho  \tau_1$ on both sides of the inequality which gives
\begin{align*}
   -  \langle \nabla f_i( x_{t+1/2})- \nabla f_i(x_t) ~,~  \beta\rho \nabla f (x_t) \rangle
    &\leq  \frac{1}{2}\| \nabla f_i( x_{t+1/2})- \nabla f_i(x_t) \|^2  +\frac{\beta^2\rho^2 }{2}\| \nabla f (x_t)\|^2. 
\end{align*}

Then we upper bound the right-hand side using the Assumption \hyperlink{assump:betasmoo}{\textbf{(A1)}}: 

\begin{align*}
   -   \langle \nabla f_i( x_{t+1/2})- \nabla f_i(x_t) ~,~  \beta\rho \nabla f (x_t) \rangle
    &\leq  \frac{{ \beta}^2}{2} \|  x_{t+1/2}- x \|^2 +\frac{\beta^2\rho^2 }{2}\| \nabla f (x_t)\|^2 \\
    &=  \frac{{ \beta}^2 \rho^2}{2} \| \nabla f_i (x_t) \|^2  +\frac{\beta^2\rho^2 }{2} \| \nabla f (x_t)\|^2. 
\end{align*}

We divide both sides with $\beta\rho$, obtaining: 

   \begin{align*}
   -   \langle \nabla f_i( x_{t+1/2})- \nabla f_i(x_t) , \nabla f (x_t) \rangle 
    &\leq  \frac{{ \beta} \rho}{2} \| \nabla f_i (x_t) \|^2  + \frac{\beta\rho}{2} \| \nabla f (x_t)\|^2.
   \end{align*}
   
Applying this to (\ref{eq:step1bounding1}) gives the bound in the lemma statement.
\end{proof}

\begin{lemma}\label{lemma:mini-batch-grad-norm-bound}
    Suppose that Assumption \hyperlink{assump:betasmoo}{\textbf{(A1)}} holds.
    Then
    \begin{equation}
    \left\|  \nabla f_i(x_{t+{1/2}})\right\|^2 \leq (\beta\rho+1)^2\| \nabla f_i ( x_t) \|^2,
    \end{equation}
    where $x_{t+1/2}=x_t + \rho \nabla f_i(x_t)$.
\end{lemma}

This second lemma shows that the norm of a stochastic SAM gradient is bounded by the norm of the stochastic gradient.
Similar to the Lemma \ref{lemma:mini-batch-grad-align}, as $\beta$ and $\rho$ grow bigger the norm for a stochastic SAM gradient can become larger than the norm of the true gradient.

\begin{proof}
    We use the following binomial squares:
    \begin{equation*}\label{eq:binomsq}
        \begin{split}
            & \| \nabla f_i ( x_{t+1/2}) \|^2  \\
            &= \| \nabla f_i ( x_{t+1/2}) - \nabla f_i(x_t) \|^2  +2 \langle \nabla f_i(  x_{t+1/2}  )-\nabla f_i ( x_t) , \nabla f_i (x_t) \rangle + \| \nabla f_i ( x_t) \|^2.
        \end{split}
    \end{equation*}
    
    We bound the right-hand side using Cauchy-Schwarz inequality and Assumption \hyperlink{assump:betasmoo}{\textbf{(A1)}}, which gives
    \begin{align*}\label{eq:mb-norm_bound}
    & \left\|  \nabla f_i(x_{t+{1/2}})\right\|^2 \\
    &= \| \nabla f_i ( x_{t+1/2}) - \nabla f_i(x_t) \|^2 + 2 \langle \nabla f_i(  x_{t+1/2}  )-\nabla f_i ( x_t), \nabla f_i (x_t) \rangle + \| \nabla f_i ( x_t) \|^2\\
    &\underset{\text{C.S.}}{\leq}
    \| \nabla f_i ( x_{t+1/2}) - \nabla f_i(x_t) \|^2 + 2\| \nabla f_i(  x_{t+1/2}  )-\nabla f_i ( x_t) \|  \|\nabla f_i (x_t)\| + \| \nabla f_i ( x_t) \|^2\\
    &\underset{\hyperlink{assump:betasmoo}{\textbf{(A1)}}}{\leq}
    ~ \beta^2\|  x_{t+1/2}-x_t \|^2 +  2 \beta\|  x_{t+1/2}-x_t \|  \|\nabla f_i (x_t)\| + \| \nabla f_i ( x_t) \|^2 ~\\
    &~=~ 
    ~ \beta^2\rho^2\| \nabla f_i(x_t) \|^2 +  2 \beta\rho\| \nabla f_i(x_t) \|^2 + \| \nabla f_i ( x_t) \|^2 ~\\
    &~=~
    (\beta\rho+1)^2\| \nabla f_i ( x_t) \|^2
    \end{align*}
\end{proof}

These two lemmas essentially show how similar a stochastic SAM gradient is to the stochastic gradient, where the two become more similar as $\beta$ and $\rho$ decrease, which aligns well with our intuition.
Using Lemma \ref{lemma:mini-batch-grad-align} and \ref{lemma:mini-batch-grad-norm-bound}, we provide the convergence result in the following theorem.

\begin{theorem}\label{thm:PL-mini-batch-SAM}
    Suppose that Assumptions \hyperlink{assump:betasmoo}{\textbf{(A1-4)}} holds.
    For any mini-batch size $B \in \mathbb{N}$ and  $\rho\leq \frac{1}{(\beta/\alpha + 1/2)\beta}$, unnormalized mini-batch SAM with constant step size $\eta^\star_B \defeq \frac{1-(\kappa_B +1/2)\beta\rho}{2\lambda\kappa_B(\beta\rho+1)^2}$ gives the following guarantee at step $t$:
    \begin{equation*}\label{eq:theorem-PL-mini-batch-SAM}
    \ex{x_t}{f(x_t)}\leq
    \left(1 -  \frac{\alpha\,\eta^\star_B}{2}  \Big(1-\Big(\kappa_B +\frac{1}{2}\Big)\beta\rho\Big)\right)^t\,f(x_0),
    \end{equation*}
    where $\kappa_B = \frac{1}{B}\left(\frac{B-1}{2}+\frac{\beta}{\alpha}\right)$. 
\end{theorem}

This theorem states that mini-batch SAM converges at a linear rate under overparameterization.

\begin{proof}
Proof can be outlined in 3 steps.\\

\begin{framed}
\vspace{-4mm}
    \begin{description}
        \item[step 1.] Handle terms containing mini-batch SAM gradient $g_{t}^{_B}(x_t + \rho g_{t}^{_B}(x_t))$ using bounds from \hyperlink{assump:betasmoo}{\textbf{(A1)}}.
        \item[step 2.] Take conditional expectation $\expec{}{\,\cdot \, | x_t}$ and substitute expectation of function of mini-batch gradient $g_t^{_B}$ with terms containing $\| \nabla f(x_t)\|$ and $\expec{}{\| \nabla f_i(x_t)\|^2 ~\Big\lvert~ x_t}$.
        \item[step 3.] Bound the two terms from \textbf{step 2}, one using Assumptions \hyperlink{assump:betasmoo}{\textbf{(A1)}} and \hyperlink{assump:interpol}{\textbf{(A4)}} and the other using Assumption \hyperlink{assump:PL}{\textbf{(A3)}} and \hyperlink{assump:interpol}{\textbf{(A4)}} which results in all the terms to contain $f(x_t)$. 
        Then finally we take total expectations to derive the final descent bound.
    \end{description}
\vspace{-3.5mm}
\end{framed}

\enspace

We start from the quadratic upper bound derived from Assumption \hyperlink{assump:lambdasmoo}{\textbf{(A2)}};
$$f(x_{t+1})\leq f(x_{t})+\langle \nabla f(x_t),~ x_{t+1}-x_t\rangle+\frac{\lambda}{2}\|x_{t+1}-x_{t}\|^2.$$

Applying mini-batch SAM update, we then have

$$f(x_{t})-f(x_{t+1})\geq \eta \left\langle~ \nabla f(x_t)~,~ g_{t}^{_B}(x_{t+{1/2}})
~\right\rangle -\frac{\eta^2\lambda}{2}\left\|  g_{t}^{_B}(x_{t+{1/2}})\right\|^2,$$ 
where $x_{t+1/2}=x_t+\rho g_{t}^{_B}(x_t)$.

\fbox{\textbf{step 1.}} \ \ We can see that there are two terms that contain a mini-batch SAM gradient $g_{t}^{_B}(x_{t+1/2})$.
We see that each can be bounded directly using Lemma \ref{lemma:mini-batch-grad-align} and \ref{lemma:mini-batch-grad-norm-bound}, which gives

\begin{align*}
    &f(x_{t})-f(x_{t+1}) \\
    &\geq
    \eta \left(\langle  g_{t}^{_B}(x_t) , \nabla f (x_t)\rangle - \frac{{ \beta} \rho}{2} \| g_{t}^{_B} (x_t) \|^2  - \frac{\beta\rho}{2} \| \nabla f (x_t)\|^2\right) \\
    &~~~~~~~~~
    -\frac{\eta^2\lambda}{2}~(\beta\rho+1)^2~\| g_{t}^{_B} ( x_t) \|^2\\
    &~~~=~~
    \eta\langle  g_{t}^{_B}(x_t) , \nabla f (x_t)\rangle - \frac{\eta \beta\rho}{2} \| \nabla f (x_t)\|^2 
    - \frac{\eta}{2} \left(\eta\lambda~(\beta\rho+1)^2 + \beta\rho \right)~\| g_{t}^{_B} ( x_t) \|^2.
\end{align*}

\fbox{\textbf{step 2.}} \ \ Now we apply $\expec{}{\, \cdot \, | x_t}$ to all the terms.

\begin{align*}
    &\expec{}{f(x_{t})-f(x_{t+1}) ~\big\lvert~ x_t}\\
    &=
    f(x_{t})-\expec{}{f(x_{t+1}) ~\big\lvert~ x_t}\\
    &\geq 
    \eta\expec{}{\langle  g_{t}^{_B}(x_t) , \nabla f (x_t)\rangle ~\Big\lvert~ x_t} - \frac{\eta \beta\rho}{2} \expec{}{\| \nabla f (x_t)\|^2  ~\big\lvert~ x_t} \\
    &\ \ \ \ \ \ 
    - \frac{\eta}{2} \left(\eta\lambda~(\beta\rho+1)^2 + \beta\rho \right)~ \expec{}{\| g_{t}^{_B} ( x_t) \|^2 ~\Big\lvert~ x_t}\\
    &=
    \eta\left( 1 - \frac{\beta\rho}{2}\right) \| \nabla f (x_t)\|^2 
    - \frac{\eta}{2} \left(\eta\lambda~(\beta\rho+1)^2 + \beta\rho \right)~ \expec{}{\| g_{t}^{_B} ( x_t) \|^2 ~\Big\lvert~ x_t}.
\end{align*}

Here we expand the term $\expec{}{\| g_{t}^{_B} ( x_t) \|^2 ~\Big\lvert~ x_t}$ by expanding the mini-batched function into individual function estimators as follows.

\begin{equation}\label{eq:mb-norm}
    \begin{split}
        &\expec{g^{_B}_t}{\norm{g_{t}^{_B}(x_t)}^2 ~\Big\lvert~ x_t}\\
        &= 
        \expec{I_t^B}{\left\langle \frac{1}{B}
        \sum_{i\in I_t^B} \nabla f_i(x_t)~,~ \frac{1}{B}
        \sum_{j \in I_t^B} \nabla f_j(x_t) \right\rangle ~\Bigg\lvert~ x_t}\\
        &=
        \frac{1}{B^2} \left\lbrace
        \sum_{i\in I_t^B} \expec{f_i}{\norm{\nabla f_i (x_t)}^2 ~\Big\lvert~ x_t}
        + \sum_{i\in I_t^B} \sum_{\substack{j\in I_t^B\\\text{(}j\neq i\text{)}}}
        \expec{f_i, f_j}{\iprod{\nabla f_i (x_t)} {\nabla f_j (x_t)} ~\Big\lvert~ x_t}
        \right\rbrace\\
        &=
        \frac{1}{B} \expec{}{\norm{\nabla f_i(x_t)}^2 ~\Big\lvert~ x_t} + \frac{B-1}{B} \norm{\nabla f(x_t)}^2.
    \end{split}
\end{equation}

Using (\ref{eq:mb-norm}), we get
\begin{align}
    &f(x_{t})-\expec{}{f(x_{t+1}) ~\big\lvert~ x_t}\\
    &\geq
    \eta\left( 1 - \frac{\beta\rho}{2}\right) \| \nabla f (x_t)\|^2 \nonumber\\
    &~~~
    - \frac{\eta}{2} \left(\eta\lambda~(\beta\rho+1)^2 + \beta\rho \right)  \left( \frac{1}{B}\expec{}{\norm{\nabla f_i(x_t)}^2 ~\Big\lvert~ x_t} + \frac{B-1}{B} \norm{\nabla f(x_t)}^2\right) \nonumber\\
    &=
    \eta \left( \left(1-\frac{\beta\rho}{2}\right) - \frac{B-1}{2B}\left(\eta\lambda(\beta\rho+1)^2+\beta\rho\right) \right) \| \nabla f (x_t) \|^2 \nonumber \\
    &~~~
    -\frac{\eta}{2B} \left(\eta\lambda(\beta\rho+1)^2+\beta\rho\right) \expec{}{\norm{\nabla f_i(x_t)}^2 ~\Big\lvert~ x_t} \label{eq:step2-result}.
\end{align}

\fbox{\textbf{step 3.}} \ \ In this step, we bound the two terms and take the total expectation to derive the final descent bound.

We first derive a bound for $\expec{}{\norm{\nabla f_i(x_t)}^2 ~\Big\lvert~ x_t}$.
We start from the following bound derived from Assumption \hyperlink{assump:betasmoo}{\textbf{(A1)}}:
\begin{equation*}
\|\nabla f_i(x_t)-\nabla f_i(x^\star)\|^2 \leq 2\beta (f_i(x_t) - f_i(x^\star)).
\end{equation*}

By Assumption \hyperlink{assump:interpol}{\textbf{(A4)}}, this reduces to
\begin{equation*}
\|\nabla f_i(x_t)\|^2 \leq 2\beta f_i(x_t).
\end{equation*}

Applying this to (\ref{eq:step2-result}) gives
\begin{align}
f(x_{t})-\expec{}{f(x_{t+1}) ~\big\lvert~ x_t}
&\geq
\eta \left( \left(1-\frac{\beta\rho}{2}\right) - \frac{B-1}{2B}\left(\eta\lambda(\beta\rho+1)^2+\beta\rho\right) \right) \| \nabla f (x_t) \|^2 \nonumber\\
&~~~~~
-\frac{\eta\beta}{B} \left(\eta\lambda(\beta\rho+1)^2+\beta\rho\right) \expec{}{f_i(x_t) \lvert x_t} \nonumber\\
&=
\eta \underbrace{\left( \left(1-\frac{\beta\rho}{2}\right) - \frac{B-1}{2B}\left(\eta\lambda(\beta\rho+1)^2+\beta\rho\right) \right)}_{\tau_2} \| \nabla f (x_t) \|^2 \nonumber\\
&~~~~~
-\frac{\eta\beta}{B} \left(\eta\lambda(\beta\rho+1)^2+\beta\rho\right)  f(x_t).
\label{eq:step3-1}
\end{align}

Next, to bound $\| \nabla f(x_t) \|^2$, we use the following bound derived from applying $f(x^*)=0$ from \hyperlink{assump:interpol}{\textbf{(A4)}} to \hyperlink{assump:PL}{\textbf{(A3)}}:
\begin{equation} \label{eq:final_pl}
    \| \nabla f (x) \|^2 \geq \alpha f(x).
\end{equation}
Assuming $\tau_2\geq0$ which we provide a sufficient condition at the end of the proof, we apply (\ref{eq:final_pl}) to (\ref{eq:step3-1}) which gives
\begin{align*}
& f(x_{t})-\expec{}{f(x_{t+1}) ~\big\lvert~ x_t}\\
&\geq
\eta\alpha \left( \left(1-\frac{\beta\rho}{2}\right) - \frac{B-1}{2B}\left(\eta\lambda(\beta\rho+1)^2+\beta\rho\right) \right)  f (x_t) -\frac{\eta\beta}{B} \left(\eta\lambda(\beta\rho+1)^2+\beta\rho\right)  f(x_t) \\
&=
\eta
\left( 
    \alpha-\alpha\bigg(\underbrace{\frac{1}{B}\Big(\frac{B-1}{2}+\frac{\beta}{\alpha}\Big)}_{\kappa_B} +\frac{1}{2} \bigg)\beta\rho
    - \eta (\beta\rho+1)^2 \underbrace{\frac{\lambda}{B}\Big(\alpha\frac{B-1}{2} + \beta\Big)}_{\lambda\alpha\kappa_B}
\right)  f(x_t) \\
&=
\eta\alpha\left(1-\left(\kappa_B +\frac{1}{2}\right)\beta\rho- \eta\lambda (\beta\rho+1)^2 \kappa_B\right)  f(x_t).
\end{align*}

Hence, we get
\begin{align*}
\expec{}{f(x_{t+1}) ~\big\lvert~ x_t}&\leq \left(1-\eta\alpha\Big(1-\Big(\kappa_B +\frac{1}{2}\Big)\beta\rho\Big) + \eta^2\alpha\lambda (\beta\rho+1)^2 \kappa_B \right)f(x_t).
\end{align*}

Applying total expectation on both sides gives

\begin{align}
    \expec{}{f(x_{t+1})} \leq \left(1-\eta\alpha\Big(1-\Big(\kappa_B +\frac{1}{2}\Big)\beta\rho\Big) + \eta^2\alpha\lambda (\beta\rho+1)^2 \kappa_B \right)\expec{}{f(x_t)}.\label{eq:final_bound_PL-SGD}
\end{align}

Optimizing the multiplicative term in (\ref{eq:final_bound_PL-SGD}) with respect to $\eta$ gives $\eta = \frac{1-(\kappa_B +1/2)\beta\rho}{2\lambda\kappa_B(\beta\rho+1)^2}$, which is $\eta^\star_B$ in the theorem statement. 
With assumption of $\rho\leq \frac{1}{(\beta/\alpha + 1/2)\beta}$ so that we have $\eta^\star_B \geq 0$, applying this to (\ref{eq:final_bound_PL-SGD}) gives 
$$\ex{}{f(x_{t+1})}\leq
\left(1 -  \frac{\alpha\,\eta^\star_B}{2}  \Big(1-\Big(\kappa_B +\frac{1}{2}\Big)\beta\rho\Big)\right)\,\ex{}{f(x_t)},$$
which provides the desired convergence rate.

Last but not least, we calculate the upper bound for $\rho$ to satisfy the assumption $\tau_2\geq 0$ by substituting $\eta$ for $\eta^\star_B$ in $\tau_2$, yielding $\rho\leq\frac{2B\kappa_B+2\beta/\alpha}{(2B-1)\kappa_B+\beta/\alpha} \frac{1}{\beta}$. 
Minimizing this upper bound with respect to $B$ gives $\rho\leq\frac{1}{\beta}$, which is a looser bound than $\rho\leq \frac{1}{(\beta/\alpha + 1/2)\beta}$. 
\end{proof}

%% file: text/Generalization_bound.tex
\section{Test error of SAM can decrease with overparameterization}
\label{sec:theory-gen-main}

Recent works have shown that overparameterization can even improve generalization both empirically and theoretically \citep{neyshabur2017exploring,brutzkus2019larger}.
Here, we present that overparameterization also improves generalization for SAM in the sense that test error can decrease with larger network widths (and thus more parameters).

We follow the same setting of \citet{allen2019learning}.
Specifically, we consider a risk minimization over some unknown data distribution $\mathcal{D}$ using a one-hidden-layer ReLU network with a smooth convex loss function (\eg, cross entropy).
The network is assumed to be initialized with Gaussian and take bounded inputs.
Then, we characterize a generalization property of a stochastic SAM as below.

\begin{theorem}\label{thm:sam-genbound-main}
(Informal)
Suppose we train a network having $m$ hidden neurons with training data sampled from $\D$.
Then, for every $\varepsilon$ in some open interval, there exists $M \propto 1/\varepsilon$ such that for every $m \geq M$, with appropriate values of $\eta, \rho, T$, a stochastic SAM gives the following guarantee on the test loss with high probability:
\begin{equation*}\label{eq:sam-genbound}
\ex{x_0, \cdots, x_{T-1}}{ \frac{1}{T} \sum_{t=0}^{T-1} \mathbb{E}_{\mathcal{D}} [f(x_t)] } \leq \varepsilon.
\end{equation*}
\end{theorem}

We present a formal version of the theorem and its proof in \cref{app:sam-genbound}.

This result suggests that to achieve $\varepsilon$-test accuracy from running $T$ iterations of SAM requires a minimum width $M$ proportional to $1/\varepsilon$.
This indicates that a network with a larger width can achieve a lower test error, and hence, overparameterization can improve generalization for SAM.

\paragraph{Experiment}

We support this result empirically on synthetic data for a simple regression task.
\begin{wrapfigure}{r}{0.25\linewidth}
  \vspace{-1em}
  \resizebox{\linewidth}{!}{
  \centering
  \includegraphics[width=\linewidth, trim={0.8em 0.8em 0.8em 0}, clip]{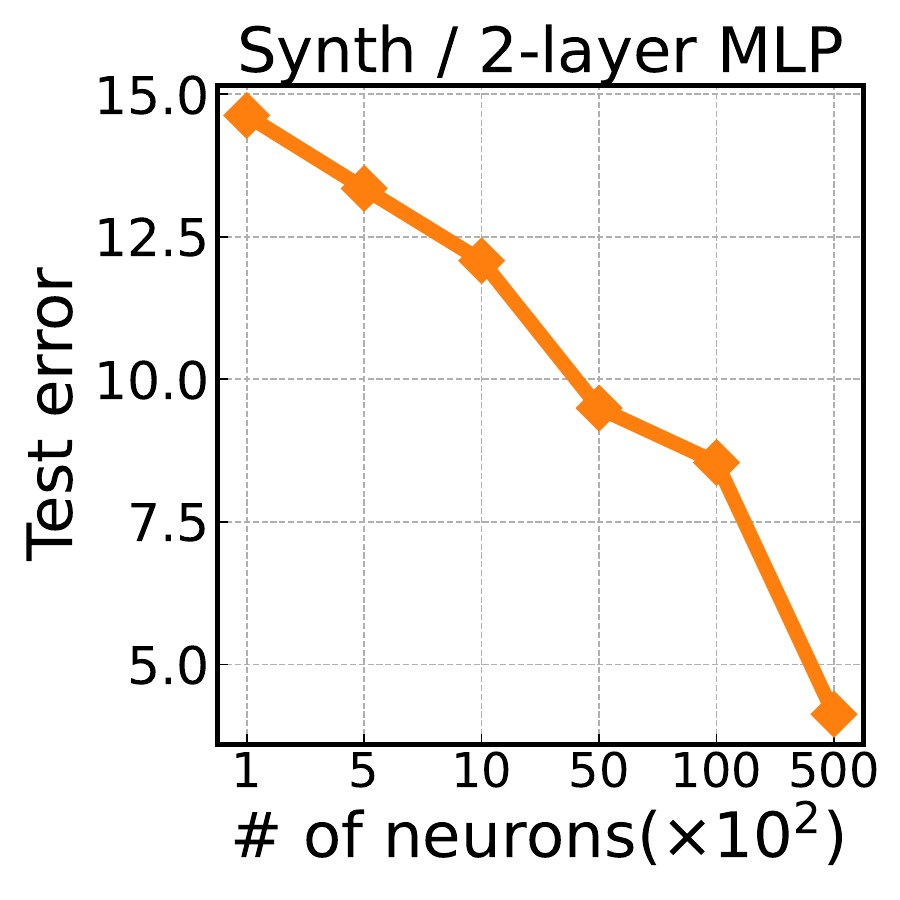}
  }
  \caption{
    Test error keeps on decreasing with a larger number of neurons.
  }
  \label{fig:generror_experiments}
  \vspace{-1.0em}
\end{wrapfigure}
Here, each element of the input $x = (x_1, x_2, x_3, x_4) \in \R^4$ for synthetic data is sampled from random Gaussian distribution and then normalized to satisfy $\| x \|_2 = 1$, and target $y$ is calculated as $y = (\sin(3x_1) + \sin(3x_2) + \sin(3x_3)-2)^2 \cdot \cos (7x_4)$. 
The weights and biases of the first layer are initialized from $\mathcal{N}(0, 1/m)$ where $m$ is the number of hidden neurons, and the weights of the second layer are initialized from $\mathcal{N}(0, 1)$.
Specifically, following the setup of \citet{allen2019learning}, we train $2$-layer ReLU networks with synthetic data.
We only train the weights of the first layer for $800$ epochs, while the biases of the first layer and the weights of the second layer are frozen to initialized values.
We use $1000$ and $5000$ data points for training and testing respectively.
We use a batch size of $50$ without weight decay and decay learning rate by $0.1$ after $50\%$ of the total epochs.
We perform the grid search over learning rate and $\rho$ from $\{10^{-k} \vert 2 \leq k \leq 7\}$ and $\{10^{-k} \vert 1 \leq k \leq 5\}$ respectively.
The results are reported in \cref{fig:generror_experiments}.

\section{Proof of Theorem \ref{thm:sam-genbound-main}}

\label{app:sam-genbound}

In this section, we provide the formal version of \cref{thm:sam-genbound-main} and its proof.

\subsection{Notation and setup}

Throughout this section, we use the same notations and setups as  \citet{allen2019learning}.
We remark that the notations are different from those used in \cref{sec:theory-gen-main,app:prooflinconv,app:sam-stability}.

First, let us assume the unknown data distribution $\D$ where each data $z = (x, y)$ consists of the input $x \in \R^d$ and the corresponding label $y \in \Y$.
We also assume, without loss of generality, that $\| x\|_2 = 1$ and $x_d = 1/2$.
The loss function  $L: \R^k \times \Y \rightarrow \R$ is assumed to be non-negative, convex, $1$-Lipschitz continuous, and $1$-smooth with respect to its first argument.

Next, we define the target network $F^* = (f^*_1, \cdots, f^*_k): \R^d \rightarrow \R^k$ as 
\begin{equation} \label{eq:targetfn}
    f^*_r(x) \eqdef \sum_{i=1}^p a^*_{r,i} \phi_i(\ip{w^*_{1,i}}{x}) \ip{w^*_{2,i}}{x}
\end{equation}
where each $\phi_i: \R \rightarrow \R$ is an infinite-order smooth function.
Here, we assume that $\|w^*_{1,i} \|_2 = \|w^*_{2,i}\|_2 = 1, | a^*_{r,i} | \leq 1$ hold for all $i \in \{1, \cdots, p\}$.
We denote the sample and network complexity of $\phi$ as $\compsam$  and $\compnet$ respectively (see Section 2 of \citet{allen2019learning} for the formal definitions).
Suppose we have a concept class $\C$ that consists of all functions $F^*$ with bounded number of parameters $p$ and complexity $\mathfrak{C}$.
We also denote the population risk achieved by the best target function $F^*$ in this concept class as $\opt$, \ie, $\opt = \underset{F^\star \in \C}{\min} \E_{(x, y) \sim \D}[L(F^*(x), y]$

Then, we define the learner network $F = (f_1, \cdots, f_k): \R^d \rightarrow \R^k$ as below.
\begin{equation} \label{eq:relu-net}
    f_r(x) \eqdef \sum_{i=1}^m \init{a}_{r,i} \operatorname{ReLU} (\ip{w_i}{x} + \init{b}_i).
\end{equation}
Note that the learner network is a $2$-layer ReLU network with $m$ neurons.
We train the network with $n$ sampled data sampled from $\D$ and denote it as $\Z = \{z_1, \cdots, z_N\}$.
We only train the weights $W = (w_1, \cdots, w_m) \in \R^{m \times d}$ and freeze the values of $a, b$ during the training.
We denote the initial value of the weight and its value at time $t$ as $\initw$ and $ \initw + W_t$ respectively.
Each element of $\initw$ and $\init{b}$ are initialized from $\mathcal{N}(0, 1/m)$ while each element of $\init{a}_r$ are initialized from $\mathcal{N}(0, \varepsilon_a^2)$ for some fixed $\varepsilon_a \in (0, 1]$.
At each step $t$, we sample a single data point $z = (x, y)$ from $\Z$ and update $W$ using un-normalized version of SAM:
\begin{align} \label{eq:samdef-gen}
    W_{t+1} &= W_t - \eta \nabla L(F(x; \initw + W_{t + 1/2}), y) \notag \\
    &= W_t - \eta \nabla L(F(x; \initw + \rho \nabla L(F(x; \initw + W_t), y)), y).
\end{align}

\subsection{Formal theorem}

Now, we are ready to present the formal version of \cref{thm:sam-genbound-main} below.

\begin{theorem} \label{thm:sam-genbound-formal} (SAM version of Theorem 1 in \citet{allen2019learning})
    For every $\varepsilon \in \left(0, \frac{1}{p k \compsam (\phi, 1)}\right)$, there exists $M_0 = \poly(\compnet(\phi, 1), 1/\varepsilon)$ and  $N_0 = \poly(\compsam(\phi, 1), 1/\varepsilon)$ such that for every $m \geq M_0$ and every $N \geq \widetilde{\Omega}(N_0)$, by choosing $\varepsilon_a = \varepsilon / \widetilde{\Theta}(1)$ for the initialization and $\eta = \wttheta(\frac{1}{\varepsilon k m}), \rho = \wttheta(\frac{1}{\varepsilon^3 k m^3}), T = \wttheta \left( \frac{(\compsam(\phi, 1))^2 \cdot k^3p^2}{\varepsilon^2} \right)$, running $T$ iterations of stochastic SAM defined in \cref{eq:samdef-gen} gives the following generalization bound with high probability over the random initialization.
    \begin{equation} \label{eq:sam-gen-formal}
        \E_{\text{SAM}} \left [ \frac{1}{T} \sum_{t=0}^{T-1} \E _{(x, y) \sim \D} L(F(x; \initw + W_t), y) \right] \leq \opt + \varepsilon.
    \end{equation}
\end{theorem}

Here, the notation of $\wt{O}(\cdot)$ ignores the factor of $\textsf{polylog}(m)$.

\subsection{Proof of Theorem \ref{thm:sam-genbound-formal}}

We here present the proof of \cref{thm:sam-genbound-formal}.

First, note that we can directly use the algorithm-independent part from \citet{allen2019learning}. 
Thus, it is sufficient to show that the similar version of Lemma B.4 in \citet{allen2019learning} also holds for SAM.

We first define the function $G = (g_1, \cdots, g_k): \R^d \rightarrow \R^k$ as similar to \citet{allen2019learning}. 
\begin{equation} \label{eq:def-g}
    g_r(x; W_t) \eqdef \sum_{i=1}^m \init{a}_{r,i} (\ip{\itert{w}_i}{x} + \init{b}_i) \Id [\ip{\init{w}_i}{x} + \init{b}_i \geq 0].
\end{equation}

Then, the following corollary holds for a stochastic SAM from Lemma B.3 of \citet{allen2019learning}.
The corollary presents an upper bound on the norm of differences between $\frac{\partial}{\partial W} L(F(\cdot), y)$ and $\frac{\partial}{\partial W} L(G(\cdot), y)$.
\begin{corollary} \label{corollary:sam-coupling-main}
    (SAM version of Lemma B.3 in \citet{allen2019learning})
    Let $\tau = \varepsilon_a (\eta + \rho) t$.
    Then, for every $x$ satisfying $\| x \|_2 = 1$, and for every time step $t \geq 1$, the following are satisfied with high probability over the random initialization. \\
    (a) For every $r \in [k]$, 
    \begin{equation*}
        \left \lvert f_r(x; \init{W} + W_t) - g_r (x; \init{W} + W_t) \right \rvert = \wt{O}(\varepsilon_a k \tau^2 m^{3/2})
    \end{equation*} 
    (b) For every $y \in \Y$, 
    \begin{equation}
        \left \| \frac{\partial}{\partial W} L(F(x; \initw + W_t), y) - \frac{\partial}{\partial W} L(G(x; \initw + W_t), y) \right \|_{2,1} \leq \wt{O}(\varepsilon_a k \tau m^{3/2} + \varepsilon_a^2 k^2 \tau^2 m^{5/2})
    \end{equation}
\end{corollary} 

Next, we present the key lemma integral to our proof.
The part $(c)$ will be directly used in the proof and presents an upper bound on the norm of differences between SAM gradient and SGD gradient for $F$.

\begin{lemma} \label{lemma:sam-coupling-samupdate}
    For every $x$ satisfying $\| x \|_2 = 1$, and for every time step $t \geq 1$, the following are satisfied with high probability over the random initialization. \\
    (a) For at most $\wt{O}(\varepsilon_a \rho \sqrt{km})$ fraction of $i \in [m]$: we have 
    \begin{equation*}
        \Id [\ip{\iterthalf{w}_i}{x} + \init{b}_i \geq 0] \neq \Id [ \ip{\itert{w}_i}{x} + \init{b}_i \geq 0].
    \end{equation*}
    (b) For every $r \in [k]$, 
    \begin{equation*}
        \left \lvert f_r(x; \init{W} + W_{t + 1/2}) - f_r (x; \init{W} + W_t) \right \rvert = \wt{O}(\varepsilon_a^3 k \rho^2 m^{3/2} + \varepsilon_a^2 \sqrt{k} \rho m)
    \end{equation*} 
    (c) For every $y \in \Y$, 
    \begin{align}
        & \left \| \frac{\partial}{\partial W} L(F(x; \initw + W_{t + 1/2}), y) - \frac{\partial}{\partial W} L(F(x; \initw + W_t), y) \right \|_{2,1} \notag \\ 
        & \leq \wt{O}(\varepsilon_a^2 k \rho m^{3/2} + \varepsilon_a^4 k^2 \rho^2 m^{5/2} + \varepsilon_a^3 k^{3/2} \rho m^2)
    \end{align}
\end{lemma}

\begin{proof}
    Recall that the following hold from the definition of $F$ (see Lemma B.3 of \citet{allen2019learning} for the details). \\
    \begin{equation} \label{eq:gen-gradient-bound}
        \left \| \frac{\partial}{\partial w_i} f_r (x; \initw + W_t) \right \|_2 \leq \varepsilon_a B \quad \text{and} \quad \left\| \frac{\partial}{\partial w_i} L(F(x; \initw + W_t), y) \right \|_2 \leq \sqrt{k} \varepsilon_a B
    \end{equation}
    (a) 
    Let $\tau = \varepsilon_a \rho$ and define $\calH \eqdef \left \{ i \in [m] \bigg\Vert \left \lvert \ip{\itert{w}_i}{x} + \init{b}_i \right \rvert \geq 2 \sqrt{k} B \tau \right \}$.
    Then, the lemma is a direct corollary from Lemma B.3 (a) of \citet{allen2019learning}. \\
    (b) 
    We divide $i$ into two cases.
    First, when $i \notin \calH$, we can directly utilize Lemma B.3.(b) of \citet{allen2019learning} and the total difference from these $i$'s is $\wt{O}(\varepsilon_a^3 k \rho^2 m^{3/2})$.
    Next, we consider the differences from $i \in \calH$.
    \begin{align*}
        &\quad \left \lvert \init{a}_{r,i} \left( \left\langle \iterthalf{w}_i, x \right \rangle + \init{b}_i \right) \Id \left[\left\langle \iterthalf{w}_i, x \right \rangle + \init{b}_i \geq 0\right] \right. \\ 
        & \hspace{2em} \left. - \init{a}_{r,i} \left( \left\langle \itert{w}_i, x \right \rangle + \init{b}_i \right) \Id \left[\left\langle \itert{w}_i, x \right \rangle + \init{b}_i \geq 0\right] \right \rvert \\
        &\leq \left \lvert \init{a}_{r,i} \left( \left\langle \iterthalf{w}_i - \itert{w}_i, x \right \rangle \right) \right \rvert \\
        &= \left \lvert \init{a}_{r,i} \left (\left \langle \rho \cdot \frac{\partial}{\partial w_i} L(F(x; \initw + W_t), y), x \right \rangle \right ) \right \rvert \\
        & \leq \rho \left \lvert \init{a}_{r,i} \right \rvert \cdot \left \| \frac{\partial}{\partial w_i} L(F(x; \initw + W_t), y)  \right \|_2 \cdot \| x \|_2 \\
        &\leq \rho (\varepsilon_a B) \cdot (\sqrt{k} \varepsilon_a B) \\
        &= \wt{O}(\varepsilon_a^2 \sqrt{k} \rho)
    \end{align*}
    The first inequality is from the fact that $i \in \calH$ and thus $\Id \left[\left\langle \iterthalf{w}_i, x \right \rangle + \init{b}_i \geq 0\right] = \Id \left[\left\langle \itert{w}_i, x \right \rangle + \init{b}_i \geq 0\right]$.
    Then, we have utilized the definition of SAM (\ref{eq:samdef-gen}) and Cauchy-Schwartz inequality.
    Since there can be at most $m$ number of $i \in \calH$, the total differences from $i \in \calH$ amount to $\wt{O}(\varepsilon_a^2 \sqrt{k} \rho m)$.
    Combining the two cases proves the (b). \\
    (c)
    By the chain rule, we have 
    \begin{equation*}
        \frac{\partial}{\partial w_i} L(F(x; \initw + W_t), y) = \nabla L(F(x; \initw + W_t), y) \frac{\partial}{\partial w_i} F(x; \initw + W_t).
    \end{equation*}
    Since $L$ is $1$-smooth, applying the above lemma (b) gives 
    \begin{align} \label{eq:gen-lipschitz-smoothness}
        & \hspace{1.1em} \left \| \nabla L(F(x; \initw + W_{t + 1/2}), y) - \nabla L(F(x; \initw + W_t), y) \right \|_2 \notag \\
        &\leq \left \| F(x; \initw + W_{t + 1/2}) - F(x; \initw + W_t) \right \|_2 \notag \\
        &\leq \wt{O} \left(\varepsilon_a^3 k^{3/2} \rho^2 m^{3/2} + \varepsilon_a^2 k \rho m \right).
    \end{align}
    
    For $i \in \calH$, we have $\Id [\ip{\iterthalf{w}_i}{x} + \init{b}_i \geq 0] = \Id [ \ip{\itert{w}_i}{x} + \init{b}_i \geq 0]$ and thus $\frac{\partial}{\partial w_i} F(x; \initw + W_{t + 1/2}) = \frac{\partial}{\partial w_i} F(x; \initw + W_t)$.
    Then, combining (\ref{eq:gen-lipschitz-smoothness}) with (\ref{eq:gen-gradient-bound}) and using the fact that there can be at most $m$ number of $i \in \calH$, this amounts to $\wt{O} \left(\varepsilon_a^4 k^2 \rho^2 m^{5/2} + \varepsilon_a^3 k^{3/2} \rho m^2 \right)$.

    Next, for $i \notin \calH$, we can directly use the result from Lemma B.3.(c) of \citet{allen2019learning} and this contributes to $\wt{O}(\varepsilon_a^2 k \rho m^{3/2})$.
    Summing these together, we prove the bound.
\end{proof}

Finally, we show that the following lemma holds, which is a SAM version of Lemma B.4 in \citet{allen2019learning}.
Combined with the algorithm-independent parts presented in \citet{allen2019learning}, proving the following lemma concludes the proof of \cref{thm:sam-genbound-formal}.
We use the notation of $L_F(\Z; W)$ for $L_F(\Z; W) \eqdef \frac{1}{|\Z|} \sum_{(x, y) \in \Z} L (F(x; W + \initw), y)$ and similarly define $L_G(\Z; W)$.

\begin{lemma} \label{lemma:sam-gen-optim} (SAM version of Lemma B.4 in \citet{allen2019learning})
    For every $\varepsilon \in \left(0, \frac{1}{p k \compsam(\phi, 1}\right)$, letting $\varepsilon_a = \varepsilon / \wttheta(1)$, $\eta = \wttheta(\frac{1}{\varepsilon k m})$, and $\rho = \wttheta(\frac{1}{\varepsilon^3 k m^3})$, there exists $M = \poly (\compnet(\phi, 1), 1/\varepsilon)$ and $T = \Theta\left(\frac{k^3 p^2 \cdot \compsam(\phi, 1)^2}{\varepsilon^2}\right)$
    such that if $m \geq M$, the following holds with high probability over random initialization.
    \begin{equation} \label{eq:gen-optimbound}
        \frac{1}{T} \sum_{t=0}^{T-1} L_F (\Z, W_t) \leq \opt + \varepsilon.
    \end{equation}
\end{lemma}
\begin{proof}
    Let $\refw$ be the weights constructed from the Corollary B.2 in \citet{allen2019learning}.
    By the convexity of $L$ and Cauchy-Schwartz inequality, we have 
    \begin{align*}
        L_G(\Z, W_t) - L_G(\Z; \refw) &\leq \ip{\nabla L_G(\Z; W_t)}{W_t - \refw} \\
        &= \ip{\nabla L_G(\Z; W_t) - \nabla L_F (\Z; W_t)}{W_t - \refw} \\
        &\quad + \ip{\nabla L_F(\Z; W_t) - \nabla L_F(\Z; W_{t+1/2})}{W_t - \refw} \\
        &\quad + \ip{\nabla L_F(\Z; W_{t+1/2})}{W_t - \refw} \\
        &\leq \| \nabla L_G(\Z; W_t) - \nabla L_F (\Z; W_t) \|_{2,1} \| W_t - \refw\|_{2,\infty} \\
        &\quad + \| \nabla L_F(\Z; W_t) - \nabla L_F(\Z; W_{t+1/2}) \|_{2,1} \| W_t - \refw\|_{2,\infty} \\
        &\quad + \ip{\nabla L_F(\Z; W_{t+1/2})}{W_t - \refw}
    \end{align*}

    From the SAM update rule (\ref{eq:samdef-gen}), we have the following equality. 
    \begin{align*}
        \| W_{t+1} - \refw \|_F^2 &= \| W_t - \eta \nabla L_F (\itert{z}, W_{t+1/2}) - \refw \|_F^2 \\
        &= \| W_t - \refw\|_F^2 -2 \eta \ip{\nabla L_F(\itert{z}, W_{t + 1/2})}{W_t - \refw} + \eta^2 \| \nabla L_F (\itert{z}, W_{t + 1/2})\|_F^2.
    \end{align*}

    Thus, we have 
    \begin{align*}
        L_G(\Z; W_t) - L_G(\Z; \refw) 
        & \leq \underbrace{\| \nabla L_G(\Z; W_t) - \nabla L_F (\Z: W_t) \|_{2,1} \|W_t - \refw\|_{2, \infty}}_{(A)} \\
        &\quad + \underbrace{\| \nabla L_F(\Z; W_t) - \nabla L_F(\Z; W_{t+1/2}) \|_{2,1} \| W_t - \refw\|_{2,\infty}}_{(B)} \\
        &\quad + \frac{\| W_t - \refw\|_F^2 - \E_{\itert{z}}[\| W_{t+1} - \refw\|_F^2]}{2\eta} \\
        &\quad + \underbrace{\frac{\eta}{2} \| \nabla L_F (W_{t+1/2}, \itert{z})\|_F^2}_{(C)}.
    \end{align*}

    Since $\| W_t - \refw\|_{2,\infty} = \wt{O}(\sqrt{k}\varepsilon_a (\eta + \rho) t + 
    \frac{k p C_0}{\varepsilon_a m})$, $(A)$ is bounded as 
    \begin{equation*}
        (A) = \wt{O}\left(\sqrt{k}\varepsilon_a (\eta + \rho) T \Delta + \frac{kpC_0}{\varepsilon_a m} \Delta \right)
    \end{equation*}
    where $\Delta = \wt{O}\left(\varepsilon_a^2 k (\eta + \rho) T m^{3/2} + \varepsilon_a^4 k^2 (\eta + \rho)^2 T^2 m^{5/2}\right)$.

    Next, we can bound $(B)$ from \cref{lemma:sam-coupling-samupdate}(c) as follows.
    \begin{equation*}
        (B) = \wt{O}(\sqrt{k} \varepsilon_a (\eta + \rho) T \Delta' + \frac{k p C_0}{\varepsilon_a m}\Delta'),
    \end{equation*}
    where $ \| \nabla L_F(\Z; W_t) - \nabla L_F(\Z; W_{t+1/2}) \|_{2,1} \leq \Delta' = \varepsilon_a^2 k \rho m^{3/2} + \varepsilon_a^4 k^2 \rho^2 m^{5/2} + \varepsilon_a^3 k^{3/2} \rho m^2$.

    We also have
    \begin{equation*}
        (C) = \wt{O}(\eta \varepsilon_a^2 k m)
    \end{equation*}
    since the norm of $\nabla L_F$ is always bounded as $\| \nabla L_F (\cdot, \itert{z})\|_F^2 = \wt{O}(\varepsilon_a^2 k m)$.

    Then, by telescoping, we have 
    \begin{align*}
        \frac{1}{T} \sum_{t=0}^{T-1} \E_{\text{SAM}} [L_G(\Z; W_t)] - L_G(\Z; \refw) 
        &\leq \wt{O}\left(\sqrt{k}\varepsilon_a (\eta + \rho) T \Delta + \frac{kpC_0}{\varepsilon_a m} \Delta \right) \\
        &+ \wt{O}\left(\sqrt{k} \varepsilon_a (\eta + \rho) T \Delta' + \frac{k p C_0}{\varepsilon_a m}\Delta'\right) \\
        &\quad + \underbrace{\frac{\| W_0 - \refw\|_F^2 }{2 \eta T}}_{(D)} + \wt{O}(\eta \varepsilon_a^2 k m).
    \end{align*}

    We can bound $(D)$ in the same way as \citet{allen2019learning}, 
    \begin{equation*}
        (D) = \frac{\| W_0 - \refw\|_F^2 }{2\eta T} = \wt{O}\left(\frac{k^2 p^2 \compsam(\phi, 1)^2}{\varepsilon_a^2 m} \cdot \frac{\1}{\eta T} \right).
    \end{equation*}

    By setting $\eta = \wttheta(\frac{\varepsilon}{km\varepsilon_a^2}), \rho = \wttheta(\frac{\varepsilon}{km^3 \varepsilon_a^4}), T = \wttheta(k^3 p^2 \compsam(\phi, 1)^2 / \varepsilon^2)$, we have 
    $\Delta = \wt{O}(\frac{k^6 p^4 \compsam(\phi, 1)^4}{m^{3/2} \varepsilon^4})$ and
    $\Delta' = \wt{O}(\frac{1}{m^{3/2} \varepsilon} + \frac{\sqrt{k}}{m})$.
    Hence, with large enough $m$, we obtain the following inequality and prove \cref{lemma:sam-gen-optim}, combined with the remaining parts from \citet{allen2019learning}.
    \begin{equation}
        \frac{1}{T} \sum_{t=0}^{T-1} \E_{\text{SAM}} [L_G(\Z; W_t)] - L_G(\Z; \refw) \leq O(\varepsilon).
    \end{equation}
\end{proof}